\documentclass{article}

\usepackage{etoolbox}
\newcommand{\arxiv}[1]{\iftoggle{icml}{}{#1}}
\newcommand{\icml}[1]{\iftoggle{icml}{#1}{}}
\newtoggle{icml}
\global\togglefalse{icml}

\newcommand{\footref}[1]{\textsuperscript{\ref{#1}}}
\newcommand{\loose}{\looseness=-1}

\icml{
\usepackage{icml2023}
\usepackage{scrextend}
\usepackage{natbib}
\bibliographystyle{plainnat}
\bibpunct{(}{)}{;}{a}{,}{,}
}

\arxiv{
\usepackage[utf8]{inputenc} %
\usepackage[T1]{fontenc}    %
\usepackage{lmodern}
}

\usepackage{url}            %
\usepackage{booktabs}       %
\usepackage{nicefrac}       %
\usepackage{microtype}      %
\usepackage{makecell}
\usepackage{enumitem}
\usepackage{etoolbox}
\usepackage{comment}
\usepackage{wrapfig}
\usepackage{colortbl}
\usepackage{setspace}
\usepackage{transparent}
\usepackage{inconsolata}
\usepackage[scaled=.90]{helvet}
\usepackage{xspace}
\usepackage{verbatim}
\usepackage{mathtools}
\allowdisplaybreaks

\usepackage{algorithm}
\usepackage[noend]{algpseudocode}

\arxiv{
\usepackage[letterpaper, left=1in, right=1in, top=1in, bottom=1in]{geometry}

\PassOptionsToPackage{hypertexnames=false}{hyperref}  %
\usepackage{parskip}
}

\usepackage[dvipsnames]{xcolor}
\usepackage[colorlinks=true, linkcolor=blue!70!black, citecolor=blue!70!black,urlcolor=black]{hyperref}

\usepackage{microtype}
\usepackage{hhline}

\makeatletter
\newcommand{\neutralize}[1]{\expandafter\let\csname c@#1\endcsname\count@}
\makeatother

\usepackage{algorithm}

\usepackage{natbib}
\usepackage{amsthm}
\usepackage{mathtools}
\usepackage{amsmath}
\usepackage{bm}
\usepackage{bbm}
\usepackage{amsfonts}
\usepackage{amssymb}
\usepackage{MnSymbol} %
\usepackage[nameinlink,capitalize]{cleveref}

\usepackage{thmtools}
\usepackage{thm-restate}
\declaretheorem[name=Theorem,parent=section]{theorem}
\declaretheorem[name=Lemma,parent=section]{lemma}

\declaretheorem[name=Assumption, parent=section]{assumption}
\declaretheorem[name=Condition, parent=section]{condition}

\declaretheorem[name=Remark,parent=section]{remark}
\declaretheorem[name=Proposition, parent=section]{proposition}

\makeatletter
\renewenvironment{proof}[1][Proof]%
{%
	\par\noindent{\bfseries\upshape {#1.}\ }%
}%
{\qed\newline}
\makeatother

\newtheorem*{theorem*}{Theorem}

\theoremstyle{plain}
\newtheorem{definition}[theorem]{Definition}

\usepackage{xpatch}
\xpatchcmd{\proof}{\itshape}{\normalfont\proofnameformat}{}{}
\newcommand{\proofnameformat}{\bfseries}

\newcommand{\pref}[1]{\cref{#1}}
\newcommand{\pfref}[1]{Proof of \pref{#1}}

\renewcommand{\eqref}[1]{\texorpdfstring{\hyperref[#1]{(\ref*{#1})}}{(\ref*{#1})}}
\crefformat{equation}{#2Eq.\,(#1)#3}
\Crefformat{equation}{#2Eq.\,(#1)#3}

\Crefformat{figure}{#2Figure~#1#3}
\Crefformat{assumption}{#2Assumption~#1#3}

\Crefname{assumption}{Assumption}{Assumptions}

\def\ddefloop#1{\ifx\ddefloop#1\else\ddef{#1}\expandafter\ddefloop\fi}
\def\ddef#1{\expandafter\def\csname bb#1\endcsname{\ensuremath{\mathbb{#1}}}}
\ddefloop ABCDEFGHIJKLMNOPQRSTUVWXYZ\ddefloop
\def\ddefloop#1{\ifx\ddefloop#1\else\ddef{#1}\expandafter\ddefloop\fi}
\def\ddef#1{\expandafter\def\csname b#1\endcsname{\ensuremath{\mathbf{#1}}}}
\ddefloop ABCDEFGHIJKLMNOPQRSTUVWXYZ\ddefloop
\def\ddef#1{\expandafter\def\csname sf#1\endcsname{\ensuremath{\mathsf{#1}}}}
\ddefloop ABCDEFGHIJKLMNOPQRSTUVWXYZ\ddefloop
\def\ddef#1{\expandafter\def\csname c#1\endcsname{\ensuremath{\mathcal{#1}}}}
\ddefloop ABCDEFGHIJKLMNOPQRSTUVWXYZ\ddefloop
\def\ddef#1{\expandafter\def\csname h#1\endcsname{\ensuremath{\widehat{#1}}}}
\ddefloop ABCDEFGHIJKLMNOPQRSTUVWXYZ\ddefloop
\def\ddef#1{\expandafter\def\csname hc#1\endcsname{\ensuremath{\widehat{\mathcal{#1}}}}}
\ddefloop ABCDEFGHIJKLMNOPQRSTUVWXYZ\ddefloop
\def\ddef#1{\expandafter\def\csname t#1\endcsname{\ensuremath{\widetilde{#1}}}}
\ddefloop ABCDEFGHIJKLMNOPQRSTUVWXYZ\ddefloop
\def\ddef#1{\expandafter\def\csname tc#1\endcsname{\ensuremath{\widetilde{\mathcal{#1}}}}}
\ddefloop ABCDEFGHIJKLMNOPQRSTUVWXYZ\ddefloop
\def\ddefloop#1{\ifx\ddefloop#1\else\ddef{#1}\expandafter\ddefloop\fi}
\def\ddef#1{\expandafter\def\csname scr#1\endcsname{\ensuremath{\mathscr{#1}}}}
\ddefloop ABCDEFGHIJKLMNOPQRSTUVWXYZ\ddefloop

\newcommand{\indic}{\mathbbm{I}}    %

\newcommand{\eps}{\epsilon}
\newcommand{\veps}{\varepsilon}

\DeclareMathOperator*{\argmin}{arg\,min} %
\DeclareMathOperator*{\argmax}{arg\,max}

\def\ddef#1{\expandafter\def\csname b#1\endcsname{\ensuremath{\mb{#1}}}}
\ddefloop abcdeghijklmnopqrstuvwxyz\ddefloop

\newcommand{\ind}[1]{^{(#1)}}

\DeclarePairedDelimiter{\abs}{\lvert}{\rvert} %
\DeclarePairedDelimiter{\brk}{[}{]}
\DeclarePairedDelimiter{\crl}{\{}{\}}
\DeclarePairedDelimiter{\prn}{(}{)}

\DeclarePairedDelimiter{\ceil}{\lceil}{\rceil}

\let\P\undefined
\DeclareMathOperator{\En}{\mathbb{E}}

\DeclareMathOperator{\P}{P}

\newcommand{\mb}[1]{\boldsymbol{#1}}
\renewcommand{\bm}[1]{\boldsymbol{#1}}
\newcommand{\wt}[1]{\widetilde{#1}}

\newcommand{\wb}[1]{\widebar{#1}}

\newcommand{\ldef}{\vcentcolon=}

\newcommand{\musikp}{$\musik+\psdp$}
\newcommand{\comblock}{CombLock\xspace}

\newcommand{\lhs}{left-hand side\xspace}
\newcommand{\rhs}{right-hand side\xspace}
\newcommand{\cMbar}{\wb{\cM}}
\newcommand{\cAbar}{\wb{\cA}}
\newcommand{\etamin}{\eta_{\mathrm{min}}}

\renewcommand{\ln}{\log}        %
\newcommand{\phistar}{\phi_{\star}}

\newcommand{\tfrak}{\mathfrak{t}}
\newcommand{\term}{\mathfrak{t}}
\newcommand{\afrak}{\mathfrak{a}}

\newcommand{\bmdps}{Block MDPs\xspace}
\newcommand{\Pbayes}{P_{\bayes}}
\newcommand{\pihat}{\hat{\pi}}
\newcommand{\pistar}{\pi_{\star}}
\newcommand{\ahat}{\hat{a}}

\let\oldparagraph\paragraph
\renewcommand{\paragraph}[1]{\oldparagraph{#1.}}
\renewcommand{\colon}{:}        %
\newcommand{\range}[2]{\brk*{#1\ldotst{}#2}}

\newcommand{\reals}{\mathbb{R}}

\newcommand{\err}{\text{err}}
\renewcommand{\P}{\mathbb{P}}

\newcommand{\sfrak}{\mathfrak{s}}

\newcommand{\E}{\mathbb{E}}
\newcommand{\nm}{\texttt{NM}}
\newcommand{\Pinm}{\Pi_{\texttt{NM}}}
\newcommand{\nn}{\nonumber} 
\newcommand{\ldotst}{%
	\mathinner{{\ldotp}{\ldotp}}%
}

\newcommand{\unifa}{\pi_\texttt{unif}}
\newcommand{\unif}{\texttt{unif}}
\newcommand{\stat}{\texttt{stat}}
\newcommand{\olive}{\texttt{OLIVE}}
\newcommand{\homer}{\texttt{HOMER}}

\renewcommand{\a}{\bm{a}}

\newcommand{\x}{\bm{x}}

\newcommand{\s}{\bm{s}}
\newcommand{\supp}{\textrm{supp}\,}

\newcommand{\bayes}{\texttt{bayes}}
\newcommand{\Pibar}{\wbar{\Pi}}
\newcommand{\sh}{\mathfrak{s}'}
\newcommand{\st}{\mathfrak{s}}

\renewcommand{\i}{\mathfrak{i}}
\renewcommand{\j}{\mathfrak{j}}
\renewcommand{\k}{\mathfrak{l}}

\newcommand{\fraka}{\mathfrak{a}}
\newcommand{\wtilde}[1]{\widetilde{#1}}
\newcommand{\wbar}[1]{\widebar{#1}}
\newcommand{\iotahat}{\hat{\iota}}
\newcommand{\fhat}{\hat{f}}
\newcommand{\phihat}{\hat{\phi}}
\newcommand{\Ebar}{\wbar\E}
\newcommand{\Pbar}{\wbar\P}
\newcommand{\dbar}{\bar{d}}
\newcommand{\Mbar}{\wbar{\cM}}
\newcommand{\Abar}{\wbar{\cA}}

\newcommand{\fk}{P_\texttt{FK}}

\newcommand{\rbar}{\bar{r}}
\newcommand{\Pibarm}{\wbar{\Pi}_{\texttt{M}}}
\newcommand{\ikdp}{\texttt{IKDP}\xspace}

\newcommand{\ikdptab}{\texttt{IKDP.Tab}\xspace}
\newcommand{\emikdptab}{\texttt{\em IKDP.Tab}\xspace}
\newcommand{\musik}{\texttt{MusIK}\xspace}

\newcommand{\musiktab}{\texttt{MusIK.Tab}\xspace}

\newcommand{\psdp}{\texttt{PSDP}\xspace}

\newcommand{\Pim}{\Pi_{\texttt{M}}}

\renewcommand{\emptyset}{\varnothing}

\newcommand{\algcommentlight}[1]{\textcolor{blue!70!black}{\transparent{0.5}\footnotesize{\texttt{\textbf{//\hspace{2pt}#1}}}}}

 \newcommand{\algcommentbiglight}[1]{\textcolor{blue!70!black}{\transparent{0.5}\footnotesize{\texttt{\textbf{/* #1~*/}}}}}

\newcommand{\bigoh}{O}
\newcommand{\bigoht}{\wt{O}}
\newcommand{\bigom}{\Omega}

\newcommand{\poly}{\mathrm{poly}}

\let\underbar\undefined

\makeatletter
\let\save@mathaccent\mathaccent
\newcommand*\if@single[3]{%
  \setbox0\hbox{${\mathaccent"0362{#1}}^H$}%
  \setbox2\hbox{${\mathaccent"0362{\kern0pt#1}}^H$}%
  \ifdim\ht0=\ht2 #3\else #2\fi
  }
\newcommand*\rel@kern[1]{\kern#1\dimexpr\macc@kerna}
\newcommand*\widebar[1]{\@ifnextchar^{{\wide@bar{#1}{0}}}{\wide@bar{#1}{1}}}
\newcommand*\underbar[1]{\@ifnextchar_{{\under@bar{#1}{0}}}{\under@bar{#1}{1}}}
\newcommand*\wide@bar[2]{\if@single{#1}{\wide@bar@{#1}{#2}{1}}{\wide@bar@{#1}{#2}{2}}}
\newcommand*\under@bar[2]{\if@single{#1}{\under@bar@{#1}{#2}{1}}{\under@bar@{#1}{#2}{2}}}
\newcommand*\wide@bar@[3]{%
  \begingroup
  \def\mathaccent##1##2{%
    \let\mathaccent\save@mathaccent
    \if#32 \let\macc@nucleus\first@char \fi
    \setbox\z@\hbox{$\macc@style{\macc@nucleus}_{}$}%
    \setbox\tw@\hbox{$\macc@style{\macc@nucleus}{}_{}$}%
    \dimen@\wd\tw@
    \advance\dimen@-\wd\z@
    \divide\dimen@ 3
    \@tempdima\wd\tw@
    \advance\@tempdima-\scriptspace
    \divide\@tempdima 10
    \advance\dimen@-\@tempdima
    \ifdim\dimen@>\z@ \dimen@0pt\fi
    \rel@kern{0.6}\kern-\dimen@
    \if#31
      \overline{\rel@kern{-0.6}\kern\dimen@\macc@nucleus\rel@kern{0.4}\kern\dimen@}%
      \advance\dimen@0.4\dimexpr\macc@kerna
      \let\final@kern#2%
      \ifdim\dimen@<\z@ \let\final@kern1\fi
      \if\final@kern1 \kern-\dimen@\fi
    \else
      \overline{\rel@kern{-0.6}\kern\dimen@#1}%
    \fi
  }%
  \macc@depth\@ne
  \let\math@bgroup\@empty \let\math@egroup\macc@set@skewchar
  \mathsurround\z@ \frozen@everymath{\mathgroup\macc@group\relax}%
  \macc@set@skewchar\relax
  \let\mathaccentV\macc@nested@a
  \if#31
    \macc@nested@a\relax111{#1}%
  \else
    \def\gobble@till@marker##1\endmarker{}%
    \futurelet\first@char\gobble@till@marker#1\endmarker
    \ifcat\noexpand\first@char A\else
      \def\first@char{}%
    \fi
    \macc@nested@a\relax111{\first@char}%
  \fi
  \endgroup
}
\newcommand*\under@bar@[3]{%
  \begingroup
  \def\mathaccent##1##2{%
    \let\mathaccent\save@mathaccent
    \if#32 \let\macc@nucleus\first@char \fi
    \setbox\z@\hbox{$\macc@style{\macc@nucleus}_{}$}%
    \setbox\tw@\hbox{$\macc@style{\macc@nucleus}{}_{}$}%
    \dimen@\wd\tw@
    \advance\dimen@-\wd\z@
    \divide\dimen@ 3
    \@tempdima\wd\tw@
    \advance\@tempdima-\scriptspace
    \divide\@tempdima 10
    \advance\dimen@-\@tempdima
    \ifdim\dimen@>\z@ \dimen@0pt\fi
    \rel@kern{0.6}\kern-\dimen@
    \if#31
      \underline{\rel@kern{-0.6}\kern\dimen@\macc@nucleus\rel@kern{0.4}\kern\dimen@}%
      \advance\dimen@0.4\dimexpr\macc@kerna
      \let\final@kern#2%
      \ifdim\dimen@<\z@ \let\final@kern1\fi
      \if\final@kern1 \kern-\dimen@\fi
    \else
      \underline{\rel@kern{-0.6}\kern\dimen@#1}%
    \fi
  }%
  \macc@depth\@ne
  \let\math@bgroup\@empty \let\math@egroup\macc@set@skewchar
  \mathsurround\z@ \frozen@everymath{\mathgroup\macc@group\relax}%
  \macc@set@skewchar\relax
  \let\mathaccentV\macc@nested@a
  \if#31
    \macc@nested@a\relax111{#1}%
  \else
    \def\gobble@till@marker##1\endmarker{}%
    \futurelet\first@char\gobble@till@marker#1\endmarker
    \ifcat\noexpand\first@char A\else
      \def\first@char{}%
    \fi
    \macc@nested@a\relax111{\first@char}%
  \fi
  \endgroup
}
\makeatother

 \usepackage{accents}
\usepackage{color-edits}
\addauthor{df}{ForestGreen}

\addauthor{dk}{magenta}
\addauthor{zm}{red}
\addauthor{sr}{red}
\addauthor{jq}{yellow!60!black}

\newcommand{\zm}[1]{\zmcomment{#1}}

\usepackage[textsize=tiny]{todonotes}

	\usepackage{crossreftools}
	\newcommand{\creftitle}[1]{\crtcref{#1}}

\makeatletter
\let\OldStatex\Statex
\renewcommand{\Statex}[1][3]{%
	\setlength\@tempdima{\algorithmicindent}%
	\OldStatex\hskip\dimexpr#1\@tempdima\relax}
\makeatother

\icml{
\addtocontents{toc}{\protect\setcounter{tocdepth}{0}}
}

\usepackage{tikz}
\usetikzlibrary{decorations.pathreplacing,calc}

\arxiv{
\title{Representation Learning with Multi-Step Inverse Kinematics:\\An
  Efficient and Optimal Approach to Rich-Observation RL}

\author{Zakaria Mhammedi\\{\small \texttt{mhammedi@mit.edu}} \and    Dylan J. Foster\\{\small \texttt{dylanfoster@microsoft.com}} \and Alexander Rakhlin\\{\small \texttt{rakhlin@mit.edu}}
}
\date{}
}
\icml{
	\icmltitlerunning{Representation Learning with Multi-Step Inverse Kinematics}
}

\begin{document}	
\arxiv{
	\maketitle
}

\icml{
		\twocolumn[
	\icmltitle{Representation Learning with Multi-Step Inverse Kinematics:\\An
		Efficient and Optimal Approach to Rich-Observation RL}

\icmlsetsymbol{equal}{*}
	
	\begin{icmlauthorlist}
		\icmlauthor{Firstname1 Lastname1}{yyy}
		\icmlauthor{Firstname2 Lastname2}{comp}
		\icmlauthor{Firstname3 Lastname3}{yyy}
	\end{icmlauthorlist}
	
	\icmlaffiliation{yyy}{Department of XXX, University of YYY, Location, Country}
	\icmlaffiliation{comp}{Company Name, Location, Country}
	
	\icmlcorrespondingauthor{Firstname1 Lastname1}{first1.last1@xxx.edu}

	\icmlkeywords{Machine Learning, Reinforcement Learning, Representation Learning, Rich Observation, Inverse Kinematics}
	
	\vskip 0.3in
	]
	
	\printAffiliationsAndNotice{}  %
}
	
	\begin{abstract}
We study the design of sample-efficient algorithms for reinforcement
learning in the presence of rich, high-dimensional observations,
formalized via the
\emph{Block MDP} problem. Existing algorithms suffer from either 1) computational intractability, 2) strong statistical
assumptions that are not necessarily satisfied in practice, or 3)
suboptimal sample complexity. We address these issues by providing the first computationally efficient algorithm that attains
rate-optimal sample complexity with respect to the desired accuracy
level, with minimal statistical assumptions. Our algorithm, \musik,
combines systematic exploration with representation learning based on
\emph{multi-step inverse kinematics}, a learning objective in which
the aim is to predict the learner's own action from the current
observation and observations in the (potentially distant) future. \musik is simple and
flexible, and can efficiently take advantage of general-purpose
function approximation. Our analysis leverages several new techniques tailored to
non-optimistic exploration algorithms, which we anticipate will find broader use.

 	\end{abstract}
        
        \arxiv{
	\addtocontents{toc}{\protect\setcounter{tocdepth}{2}}
        {
          \hypersetup{hidelinks}
          \tableofcontents
        }
        }

	\section{Introduction}
        \label{sec:intro}

Many of the most promising application domains for reinforcement learning entail navigating unknown environments in the presence of complex, high-dimensional sensory inputs. For example, a challenging task in robotic control is to navigate to a goal state in a new, unmapped environment using only raw pixels from a camera as feedback \citep{baker2022video,bharadhwaj2022information}. Such tasks demand reinforcement learning agents capable of both 1) deliberate exploration, and 2) representation learning, as a means to learn from high-dimensional (``rich'') observations. In this context, a major challenge---in theory and practice---is to develop algorithms that are practical and sample-efficient, yet require minimal prior knowledge.

We study the design of sample-efficient algorithms for rich-observation reinforcement learning through a canonical model known as the \emph{Block MDP}  \citep{jiang2017contextual,du2019latent}. The Block MDP is a setting in which the \emph{observed} state space $\cX$ is high-dimensional (e.g., pixels from a camera), but the dynamics are governed by a small, finite \emph{latent} state space (e.g., a robot's actuator configuration). The key structural property of the Block MDP model, which makes the problem tractable statistically, is that the latent states can be uniquely \emph{decoded} from observations (avoiding issues of partial observability). However, the mapping from observations to latent states is not known in advance, necessitating the use of representation learning in tandem with exploration. As such, the Block MDP is appealing as a stylized testbed in which to study design of sample-efficient algorithms based on representation learning.

Algorithm design for the Block MDP is particularly challenging because representation learning and exploration are not only required, but must be \emph{interleaved}: learning a good representation is necessary to effectively control the agent and explore, but it is difficult to learn such a representation without exploring and gathering diverse feedback.
  In spite of extensive research into the design of algorithms with provable guarantees \citep{jiang2017contextual,du2019latent,misra2020kinematic,zhang2022efficient,uehara2022}, all existing algorithms suffer from one or more of the following drawbacks:
\begin{enumerate}
\item Computational intractability.
\item Strong statistical assumptions that are not necessarily satisfied in practice.
\item Suboptimal sample complexity.
\end{enumerate}
In more detail, computationally efficient algorithms can be split into two families. The first achieves rate-optimal sample complexity with respect to the desired accuracy level \citep{misra2020kinematic,modi2021model}, but their guarantees scale inversely proportional to a \emph{reachability} parameter which captures the minimum probability with which any state can be reached by a policy targeting it; when reachability is violated, these results give no guarantees. More recent approaches dispense with the reachability assumption \citep{zhang2022efficient}, but do not attain rate-optimal sample complexity.

\paragraph{Our contributions} We address issues (1), (2), and (3) by providing the first computationally efficient algorithm that attains rate-optimal sample complexity\footnote{We use the term ``rate-optimal'' to refer to optimality of the rate with respect to the accuracy parameter $\veps$, but not necessarily with respect to other parameters.} without reachability or other strong statistical assumptions\arxiv{ (\cref{tb:resultscomp})}\icml{ (\cref{tb:resultscomp} in \cref{sec:omitted})}. Our algorithm, \musik (``\texttt{Mu}lti-\texttt{s}tep \texttt{I}nverse \texttt{K}inematics''), interleaves exploration with representation learning based on \emph{multi-step inverse kinematics} \citep{lamb2022guaranteed}, a learning objective in which the aim is to predict the learner's own action from the current observation and observations in the (potentially distant) future.
\musik is simple and flexible: it can take advantage of general-purpose function approximation, and is computationally efficient whenever a standard supervised regression objective for the function class of interest can be solved efficiently. In a validation experiment, we find that it obtains comparable or superior performance to other provably efficient methods \citep{misra2020kinematic,zhang2022efficient}. \loose
\arxiv{

        \begin{table}[tp]
        	               \caption{Comparison of sample complexity required learn an $\veps$-optimal
        		policy.
        		For approaches that require a minimum
        		reachability assumption $\eta_{\min} \coloneqq \min_{s\in
        			\cS}\max_{\pi \in \Pim}d^{\pi}(s)$ denotes
        		the reachability parameter. $\Phi$ and $\Psi$ denote the decoder and model classes, respectively. 
        	}
        	\label{tb:resultscomp}
          \renewcommand{\arraystretch}{1.6}
		\fontsize{9}{10}\selectfont
		\centering 
		\begin{tabular}{ccccc}
			\hline
			& Sample complexity & Model-free &
                                                           Comp. efficient
                  &\makecell{Rate-optimal \\ $1/\veps^2$-sample
                  comp.} \\
			\hline
			$\olive$ \citep{jiang2017contextual} & $\frac{A^2 H^3 S^3 \ln |\Phi|}{\veps^2}$ & Yes & No & Yes \\
			$\texttt{MOFFLE}$ \citep{modi2021model} & $\frac{A^{13 }H^8 S^7 \ln |\Phi|}{(\veps^2\eta_{\min} \wedge \eta_{\min}^5)}$ & Yes& Yes &  No \\
			$\homer$ \citep{misra2020kinematic} &
			$\frac{A H S^6 (S^2 A^3+\ln |\Phi|)}{(\veps^2 \wedge \eta_{\min}^3)}$ &Yes & Yes &  No\\
			$\texttt{Rep-UCB}$ \citep{uehara2022} &$\frac{ A^2 H^5 S^4 \ln (|\Phi|\textcolor{red!70!black}{|\Psi|}) }{\veps^2}$ & No & Yes&  Yes \\
			$\texttt{BRIEE}$ \citep{zhang2022efficient} & $ \frac{A^{14} H^9 S^8 \ln |\Phi|}{\veps^4}$& Yes & Yes & No  \\
			$\musik$ (this paper) & $\frac{A^2 H^4 S^{10} (A S^3 + \ln |\Phi|)}{\veps^2}$& \textbf{Yes} &  \textbf{Yes} &  \textbf{Yes}\\
			\hline
		\end{tabular}
	\end{table}

 }

\paragraph {Organization}
\arxiv{
\pref{sec:setting} introduces the Block MDP setting and the online reinforcement learning framework, as well as necessary notation. In \cref{sec:multi}, we present our main algorithm, $\musik$, formally state its main guarantee for reward-free exploration, and discuss some of its implications. We also provide (\cref{sec:rewardsetting}) guarantees for \emph{reward-based} learning with \musik. In \cref{sec:overview}, we give an overview of the main analysis ideas behind \musik, and in \cref{sec:experiments} we present experimental results. We conclude with discussion and future directions in \cref{sec:discussion}. All proofs are deferred to the appendix unless otherwise stated.
}
\icml{
  \pref{sec:setting} introduces the Block MDP setting and the online reinforcement learning framework, as well as necessary notation. In \cref{sec:multi}, we present our main algorithm, $\musik$, formally state its main guarantee for reward-free exploration, and discuss some of its implications. In \cref{sec:key}, we give an overview of the main analysis ideas behind \musik, with a more thorough overview deferred to \pref{sec:overview}. Lastly, in \cref{sec:experiments} we present an experimental validation.\arxiv{ We conclude with discussion and future directions in \cref{sec:discussion}.} All proofs are deferred to the appendix unless otherwise stated.
}

\section{Problem Setting}
\label{sec:setting}

We consider an episodic finite-horizon reinforcement learning framework, with $H\in\bbN$ denoting the horizon. A Block MDP  $\cM=(\cX,\cS,\cA,T,q)$ consists of an \emph{observation space} $\cX$, \emph{latent state space} $\cS$, \emph{action space} $\cA$, \emph{latent space transition kernel} $T:\cS\times\cA\to\Delta(\cS)$, and \emph{emission distribution} $q:\cS\to\Delta(\cX)$ \citep{du2019latent}. For each layer $h\in\brk{H}$, the \emph{latent state} $\s_h\in\cS$ evolves in a Markovian fashion based on the agent's action $\a_h\in\cA$ via \icml{$ \s_{h+1} \sim{} T(\cdot\mid{}\s_h,\a_h)$, }\arxiv{
\begin{align}
  \label{eq:bmdp_latent}
  \s_{h+1} \sim{} T(\cdot\mid{}\s_h,\a_h),
\end{align}
}
with $\s_1\sim{}T(\cdot\mid{}\emptyset)$, where $T(\cdot\mid{}\emptyset)$ denotes the initial state distribution. The latent state is not observed directly. Instead, we observe \emph{observations} $\x_h\in\cX$ generated by the emission process
\begin{align}
  \x_{h} \sim {} q(\cdot\mid{}\s_h).\nn 
\end{align}
We assume that the latent space $\cS$ and action space $\cA$ are finite, with $S\ldef{}\abs{\cS}$ and $A\ldef{}\abs{\cA}$, but the observation space $\cX$ may be large (with $\abs{\cX}\gg\abs{\cS}$) or potentially infinite. The most important property of the BMDP model, which facilitates sample-efficient learning, is \emph{decodability}:
	\begin{align}
          \supp q(\cdot \mid s)\cap\supp q(\cdot \mid s')=\emptyset, \quad\forall s'\neq s\in\cS.
   \nn 
	\end{align}
        Decodability implies that latent states can be uniquely recovered from observations. In particular, there exists a (unknown to the agent) \emph{decoder} $\phi_\star\colon \cX \rightarrow \cS$ such that $\phi_\star(\x_h)=\s_h$ a.s. for all $h\in\brk{H}$.

        To simplify presentation and keep notation compact, we assume that the BMDP $\cM$ is \emph{layered} in the sense that $\cS = \cS_1\cup \dots\cup  \cS_H$ for $\cS_i \cap \cS_j=\emptyset$ for all $i\neq j$, where $\cS_h\subseteq \cS$ is the subset of states in $\cS$ that are reachable at layer $h\in[H]$. This comes with no loss of generality (up to dependence on $H$), as one can always augment the state space to include the layer index. We also define $\cX_h \coloneqq \bigcup_{s\in \cS_h} \supp q(\cdot \mid s)$, for all $h\in [H]$, and note that by decodability, we have that $\cX_i \cap \cX_j=\emptyset$, for all $i\neq j$.

        \paragraph{Online reinforcement learning and reward-free exploration}
        We consider the standard \emph{online reinforcement learning} framework in which the underlying BMDP $\cM$ is unknown, but the learning agent can interact with it by repeatedly executing a policy $\pi:\cX\to\cS$ (or, a potentially non-Markovian policy, as we will consider in the sequel) and observing the resulting trajectory $(\x_1,\a_1),\ldots,(\x_H,\a_H)$. We do not assume that a reward function is given. Instead, we aim to perform the more general problem of \emph{reward-free exploration}, which entails learning a collection of policies that covers the latent state space to the greatest extent possible \citep{du2019latent,misra2020kinematic,efroni2021provably}.

        In more detail, we consider the reward-free exploration task of learning an \emph{approximate policy cover}, which is a collection of policies which can reach any latent state with near-optimal probability. To formalize this notion, for $s\in\cS_h$, we let $d^{\pi}(s)\ldef\bbP^{\pi}\brk{\s_h=s}$ denote the probability of reaching state $s$ when executing a policy $\pi$, and let $\Pim\coloneqq \left\{\pi \colon \bigcup_{h=1}^H  \cX_h \rightarrow \cA\right\}$ be the set of all Markovian policies.
		\begin{definition}[Approximate policy cover]
		\label{def:polcover101}
                A collection of policies $\Psi$ is an $(\alpha,\veps)$-policy cover for layer $h$ if for all  $s\in \cS_h$ such that $\max_{\pi \in \Pim} d^{\pi}(s)  \geq \veps$, we have
		\begin{align}
			\max_{\pi \in \Psi} d^{\pi}(s)\geq  \alpha \cdot \max_{\pi' \in \Pim} d^{\pi'}(s). \nn 
		\end{align}
              \end{definition}
Informally, an $(\alpha,\veps)$-policy cover $\Psi$ has the property that for every state $s\in\cS$ that can be reached with probability at least $\veps$, there exists a policy in $\Psi$ that reaches it with probability at least $\alpha\cdot\veps$. For our results, $\alpha$ will be a numeric constant (say, $1/2$), and $\veps$ will be a parameter to the algorithm. We show (\pref{sec:rewardsetting}) that given access to such a policy cover, it is possible to optimize any downstream reward function to precision $\bigoh(\veps)$.

\arxiv{
\begin{remark}
  When $\veps=0$, \pref{def:polcover101} recovers the policy cover definition used in \citep{misra2020kinematic}. The relaxed notion of a policy cover in \cref{def:polcover101}, which allows one to ``sacrifice'' states that are hard to reach with any policy (i.e.~those states $s$ for which $\max_{\pi \in \Pim} d^{\pi}(s)  < \veps$), is natural in our setting, as we do not assume that all states can be reached with some minimum probability.
\end{remark}
}
\icml{
  When $\veps=0$, \pref{def:polcover101} recovers the policy cover definition used in \cite{misra2020kinematic}. The relaxed notion of a policy cover in \cref{def:polcover101}, which allows one to ``sacrifice'' states that are hard to reach with any policy (i.e.~those states $s$ for which $\max_{\pi \in \Pim} d^{\pi}(s)  < \veps$), is natural in our setting, as we do not assume that all states can be reached with some minimum probability.
}

        \paragraph{Function approximation}
To provide sample-efficient learning guarantees, we make use of function approximation. In particular, we do not assume that the true decoder $\phistar:\cX\to\cS$ is known to the learner and, as in prior work \citep{du2019latent,misra2020kinematic, zhang2022efficient}, assume access to a \emph{decoder class} $\Phi\subseteq(\cX\to\cS)$ that contains $\phistar$.
	\begin{assumption}[Realizability]
		\label{assum:real}
		The decoder class $\Phi\subseteq \{\phi \colon \cX \rightarrow\cS\}$ contains the true decoder $\phistar$.
              \end{assumption}
              The class $\Phi$ captures the learner's prior knowledge about the environment, and may consist of neural networks or other flexible function approximators. To simplify presentation, we assume that $\Phi$ is finite; as our results only invoke standard uniform convergence arguments, extension to infinite classes and other notions of statistical capacity is straightforward \citep{misra2020kinematic}.
              We aim to learn an $(\alpha,\veps)$-policy cover (for constant $\alpha$) using a number of episodes/trajectories (``sample complexity'') that scales with\arxiv{
              \[
                \poly(H,S,A,\log\abs{\Phi}) \cdot 1/\veps^{2}.
                \]}
\icml{$\poly(H,S,A,\log\abs{\Phi}) \cdot 1/\veps^{2}$.
}
Notably, this guarantee depends on the number of latent states $S$ and the complexity $\log\abs{\Phi}$ for the decoder class, but does not explicitly depend on the size of the observation space $\cX$.

\subsection{Preliminaries}
\renewcommand{\o}{\bm{o}}
We proceed to introduce additional notation required to present our main results. Most important will be the notion of \emph{partial policies}, both Markovian and non-Markovian. For any $n,m\in\mathbb{N}$, we denote by $[m\ldotst{}n]$ the integer interval $\{m,\dots, n\}$. We also let $[n]\coloneqq [1\ldotst{}n]$. Further, for any sequence of objects $o_1, o_2,\dots$, we define $o_{m:n}\coloneqq (o_{i})_{i\in[m \ldotst n]}$.

\paragraph{Partial policies} A \emph{partial policy} is a policy that is defined only over a contiguous subset of layers $\brk{l\ldotst{}r}\subseteq\brk{H}$. We let $\Pim^{l:r} \coloneqq \left\{\pi \colon \bigcup_{h=l}^r  \cX_h \rightarrow \cA\right\}$ be the set of \emph{Markovian partial} policies that are defined over layers $l$ to $r$. For a policy $\pi\in\Pim^{l:r}$ and layer $h\in\brk{l\ldotst{}r}$, $\pi(x_h)$ denotes the action taken by the policy at layer $h$ when $x_h\in\cX_h$ is the current observation. We will use the notation $\Pim \equiv \Pim^{1:H}$.

We also consider \emph{non-Markov} (history-dependent) partial policies. For $1\leq l\leq r\leq H$, we let 
	\begin{align} \Pinm^{l:r} \coloneqq \left\{\pi : \bigcup_{h=l}^r (\cX_{l}\times \dots \times \cX_h) \rightarrow \cA\right\}\nn 
\end{align}
denote the set of non-Markovian partial policies that are defined over layers $l$ to $r$. The action of a partial policy $\pi\in\Pinm^{l:r}$ is only defined for layers $h\in\range{l}{r}$, but may depend on the entire history of observations $x_{l:h}=x_l,\ldots,x_h$ beginning from layer $l$. In particular, for layer $h\in\range{l}{r}$, $\pi(x_{l:h})$ denotes the policy's action when $x_{l:h}\in\cX_l\times\cdots\times\cX_h$ is the history. \icml{For any $1\leq t\leq h\leq H$, and any pair of partial policies $\pi \in \Pinm^{1:t-1}, \pi'\in \Pinm^{t:h}$, we let $\pi \circ_t \pi'$ be the partial policy in $\Pinm^{1:h}$ that satisfies $(\pi \circ_t \pi')(x_{1:\tau}) = \pi(x_{1:\tau})$ for all $\tau<t$ and $(\pi \circ_t \pi')(x_{1:\tau}) = \pi'(x_{t:\tau})$ for all $\tau \in [t\ldotst h]$. We define $\pi \circ_t \pi'$ similarly when $\pi \in \Pinm^{1:\tau}$ for $\tau\geq t$.}

 \arxiv{
 \paragraph{Composition of partial policies} For any $1\leq t\leq h\leq H$, and any pair of partial policies $\pi \in \Pinm^{1:t-1}, \pi'\in \Pinm^{t:h}$, we let $\pi \circ_t \pi'$ be the partial policy in $\Pinm^{1:h}$ that satisfies $(\pi \circ_t \pi')(x_{1:\tau}) = \pi(x_{1:\tau})$ for all $\tau<t$ and $(\pi \circ_t \pi')(x_{1:\tau}) = \pi'(x_{t:\tau})$ for all $\tau \in [t\ldotst h]$. We define $\pi \circ_t \pi'$ similarly when $\pi \in \Pinm^{1:\tau}$ for $\tau\geq t$.

\paragraph{BMDP notation and occupancy measures} 
Given any policy $\pi $ and BMDP $\cM$, we denote by $\P^{\cM,\pi}$ the probability law over $\{(\s_h , \x_h, \a_h)\colon h\in[H]\}$ induced by executing $\pi$ in $\cM$. We let $\E^{\cM,\pi}$ denote the corresponding expectation. For any $h \in [H]$ and $s\in \cS_h$, we denote by $d^{\cM,\pi}(s) \coloneqq \P^{\cM,\pi}[\s_h=s]$ the \emph{occupancy} of $s$ under $\pi$. We drop the $\cM$ superscript when the underlying BMDP is clear from context. %

 \paragraph{Further notation}
 Given a set of partial policies $\Psi \coloneqq \{\pi^{(i)}\colon i \in [N]\}$, we denote by $\unif(\Psi)$ the random partial policy obtained by sampling $\bi\sim \unif([N])$ and playing $\pi^{(\bi)}$. We overload notation slightly and denote by $\unifa$ the random policy that plays actions in $\cA$ uniformly at random at all layers. We use the notation $\wtilde{O}(1)$ to hide polylogarithmic factors in $H,S,A, \ln |\Phi|$, and $\veps^{-1}$.
}

\icml{
	\paragraph{Further notation} 
	Given any policy $\pi $ and BMDP $\cM$, we denote by $\P^{\cM,\pi}$ the probability law over $\{(\s_h , \x_h, \a_h)\colon h\in[H]\}$ under the process induced by executing $\pi$ in $\cM$. We let $\E^{\cM,\pi}$ denote the corresponding expectation. For any $h \in [H]$ and $s\in \cS_h$, we denote by $d^{\cM,\pi}(s) \coloneqq \P^{\cM,\pi}[\s_h=s]$ the \emph{occupancy} of $s$ under $\pi$. We drop the $\cM$ superscript when clear from the context. Given a set of partial policies $\Psi \coloneqq \{\pi^{(i)}\colon i \in [N]\}$, we denote by $\unif(\Psi)$ the random partial policy obtained by sampling $\bi\sim \unif([N])$ and playing $\pi^{(\bi)}$. We overload notation slightly and denote by $\unifa$ the random policy that plays actions in $\cA$ uniformly at random at all layers. We use the notation $\wtilde{O}(1)$ to hide poly-logarithmic factors in $H,S,A, \ln |\Phi|$, and $\veps^{-1}$.
}

\arxiv{	\section{Multi-Step Inverse Kinematics: Algorithm and Main Results}	}
\icml{	\section{Algorithm and Main Results}	}
	\label{sec:multi}

We now present our algorithm, \musik, and prove that it efficiently learns a policy cover with rate-optimal $\mathrm{poly}(H, S,A,\ln |\Phi|) \cdot 1/\veps^2$ sample complexity. First, in \pref{sec:challenges}, we highlight the challenges faced in achieving similar guarantees with existing approaches, with an emphasis on difficulties removing a statistical assumption known as \emph{reachability}. With this out of the way, we introduce the \musik algorithm (\pref{sec:musik}) and give an overview of its main performance guarantee and key features (\pref{sec:main_theorem}).
\arxiv{Finally, in \cref{sec:rewardsetting}, we show how to use the policy cover produced by \musik to perform reward-based reinforcement learning with any reward function of interest.
}
\icml{Extensions to reward-based reinforcement learning are deferred to \cref{sec:rewardsetting}.
}

\subsection{Challenges and Related Work}
\label{sec:challenges}

For the Block MDP model, the optimal sample complexity to learn an $\veps$-optimal policy or learn an $(\alpha,\veps)$-approximate policy cover for constant $\alpha$ scales with $1/\veps^2$.\footnote{An $\bigom(1/\veps^2)$ lower bound on the sample complexity follows from standard lower bounds for tabular RL \citep{jin2020provably}.} Previous approaches---both for reward-free and reward-based exploration---either achieve this rate, but are not computationally efficient, or only achieve it under additional statistical assumptions that may not be satisfied in general. To motivate the need for new algorithm design and analysis ideas, let us highlight where these challenges arise.

Existing algorithms can be broken into two families, \emph{optimistic algorithms}, and algorithms that are not optimistic, but require \emph{reachability} conditions. Optimistic algorithms use the principle of \emph{optimism in the face of uncertainty} to drive exploration. Implementing optimism in the BMDP setting is challenging because the latent states are not observed, which prevents the naive application of state-action exploration bonuses found in tabular RL \citep{azar2017minimax,jin2018q}. An alternative is to appeal to \emph{global optimism}, which computes an optimistic policy by optimizing over a \emph{version space} of plausibly-optimal value functions. This approach enjoys rate-optimal sample complexity \citep{jiang2017contextual,du2021bilinear,jin2021bellman}, but cannot be implemented efficiently in general because it requires searching for value functions that satisfy non-convex constraints at all layers $h\in\brk{H}$ simultaneously (``globally'') \citep{dann2018oracle}.

As a tractable replacement for global optimism, a more recent line of algorithms implement optimism using a \emph{plug-in} approach which computes layer-wise bonuses with respect to an \emph{estimated decoder}. First, \citet{uehara2022} show that under the stronger assumption that the learner has access to a realizable \emph{model class}, it is possible to learn a decoder for which the plug-in approach attains rate-optimal sample complexity; this observation, while interesting, falls short of a model-free guarantee that scales only with $\log\abs{\Phi}$. More recently, \citet{zhang2022efficient} observed that similar results can be achieved with only decoder realizability by appealing to a certain min-max representation learning objective.\footnote{\citet{modi2021model} employ a similar representation learning objective, but require a minimum reachability assumption.} However, this objective involves a form of adversarial training that increases the sample complexity, leading to a final guarantee that scales with $1/\veps^{4}$ instead of $1/\veps^2$.

Given the challenges faced by optimistic approaches, an alternative is to do away with optimism entirely. Algorithms from this family \citep{du2019provably,misra2020kinematic} proceed in a forward fashion: They first solve a representation learning objective which enables building a policy cover for layer $2$. Then, using this policy cover, they explore to collect data that can be used to solve a similar representation learning objective for layer $3$, then use this to build a policy cover for layer $3$, and so on. A-priori, a natural concern is that the myopic nature of these step-by-step approaches might lead to approximation errors that compound exponentially as a function of the horizon $H$. To avoid, this, existing work \citep{du2019provably,misra2020kinematic} makes a \emph{minimum reachability assumption}.
	\begin{definition}[Minimum reachability]
\label{def:reachability}
There exists $\etamin>0$ such that for all $h\in[H]$ and $s\in \cS_{h}$, there exists $\pi \in\Pim$ such that $d^{\pi}(s)\geq \etamin$.
\end{definition}
Reachability is a useful assumption because it ensures that for every possible state $s$ in the latent space, we can learn a policy that can reach $s$ with sufficiently high probability (say, with probability at least $\etamin/2$), which prevents errors from cascading as one moves forward from layer $h$ to layer $h+1$. The best algorithm from this family, \homer, attains sample complexity that is proportional to $1/\veps^2$, but scales inversely proportional to the reachability parameter $\etamin$, and provides no guarantees when $\etamin=0$. Prior to our work, it was not known whether any algorithm based on the non-optimistic layer-by-layer approach could succeed at all in the absence of reachability, let alone achieve rate-optimal sample complexity. \arxiv{We refer to \pref{tb:resultscomp} for a summary}\icml{We refer to \pref{tb:resultscomp} in \pref{sec:omitted} for a summary}.

\subsection{The \musik Algorithm}
\label{sec:musik}
\arxiv{
        \begin{algorithm}[ht]
          \caption{$\musik$: Multi-Step Inverse Kinematics}
          \label{alg:GenIk}
          \begin{algorithmic}[1]\onehalfspacing
            \Require Decoder class $\Phi$. Number of samples $n$.
            \State Set $\Psi\ind{1}= \emptyset$.
            \For{$h=2\ldots, H$} 
            \State Let
            $\Psi\ind{h}=\ikdp(\Psi\ind{1},\dots,\Psi\ind{h-1},
            \Phi,n)$.\quad\algcommentlight{\pref{alg:IKDP}.}
            \label{line:ikdp1}
            \EndFor
            \State \textbf{Return:} Policy covers $\Psi\ind{1},\dots,\Psi\ind{H}$. 
          \end{algorithmic}
	\end{algorithm}
}
\icml{
        \begin{algorithm}[ht]
	\caption{$\musik$: Multi-Step Inverse Kinematics}
	\label{alg:GenIk}
	\begin{algorithmic}[1]\onehalfspacing
		\Require Decoder class $\Phi$. Number of samples $n$.
		\State Set $\Psi\ind{1}= \emptyset$.
		\For{$h=2\ldots, H$} 
		\State \hbox{Let
		$\Psi\ind{h}=\ikdp(\Psi\ind{1},\dots,\Psi\ind{h-1},
		\Phi,n)$\hfill\algcommentlight{Alg.\,\ref{alg:IKDP}}}
		\label{line:ikdp1}
		\EndFor
		\State \textbf{Return:} Policy covers $\Psi\ind{1},\dots,\Psi\ind{H}$. 
	\end{algorithmic}
\end{algorithm}
}

Our main algorithm, \musik, is presented in \pref{alg:GenIk}. $\musik$ performs reward-free exploration, iteratively building approximate
        policy covers $\Psi\ind{1},\ldots,\Psi\ind{H}$ for layers $h=1,\dots, H$. The algorithm first gathers data from the initial state distribution, and uses this to learn a policy cover $\Psi\ind{2}$ for layer $2$ (we adopt the convention that $\Psi\ind{1}=\emptyset$). The algorithm then collects data using $\Psi\ind{2}$, and uses this to build an approximate
        policy cover $\Psi\ind{3}$ for layer $3$, and so on. Once layer $H$ is reached, the algorithm returns $\Psi\ind{1},\ldots,\Psi\ind{H}$. The crux of the \musik algorithm is a subroutine, \ikdp (Inverse Kinematics for Dynamics Programming, \pref{alg:IKDP}) which, at each step $h$, makes use of the previous policy covers $\Psi\ind{1},\ldots,\Psi\ind{h-1}$ to compute the policy cover $\Psi\ind{h}$. In what follows, we give a detailed overview of \ikdp.

        \paragraph{The \ikdp subroutine}
        For each $h\in\brk{H}$, the \ikdp subroutine (\pref{alg:IKDP}) uses the policy covers $\Psi\ind{1},\ldots,\Psi\ind{h-1}$ to construct the policy cover $\Psi\ind{h}$ for layer $h$ in a backwards fashion inspired by dynamic programming: Beginning from layer $h-1$, the algorithm builds a collection of partial policies $\crl{\hat
          \pi\ind{i,h-1}}_{i\in\brk{S}}\in \Pim^{h-1:h-1}$ using data collected by rolling in with $\Psi\ind{h-1}$; each policy $\pihat\ind{i,h-1}$ is responsible for targeting a single latent state in layer $h$. The algorithm then moves back one layer, and constructs a collection
        $\crl{\pihat \ind{i,h-2}}_{i\in\brk{S}}\in \Pim^{h-2:h-1}$ using data collected by rolling in with $\Psi\ind{h-2}$ and rolling out using the collection $\crl{\hat
          \pi\ind{i,h-1}}_{i\in\brk{S}}$. This process is repeated until the first layer is reached, and the final collection of policies $\Psi\ind{h}=\crl{\pihat \ind{i,1}}_{i\in\brk{S}}$ is returned. The key invariant maintained throughout this process is that for all layers $t\in\brk{h-1}$, for every latent state $s\in\cS_h$, there exists a partial policy in the set $\crl{\pihat \ind{i,t}}_{i\in\brk{S}}$ that reaches $s$ with near-optimal probability starting from layer $t$ (in a certain average-case sense).

        \paragraph{Multi-step inverse kinematics objective}
        For each layer $t\in\brk{h-1}$, given the partial policies $\crl{\pihat \ind{i,t+1}}_{i\in\brk{S}}$ from the previous backward step, \ikdp computes the collection $\crl{\pihat \ind{i,t}}_{i\in\brk{S}}$ by appealing to a regression objective (\pref{line:inversenew2}) based on \emph{multi-step inverse kinematics} \citep{lamb2022guaranteed}. To motivate the approach, we recall that a significant challenge faced in the BMDP setting is that the latent states are not directly observed. Were not the case, it would be possible to build a policy cover by directly optimizing ``visitation'' reward functions of the form $\br_h^{(s)} \coloneqq \mathbb{I}\{\s_h=s\}$ for each $s\in\cS_h$ (this can be accomplished using standard methods such as \psdp \citep{bagnell2003policy,misra2020kinematic}). As an alternative, one can think of $\ikdp$ as constructing proxies for the state-action value functions ($Q$-functions) associated with the (unobserved) reward functions $\br_h^{(s)}$ for each $s\in\cS_h$. These proxies are constructed using the objective in \cref{line:inversenew2}, which involves predicting actions from observations at different layers (multi-step inverse kinematics).

        In more detail, for each backward iteration $t\in[h-1]$, $\ikdp$ samples $\pi \sim \Psi\ind{t}$, executes $\pi$ up to layer $t$, plays a random action $\a_t \sim \unifa$, then selects a random index $\bi_t\sim \unif([S])$ and executes $\pihat ^{(\bi_t,t+1)} \in \Pinm^{t+1:h-1}$ from layer $t+1$ onward. The regression objective in \cref{line:inversenew2} then uses this data to estimate the conditional density for the pair $(\a_t, \bi_t)$, conditioned on the observations $\x_t$ and $\x_h$. This estimate for the conditional density acts as a proxy for the $Q$-functions associated with the unobserved visitation reward functions $\br_h^{(s)}$ described above. Thanks to the decodability property of the BMDP model, it can be shown that the Bayes-optimal solution to the regression objective in \cref{line:inversenew2} depends on observations only through latent states. This allows us to parameterize the objective using the decoder class $\Phi$, which is key to achieving low sample complexity.

        \paragraph{Policy composition}
        After solving the multi-step inverse kinematics objective in \cref{line:inversenew2}, \ikdp uses the resulting decoder $\phihat\ind{t}$ and function $\fhat\ind{t}$ to build the set of partial policies $\crl{\pihat \ind{i,t}}_{i\in\brk{S}}$ from the set $\crl{\pihat \ind{i,t+1}}_{i\in\brk{S}}$ produced at the previous backward step (\pref{line:second,line:nonmark}). Here, the challenge is that there is no way to know which policy $\crl{\pihat \ind{i,t+1}}_{i\in\brk{S}}$ is responsible for targeting a given state $s\in \cS_h$ due to non-identifiability. We address this using a \emph{non-Markovian} policy construction in \pref{line:second,line:nonmark}, which we now describe.

        Recall that the objective in \cref{line:inversenew2} predicts both actions and \emph{indices of roll-out policies}. Predicting the indices of roll-out policies offers a mechanism to associate partial policies at successive layers. To do so, \pref{line:second} of \ikdp defines
        \[
        (\ahat\ind{i,t}(x), \iotahat\ind{i,t}(x))=
        \argmax_{(a,j)} \fhat\ind{t}((a,j) \mid
        \phihat\ind{t}(x),i), \quad x\in \cX_t.
      \]
      One should interpret $j=\iotahat\ind{i,t}(x)$ as the \emph{most likely} (or most closely associated) roll-out policy $\pihat\ind{j,t+1}$ when the (decoded) latent state at layer $h$ is $i\in\brk{S}$ and $x\in\cX_t$ is the current observation at layer $t$. Meanwhile, the action $\ahat\ind{i,t}(x)$ (approximately) maximizes the probability of reaching $i$ if we roll out with $\pihat\ind{j,t+1}$.
      With this in mind, the composition rule in \pref{line:nonmark} constructs $\pihat\ind{i,t}$ via
        \[
        \hat  \pi^{(i,t)}(x_{t:\tau})\coloneqq\left\{
        \begin{array}{ll}
        	\ahat\ind{i,t}(x_t),&\quad \tau=t,\;\;
        	x_t\in\cX_t,\\
        	\pihat
        	^{(\iotahat\ind{i,t}(x_t),t+1)}(x_{t+1:\tau}),&\quad\tau\in\range{t+1}{h-1},\;\;
        	x_{t:\tau}\in \cX_t\times \dots \times\cX_{\tau}.
        \end{array}
        \right.
      \]
      That is, for layers $t+1,\ldots,h-1$, this construction follows the policy $\pihat\ind{\iotahat\ind{i,t}(\x_t),t+1}$ which---per the discussion above---is most associated with the decoded state $i\in\brk{S}$. At layer $t$, we select $\a_t=\ahat\ind{i,t}(\x_t)$, maximizing the probability of reaching the decoded latent state $i\in\brk{S}$ when we roll-out with $\pihat\ind{\iotahat\ind{i,t}(\x_t),t+1}$. This construction, while intuitive, is non-Markovian, since for layers $t+1$ and onward the policy depends on $\x_t$ through $\iotahat\ind{i,t}(\x_t)$.
        
        \loose

We refer to \pref{sec:overview} for a detailed overview of the analysis ideas behind \ikdp, as well as further intuition. %

\arxiv{
	\begin{algorithm}[t]
		\caption{$\texttt{IKDP}$:  Inverse Kinematics for
                  Dynamic Programming }
		\label{alg:IKDP}
		\begin{algorithmic}[1]\onehalfspacing
                  \Require ~
                  \begin{itemize}[leftmargin=*]
                    \item Approximate covers
                      $\Psi\ind{1},\dots,\Psi\ind{h-1}$ for layers $1$
                      to $h-1$, where $\Psi\ind{t} \subseteq\Pinm^{1:t-1}$.
                    \item Decoder class
                      $\Phi$.
                      \item Number of samples $n$.
                  \end{itemize}
			\For{$t=h-1,\ldots, 1$} \label{line:mainiter}
			\State $\cD\ind{t}\leftarrow \emptyset$.
                        \Statex[1]\algcommentbiglight{Collect data by rolling in with
                          policy cover and rolling out with partial policy}
			\For{$n$ times}
			\State Sample $\bi_t \sim \unif([S])$. \label{line:index}
			\State Sample $(\x_t, \a_t, \x_{h})\sim \unif(\Psi\ind{t}) \circ_t \unifa\circ_{t+1} \pihat ^{(\bi_t,t+1)}$. \algcommentlight{with $\pihat ^{(i,h)} \coloneqq \unifa$, for all $i\in[S]$} \label{line:action}
			\State $\cD\ind{t} \leftarrow \cD\ind{t}\cup \{(\bi_t,\a_t, \x_t, \x_{h})\}$.
			\EndFor
                        \Statex[1] \algcommentbiglight{Inverse kinematics}
			\State Solve
                        \begin{equation}
                    \fhat\ind{t}, \phihat\ind{t}\gets \argmax_{f:  [S]^2 \rightarrow
                   \Delta(\cA\times [S]) , \phi\in  \Phi}
                 \sum_{(j,a,x,x')\in \cD\ind{t}} \ln   f( (a,j) \mid
                 \phi(x), \phi(x')).
                 \end{equation}
                 \label{line:inversenew2}
                 \Statex[1] \algcommentbiglight{Update partial policy cover}
                 \State For each $i\in \brk{S}$, define
                 \begin{equation}
                 (\ahat\ind{i, t}(x), \iotahat\ind{i, t}(x))=
                   \argmax_{(a,j)} \fhat\ind{t}((a,j) \mid
                   \phihat\ind{t}(x),i), \quad x\in \cX_t.
                   \end{equation}
                 \label{line:second}
			\State  For each $i\in [S]$, define $\hat
                        \pi^{(i, t)}\in \Pinm^{t:h-1}$  via
\label{line:nonmark}
                        \begin{align}
				\hat  \pi^{(i, t)}(x_{t:\tau})\coloneqq \left\{
                          \begin{array}{ll}
                            \ahat\ind{i, t}(x_t),&\quad \tau=t,\;\;
                                                  x_t\in\cX_t,\\
                            \pihat
                            ^{(\iotahat\ind{i, t}(x_t),t+1)}(x_{t+1:\tau}),&\quad\tau\in\range{t+1}{h-1},\;\;
                                                                   x_{t:\tau}\in \cX_t\times \dots \times\cX_{\tau}.
                          \end{array}
\right.\label{eq:partial}
\end{align}
			\EndFor
			\State \textbf{Return:} Policy cover
                        $\Psi\ind{h}=\{\hat\pi^{(i,1)}\colon i
                        \in[S]\}\subseteq \Pinm^{1:h-1}$ for layer $h$. \label{line:line}
		\end{algorithmic}
              \end{algorithm}
              
       }

\icml{
	\begin{algorithm}[t]
		\caption{$\texttt{IKDP}$:  Inverse Kinematics for
			Dynamic Programming }
		\label{alg:IKDP}
		\begin{algorithmic}[1]\onehalfspacing
			\Require ~ Approximate covers
				$\Psi\ind{1},\dots,\Psi\ind{h-1}$ for layers $1$
				to $h-1$, where $\Psi\ind{t} \subseteq\Pinm^{1:t-1}$. Decoder class
				$\Phi$. Number of samples $n$.
			\For{$t=h-1,\ldots, 1$} \label{line:mainiter}
			\State $\cD\ind{t}\leftarrow \emptyset$.
			\For{$n$ times}
			\State Sample $\bi_t \sim \unif([S])$ and set $\pihat = \pihat ^{(\bi_t,t+1)}$. \label{line:index}
			\State Sample $(\x_t, \a_t, \x_{h})\sim \P^{\unif(\Psi\ind{t}) \circ_t \unifa\circ_{t+1}\pihat }$. \label{line:action}
			\State $\cD\ind{t} \leftarrow \cD\ind{t}\cup \{(\bi_t,\a_t, \x_t, \x_{h})\}$.
			\EndFor
			\Statex[1] \algcommentbiglight{Inverse kinematics}
			\State 	\label{line:inversenew2} Compute the solution $(\fhat\ind{t},
			\phihat\ind{t})$ of the problem
				\vspace{-0.3cm}
			\begin{align}
			\max_{f\in \cF, \phi\in  \Phi}
				\sum_{(j,a,x,x')\in \cD\ind{t}} \ln   f( (a,j) \mid
				\phi(x), \phi(x')),
			\end{align}
			\vspace{-0.5cm}
		\Statex[1] where $\cF\coloneqq [S]^2 \rightarrow
		\Delta(\cA\times [S])$. 
			\Statex[1] \algcommentbiglight{Update partial policy cover}
			\State \label{line:second} For each $i\in \brk{S}$ and $x\in \cX_t$ define
				\vspace{-0.3cm}
			\begin{align}
				& (\ahat\ind{i, t}(x), \iotahat\ind{i, t}(x))\notag\\ &\quad  =
			\argmax{}_{(a,j)} \fhat\ind{t}((a,j) \mid
				\phihat\ind{t}(x),i).
			\end{align}
		\vspace{-0.8cm}
			\State \label{line:nonmark} For each $i\in [S]$, $\tau \in \range{t+1}{h-1}$, $x_{t:\tau}\in  \bigtimes_{k=t}^\tau \cX_{k}$, 
			\Statex[1] define $\hat
			\pi^{(i, t)}\in \Pinm^{t:h-1}$ via $\hat  \pi^{(i, t)}(x_{t})=	\ahat\ind{i, t}(x_t)$
			\Statex[1] and $\hat  \pi^{(i, t)}(x_{t:\tau})=\pihat
			^{(\iotahat\ind{i, t}(x_t),t+1)}(x_{t+1:\tau})$.
		\EndFor
		\State \textbf{Return:} Layer $h$ cover
		$\Psi\ind{h}=\{\hat\pi^{(i,1)}\}_{i
		\in[S]}\subseteq \Pinm^{1:h-1}.$ %
	\end{algorithmic}
\end{algorithm}

}

\paragraph{On inverse kinematics}

\musik can be viewed as generalizing the notion of \emph{one-step inverse kinematics} to multiple steps. One-step inverse kinematics, which aims to predict the action $\a_h$ from $\x_h$ and $\x_{h+1}$, has been explored in a number of empirical works \citep{pathak2017curiosity,badia2020agent57,baker2022video,bharadhwaj2022information}. In theory, however, it can be shown that this approach can fail to meaningfully recover latent state information \citep{misra2020kinematic,efroni2021provably}. In particular, it is prone to incorrectly merging latent states with different dynamics. Multi-step inverse kinematics generalizes one-step inverse kinematics by predicting $\a_h$ from $\x_h$ and $\x_{h'}$ for all possible choices for $h'>h$. Recent work of \citet{lamb2022guaranteed} observed that---in the infinite-data limit---multi-step inverse kinematics can rectify the issues with one-step IK, and enjoys other benefits including robustness to exogenous information. Our work is the first to provably combine multi-step inverse kinematics with systematic exploration to derive finite-sample guarantees.\footnote{The work of \cite{efroni2021provably} also makes use of multi-step inverse models, but is limited to deterministic systems. \citet{mhammedi2020learning} also uses a form of multi-step inverse kinematics in the context of linear control with rich observations, but their approach is specialized to the linear setting.}

\begin{remark}
  $\ikdp$ also bears some similarity to the $\psdp$ algorithm (see \citet{bagnell2003policy,misra2020kinematic} and \cref{alg:PSDP}), and uses the principle of dynamic programming in a similar fashion. Unlike $\psdp$, $\ikdp$ does not require feedback from an external reward function, and can be thought of as automatically discovering its own reward function to drive exploration.
\end{remark}

        \paragraph{Efficient implementation}
        \musik is practical, and is computationally efficient whenever the standard log-loss conditional density estimation problem\arxiv{
        \[
\fhat\ind{t} ,
                 \phihat\ind{t}\gets \argmax_{f\in\cF , \phi\in \Phi}
                 \sum_{(j,a,x,x')\in \cD\ind{t}} \ln   f( (a,j) \mid
                 \phi(x), \phi(x')).
        \]
    }
on \pref{line:inversenew2} of \ikdp can be solved efficiently for the decoder class $\Phi$ of interest.\arxiv{\footnote{We also note that the log-loss conditional density estimation objective in $\musik$ can be replaced by standard supervised square-loss regression objectives without changing the guarantee of $\musik$ in \cref{thm:policycover}.}}
      In practice, $\Phi$ and $\cF\ldef{}[S]^2 \rightarrow
        \Delta(\cA\times [S])$ can both be approximated with neural networks or other flexible function classes, and the conditional density estimation problem in \cref{line:inversenew2} can be solved by appealing to stochastic gradient descent or other off-the-shelf training procedures; this is the approach taken in our experiments (\pref{sec:experiments}).%

Let us also remark on the complexity of representing and executing the partial policies $\crl*{\pihat\ind{i,t}\colon  i\in\brk{S},t\in\brk{h-1}}$ computed in \pref{line:nonmark} of \ikdp. These policies are non-Markovian, which presents a problem at first glance, since general non-Markovian policies in a horizon-$H$ MDP with $S$ states require a table of size $S^{H}$ to represent. Fortunately, the non-Markovian policies in \ikdp are quite structured, and can be represented and executed with runtime and memory complexity that is polynomial in $H$ instead of exponential;\arxiv{ see \cref{alg:gen} for pseudocode}\icml{ see \cref{alg:gen} in \cref{sec:omitted} for pseudocode}. In particular, the partial policies for layer $h$ can be fully represented using $O(H)$ memory via the collection of functions $\{(\hat f\ind{t}, \hat \phi\ind{t})\colon t \in[h-1]\}$ learned in \cref{line:inversenew2} of \cref{alg:IKDP} (assuming that, for $t\in[h-1]$, storing $(\fhat\ind{t},\phihat\ind{t})$ requires $O(1)$ memory). One can then execute the partial policies to generate a trajectory using $O(H S A)$ runtime,\arxiv{\footnote{$O(SA)$ work is required to compute $\argmax_{(a,j)} \fhat\ind{\tau}((a,j) \mid  k,i)$, for $\tau \in[t\ldotst h-1]$ and $i,k \in[S]$. This is needed in Line \ref{line:argmax} of \cref{alg:gen}}} \icml{(assuming that evaluating $\phihat\ind{t}(x)$ costs $O(1)$ units of time).}\arxiv{assuming that evaluating $\phihat\ind{t}(x)$ costs $O(1)$ units of time for all $t\in[h-1]$ and $x\in \cX_t$.}

\subsection{Main Result}
\label{sec:main_theorem}

                We now state the main guarantee for $\musik$ (proven in \cref{app:geniklemmaproof}) and discuss some of its implications. 

	\begin{theorem}[Main theorem for \musik]
		\label{thm:policycover}
		Let $\veps,\delta \in(0,1)$ be given. Suppose that \cref{assum:real} holds, and that $n$ is chosen such that \[
		n \geq \frac{c A^2 S^{10} H^2 \left(S^3 A \ln n+ \ln (|\Phi| H^2/\delta)\right)}{\veps^2},\] for some absolute constant $c>0$ independent of all problem parameters. Then, with probability at least $1-\delta$, the policies $\Psi\ind{1},\dots,\Psi\ind{H}$ produced by \musik (\pref{alg:GenIk}) are $(1/4,\veps)$-policy covers for layers 1 to $H$. The total number of trajectories used by the algorithm is at most
\begin{align}
 \wtilde{O}(1)\cdot 
 \frac{A^2 S^{10} H^4 \left( A S^3 + \ln (|\Phi| H^2/\delta)\right)}{\veps^2}. \label{eq:trajectories} 
\end{align}
	\end{theorem}

        \pref{thm:policycover} is the first sample complexity guarantee for the BMDP setting that 1) is attained by an efficient algorithm, 2) does not scale with the reachability parameter $\etamin$, and 3) attains rate-optimal $1/\veps^2$ sample complexity. Previous efficient BMDP algorithms such as $\texttt{MOFFLE}$ or $\homer$ have sample complexity scaling with $1/\veps^2 \cdot \text{poly}(1/\eta_{\min })$, where $\etamin \coloneqq \min_{s\in\cS} \sup_{\pi \in \Pim} d^{\pi}(s)$ is the reachability parameter, and do not provide guarantees if $\etamin=0$. More recents results \citep{zhang2022efficient} do not require $\etamin>0$, but have suboptimal dependence on $\veps$. We remark that the dependence on the problem-dependent parameters $S$, $A$, and $H$ in our result is loose, and improving this with an efficient algorithm is an interesting open question; other efficient algorithms have similarly loose dependence, per \pref{tb:resultscomp}.

        \paragraph{Practicality}
        As discussed in the prequel, \musik is computationally efficient whenever the standard conditional density estimation problem in \pref{line:inversenew2} of \ikdp can be solved efficiently for the decoder class $\Phi$ of interest, allowing for the use of off-the-shelf models and estimation algorithms; in experiments (\pref{sec:experiments}), we appeal to deep neural networks and stochastic gradient descent.

        From prior work, the only other computationally-efficient (and model-free) algorithm that does not require minimum reachability in BMDPs is $\texttt{BRIEE}$ \citep{zhang2022efficient}.  The log-loss conditional density estimation objective in $\musik$ is somewhat simpler than the min-max representation learning objective in $\texttt{BRIEE}$, with the latter necessitating adversarial training.%

        \arxiv{
        \paragraph{Proof techniques}
        We find it somewhat surprising that \musik attains rate-optimal sample complexity in spite of forgoing optimism. The proof of \pref{thm:policycover}, which we sketch  in \pref{sec:overview},
        has two main components. For the first component, we prove that the multi-step inverse kinematics objective learns a decoder that can be used to drive exploration; this formalizes the intuition in \pref{sec:musik}. With this established, proving that \musik succeeds under minimum reachability (\pref{def:reachability}) is somewhat straightforward, but proving that the algorithm 1) succeeds in  absence of this assumption, and 2) achieves optimal sample complexity is more involved. For this component of the proof, we use a new analysis tool we refer to as an \emph{extended BMDP} which, in tandem with another tool we refer to as a \emph{truncated policy class}, allows one to emulate certain consequences of reachability even when the condition does not hold. These techniques, which we anticipate will find broader use in the analysis of non-optimistic algorithms, appear to be new even for tabular reinforcement learning.
        }

\arxiv{
 \arxiv{
 	\begin{algorithm}[t]
 	\caption{Execute non-Markov partial policy
          produced by $\musik$.}
 	\label{alg:gen}
 	\begin{algorithmic}[1]\onehalfspacing
 		\Require 
                ~
        \begin{itemize}[leftmargin=*]
        \item Indices $t,h\in[H]$ such that $t<h$.
        \item Index $i\in\brk{S}$.
          \hfill\algcommentlight{Index for policy
            $\pihat\ind{i,t}\in\Pinm^{t:h}$ produced in \cref{line:nonmark} of \cref{alg:IKDP}.}
                    \item                 Initial observation $x_t\in
        \cX_t$.
                      \item Functions $(\hat f\ind{t},\phihat\ind{t}),\dots, (\hat
        f\ind{h-1}, \phihat\ind{h-1})$ produced in 
        \cref{line:inversenew2} of \cref{alg:IKDP}.
                  \end{itemize}
 	    \State Set $\bj_{t-1} = i$.
 		\For{$\tau=t,\ldots, h-1$}
 		\State Compute $(\a_\tau, \bm{j}_\tau)\in \argmax_{(a,j)} \fhat\ind{\tau}((a,j) \mid \phihat\ind{\tau}(\x_\tau),\bm{j}_{\tau-1})$. \label{line:argmax}
 		\State Play action $\a_\tau$ at layer $\tau$ and
                observe $\x_{\tau+1}$.
 		\EndFor
                \State \textbf{Return:} Partial trajectory
                $(\a_{t:h-1},\x_{t:h})$ generated by $\pihat\ind{i,t}\in
                \Pinm^{t:h-1}$ (\cref{line:nonmark}).
 	\end{algorithmic}
 \end{algorithm}	

}

\icml{
 
\begin{algorithm}[ht]
	\caption{Execute non-Markov partial policy
		produced by $\musik$.}
	\label{alg:gen}
	\begin{algorithmic}[1]\onehalfspacing
		\Require 
		~Indices $t,h\in [H]$ and $i\in\brk{S}$ (Indexes policy
			$\pihat\ind{i,t}\in\Pinm^{t:h-1}$ produced in \cref{line:nonmark} of \cref{alg:IKDP}).
		Initial observation $x_t\in
		\cX_t$. Functions $(\hat f\ind{t},\phihat\ind{t}),\dots, (\hat
		f\ind{h-1}, \phihat\ind{h-1})$ produced in 
		\cref{line:inversenew2}.
		\State Set $\bj_{t-1} = i$.
		\For{$\tau=t,\ldots, h-1$}
		\State $(\a_\tau, \bm{j}_\tau)\gets \argmax\limits_{(a,i)} \fhat\ind{\tau}((a,i) \mid \phihat\ind{\tau}(\x_\tau),\bm{j}_{\tau-1})$ \label{line:argmax}
		\State Play action $\a_\tau$ at layer $\tau$ and
		observe $\x_{\tau+1}$.
		\EndFor
		\State \textbf{Return:} Partial trajectory
		$(\a_{t:h-1},\x_{t:h})$ generated by $\pihat\ind{i,t}\in
		\Pinm^{t:h-1}$ (\cref{line:nonmark} of \cref{alg:IKDP}).
	\end{algorithmic}
\end{algorithm}

}

 }

\arxiv{
 \subsection{Application to Reward-Based RL: Planning with an Approximate Cover}
 \label{sec:rewardsetting} 
\arxiv{
 To conclude the section, we show how the policy cover learned by $\musik$ can be used to optimize any downstream reward function of interest.}
\icml{In this section, we show how the policy cover learned by $\musik$ can be used to optimize any downstream reward function of interest. }
 For the results that follow, we assume that at each layer $h\in\brk{H}$, the learner observes a reward $\br_h\in\brk*{0,1}$ in addition to the observation $\x_h\in\cX$, so that trajectories take the form $(\x_1,\a_1,\br_1),\ldots,(\x_H,\a_H,\br_H)$.  We will make the following standard BMDP assumption \citep{misra2020kinematic,zhang2022efficient}, which asserts that the mean reward function depends only on the latent state, not the full observation.
\begin{assumption}[Realizability]
	\label{assum:reward}
	For all $h\in[H]$, there exists $\rbar_h\colon \cS \times \cA \rightarrow [0,1]$ such that $\E[\br_h \mid \x_h=x,\a_h=a]=\rbar_h(\phi_\star(x),a)$. 
\end{assumption}

\begin{algorithm}
	\caption{$\texttt{PSDP}$: Policy Search by Dynamic Programming
		(variant of \citet{bagnell2003policy})}
	\label{alg:PSDP}
	\begin{algorithmic}[1]\onehalfspacing
		\Require Policy cover $\Psi\ind{1},\dots,\Psi\ind{H}$. Decoder class $\Phi$. Number of samples $n$.
		\For{$h=H, \dots, 1$} 
		\State $\cD\ind{h} \gets\emptyset$. 
		\For{$n$ times}
		\State Sample $(\x_h, \a_h, \bm{r}_{h:H})\sim
		\unif(\Psi\ind{h})\circ_h \unif (\cA) \circ_{h+1} \pihat \ind{h+1}$.
		\State Update dataset: $\cD\ind{h} \gets \cD\ind{h} \cup \left\{\left(\x_h, \a_h, \sum_{t =h}^H \br_{t}\right)\right\}$.
		\EndFor
		\State Solve regression:
		\[(\fhat\ind{h},\phihat\ind{h}) \gets\argmin_{f \colon [S]\times \cA\rightarrow [0,H-h +1], \phi \in \Phi}  \sum_{(x, a, R)\in\cD} (f(\phi(x),a)-R)^2.\] \label{eq:mistake}
		\State Define $\pihat\ind{h}\in\Pim^{h:H}$ via
		\[
		\pihat \ind{h}(x) = \left\{
		\begin{array}{ll}
			\argmax_{a\in \cA} \fhat\ind{h}(\phihat\ind{h}(x),a),&\quad x\in\cX_h,\\
			\pihat\ind{h}(x),& \quad x\in \cX_{t},\;\;  t\in [h+1 \ldotst H].
		\end{array}
		\right.
		\]
		\EndFor
		\State \textbf{Return:} Near-optimal policy $\pihat \ind{1}\in \Pim$. 
	\end{algorithmic}
\end{algorithm}

\paragraph{The PSDP algorithm}
To optimize rewards, we take a somewhat standard approach and appeal to a variant of the Policy Search by Dynamic Programming (\psdp) algorithm of \citet{bagnell2003policy,misra2020kinematic}. $\psdp$ uses the approximate policy cover produced by \musik as part of a dynamic programming scheme, which constructs a near-optimal policy in a layer-by-layer fashion. In particular, starting from layer $H$, $\psdp$ first constructs a partial policy $\hat
\pi\ind{H}\in \Pim^{H:H}$ using data collected with $\Psi\ind{H}$, then moves back a layer and constructs a partial policy
$\pihat \ind{H-1}\in \Pim^{H-1:H}$ using data collected with $\Psi\ind{H-1}$ and $\pihat\ind{H}$,
and so on, until the first layer is reached. The variant of $\psdp$ we present here differs slightly from the original version in \cite{bagnell2003policy,misra2020kinematic}, with the main difference being that instead of using a policy optimization sub-routine to compute the policy for each layer, we appeal to least-squares regression (see \cref{eq:mistake} of \cref{alg:PSDP}) to estimate a $Q$-function, and then select the greedy policy this function induces.

The following result, proven in \cref{app:PSDPthmproof}, provides the main sample complexity guarantee for \psdp.\footnote{This result does not immediately follow from prior work \citep{misra2020kinematic} because it allows for an $(\alpha,\veps)$-policy cover with $\veps>0$; previous work only handles the case where $\veps=0$.}
\begin{theorem}
	\label{thm:PSDPthm}
	Let $\alpha$, $\veps$, $\delta \in(0,1)$ be given. Suppose that \cref{assum:real,assum:reward} hold, and that for all $h\in\brk{H}$:
	\begin{enumerate}
		\item $\Psi\ind{h}$ is a $(\alpha, \eps)$-approximate cover for layer $h$, where $\eps\coloneqq \veps/(2SH^2)$.
		\item $|\Psi\ind{h}|\leq S$.
	\end{enumerate}
	Then, for appropriately chosen $n\in\bbN$, the policy $\pihat\ind{1}$ returned by \pref{alg:PSDP} satisfies
	\begin{align}
		\E^{\pihat\ind{1}}\left[\sum_{h=1}^H \br_h\right]\geq    \max_{\pi \in \Pim}  \E^{\pi}\left[\sum_{h=1}^H \br_h\right] - \veps \nn
	\end{align}
	with probability at least $1-\delta$. Furthermore, the total number of sampled trajectories used by the algorithm is bounded by
	\begin{align}
	\nn
		\bigoht(1)\cdot \frac{ H^5 S^6 (S A + \ln (|\Phi|/\delta))}{ \alpha^2 \veps^2}.
	\end{align}   
\end{theorem}

\paragraph{Sample complexity to find an $\veps$-suboptial policy with $\musik+\psdp$} From \cref{thm:PSDPthm}, to find an $\veps$-suboptimal policy, $\psdp$ requires an $(\alpha,\eps)$-approximate cover for all layers, where $\eps \ldef \veps/(2 SH^2)$. Focusing only on dependence on the accuracy parameter $\veps$, it follows from the results in \pref{sec:main_theorem} that $\musik$ can generate an $(1/4, \eps)$-approximate cover using $\wtilde O(1/\veps^2)$ trajectories (see \eqref{eq:trajectories}). Thus, the total number of trajectories required to find an $\veps$-suboptimal policy in reward-based RL using $\musik+\psdp$ scales with $\wtilde O(1/\veps^2)$. To the best of our knowledge, this is the first computationally efficient approach that gives $\bigoht(1/\veps^2)$ sample complexity for reward-based reinforcement learning in BMDPs (without reachability).

 }

\arxiv{
		\section{Overview of Analysis}
	\label{sec:overview}
\icml{
  In this section, we give an overview of the analysis of our main result, \cref{thm:policycover}, with the full proof deferred to  \cref{sec:BMDP}. First, in \cref{sec:warmup} we show how to analyze a simplified version of \musik for the \emph{tabular} setting in which the state $\s_h$ is directly observed. Then, in \cref{sec:block}, we build on these developments to give a proof sketch for the full Block MDP setting.
  }

\arxiv{
  In this section, we give an overview of the analysis of our main result for \musik, \cref{thm:policycover}, with the full proof deferred to  \cref{sec:BMDP}. First, in \cref{sec:keytools} we introduce two analysis tools, the \emph{extended BMDP} and \emph{truncated policy class}, which play a key role in providing tight guarantees for \musik (and more broadly, non-optimistic algorithms) in the absence of minimum reachability. Then, as a warmup (\cref{sec:warmup}), we show how to analyze a simplified version of \musik for the \emph{tabular} setting in which the state $\s_h$ is directly observed (i.e., $\cX=\cS$ and $\bx_h=\bs_h$ almost surely). Finally, in \cref{sec:block}, we build on these developments to give a proof sketch for the full Block MDP setting.
  }

\arxiv{\subsection{Key Analysis Tools: Extended BMDP and Truncated Policy Class}
\label{sec:keytools}
Recall that \musik proceeds by inductively building a sequence of policy covers $\Psi\ind{1},\ldots,\Psi\ind{H}$. A key invariant maintained by the algorithm is that for each layer $h$, the previous covers $\Psi\ind{1},\ldots,\Psi\ind{h-1}$ provide good coverage for layers $1,\ldots,h-1$, and thus can be used to efficiently gather data to build the next cover $\Psi\ind{h}$. Prior approaches that build policy covers in this inductive fashion \citep{du2019provably,misra2020kinematic} require the assumption of minimum reachability (\pref{def:reachability}) to ensure that for each $h$, $\Psi\ind{h}$ \emph{uniformly} covers all possible states in $\cS_h$. In the absence of reachability, we inevitably must sacrifice certain hard-to-reach states, which necessitates a more refined analysis. In particular, we must show that the effects of ignoring hard-to-reach states at earlier layers do not compound as the algorithm proceeds forward.

To provide such an analysis, we make use of a tool we refer to as the \emph{extended BMDP} $\Mbar$. The extended BMDP $\Mbar$ augments $\cM$ by adding a set of $H$ \emph{terminal states} $\tfrak_{1:H}$ and one additional \emph{terminal action} $\afrak$ as follows:
	\begin{enumerate}
		\item The latent state space is $\wbar \cS \coloneqq \bigcup_{h=1}^H \wbar \cS_h$, where $\wb{\cS}_h\ldef\cS_h\cup\crl{\tfrak_h}$. 
		\item The action space is $\cAbar\ldef\cA \cup\crl{\afrak}$. Here, $\afrak$ is a ``terminal action'' that causes the latent state to deterministically transition to $\term_{h+1}$ from every state at layer $h\in[H-1]$.
		\item For $h\in[H-1]$, taking any action in $\wbar\cA$ at latent state $\tfrak_h$ transitions to $\tfrak_{h+1}$ deterministically.\footnote{The reason we introduce $H$ states $\tfrak_{1:H}$ instead of a single self-looping state $\tfrak$ is to keep the convention that the state space is layered.}
                \end{enumerate}
            The dynamics of $\cMbar$ (including the initial state distribution) are otherwise identical to $\cM$. We assume the state $\tfrak_h$ emits itself as an observation and we write $\wb{\cX}_h \coloneqq  \cX_h \cup\{\tfrak_h\}$, for all $h\in[H]$. We will use the convention that $\phi(\tfrak_h)=\tfrak_h$, for all $\phi\in\Phi$ and $h\in[H]$. For any policy $\pi \in \Pibar_{\nm} \coloneqq \left\{\pi \colon \bigcup_{h=1}^H (\wbar\cX_{1}\times \dots \times \wbar\cX_h) \rightarrow \Abar\right\}$, we define 
	\begin{align*}
		\Pbar^{\pi}\coloneqq \P^{\Mbar,\pi}, \quad  \Ebar^{\pi}\coloneqq \E^{\Mbar,\pi},\quad  \text{and} \quad \dbar^{\pi}(s)\coloneqq \P^{\Mbar,\pi}[\s_h = s], \quad \text{for all } s\in \cS_h \text{ and } h \in[H].
	\end{align*}
	\paragraph{Truncated policy class} On its own, the extended BMDP is not immediately useful. The main idea behind our analysis is to combine it with a restricted sub-class of policies we refer to as the \emph{truncated policy class.} Define $\Pibarm  \coloneqq \{ \pi\colon \bigcup_{h=1}^H \wbar\cX_h \rightarrow \Abar\}$. For $\epsilon \in(0,1)$, we define a sequence of policy classes $\Pibar_{0, \epsilon},\dots,\Pibar_{H, \epsilon}$, inductively, starting from $\Pibar_{0, \epsilon}=\Pibarm$:
	\begin{align}
		\pi \in \Pibar_{t, \epsilon} & \iff \nn \\  &  \hspace{-.1cm} \exists \pi'\in \Pibar_{t-1, \epsilon}: \forall h \in[H], \forall s\in \wbar\cS_h, \forall x\in \phi^{-1}_\star(s), \ 	\pi(x) = \left\{\begin{array}{ll} \fraka, & \text{if } h=t \text{ and } \max_{\tilde \pi\in \Pibar_{t-1, \epsilon}} \dbar^{\tilde \pi}(s)< \epsilon, \\   \pi'(x), & \text{otherwise}.    \end{array}    \right.  \label{eq:define0}
	\end{align}
	Restated informally, the class $\Pibar_{t,\eps}$ is identical to $\Pibar_{t-1,\eps}$, except that at layer $t$, all policies in the class $\Pibar_{t,\eps}$ take the terminal action $\afrak$ in latent states $s$ for which $\max_{\tilde \pi\in \Pibar_{t-1, \epsilon}} \dbar^{\tilde \pi}(s)< \epsilon$.
	
	We define the \emph{truncated policy class} as $\Pibar_{\epsilon} \coloneqq \Pibar_{H,\epsilon}$. The truncated policy class satisfies two fundamental technical properties. First, by construction, all policies in the class take the terminal action $\afrak$ when they encounter states that are not $\eps$-reachable by $\Pibar_\eps$. 	Second, in spite of the fact that policies in $\Pibar_{\epsilon}$ always take the terminal action on states with low visitation probability, they can still achieve near-optimal visitation probability for all states in $\cMbar$ (up to additive error). The following lemmas formalize these properties.
	\begin{lemma}[Behavior on low-reachability states]
		\label{lem:bla}
		Let $\epsilon\in(0,1)$ be given, and define
		\begin{align}
		  \label{eq:truncated_reachable}
			\cS_{h,\epsilon} \ldef  \crl*{  s\in \cS_h  \colon	\max_{\pi \in \Pibar_{\epsilon}}\dbar^{\pi}(s) \geq \epsilon}.
		\end{align}
		Then, for all $h \in[H]$ if $s\in \cS_{h }\setminus\cS_{h,\epsilon}$, then for all $\pi\in\Pibar_\eps$, $\pi(x)=\fraka$ for all $x\in \phi^{-1}_\star(s)$.
	\end{lemma}
	\begin{lemma}[Approximation for truncated policies]
		\label{lem:critical}
		Let $\epsilon\in(0,1)$ be given. For all $h \in[H]$ and $s\in \cS_h$, 
		\begin{align}
		         \label{eq:truncated_approximation}
			\max_{\pi \in \Pibarm} \dbar^{\pi}(s) \leq \max_{\pi \in \Pibar_{\epsilon}} \dbar^{\pi}(s) +S \epsilon.
		\end{align}
	\end{lemma}
	The proofs for these results (and other results in this subsection) are elementary, and are given in \pref{app:structural}.       
	Building on these properties, our proof of \pref{thm:policycover} makes use of two key ideas:
	\begin{enumerate}
		\item Even though the extended BMDP $\cMbar$ does not necessarily enjoy minimum reachability (\pref{def:reachability}), if we restrict ourselves to competing against policies in $\Pibar_\eps$, \pref{lem:bla} will allow us to ``emulate'' certain properties enjoyed by $\eps$-reachable BMDPs. This in turn will imply that if are satisfied with learning a policy cover with good coverage ``relative'' to $\Pibar_\eps$, \pref{alg:GenIk} will succeed.
		\item By \pref{lem:critical}, we lose little by restricting our attention to the class $\Pibar_\eps$. This will allow us to transfer any guarantees we achieve with respect to the extended BMDP $\cMbar$ and truncated policy class $\Pibar_\eps$ back to the original BMDP $\cM$ and unrestricted policy class $\Pim$.
	\end{enumerate}
	We make the first point precise in the sections that follow (\cref{sec:warmup,sec:block}). Before proceeding, we formalize the second point via another technical result, \pref{lem:transfer}. To do so, we introduce the notion of a \emph{relative policy cover}.
	\begin{definition}[Relative policy cover]
		\label{def:polcover}
		Let $\alpha, \veps\in[0,1)$ be given. Consider a BMDP $\cM'$, and let $\Pi$ and $\Psi$ be two sets of policies. We say that $\Psi$ is an $(\alpha,\veps)$-policy cover relative to $\Pi$ in $\cM'$ for layer $h$ if
		\begin{align}
			\max_{\pi \in \Psi} d^{\cM',\pi}(s) \geq  \alpha \cdot \max_{\pi \in \Pi} d^{\cM',\pi}(s) \quad \text{for all } \  s\in \cS_h \ \text{ such that }\  \max_{\pi \in \Pi} d^{\cM',\pi}(s)  \geq \veps. \nn
		\end{align}
	\end{definition} 
	\begin{lemma}[Policy cover transfer]
	\label{lem:transfer}
		Let $\veps\in(0,1)$ be given, and define $\eps\ldef\veps/2S$.
		Let $\Psi$ be a set of policies for $\cMbar$ that never take the terminal action $\afrak$. If $\Psi$ is a $(1/2,\epsilon)$-policy cover relative to $\Pibar_\epsilon$ in $\Mbar$ for all layers, then $\Psi$ is a $(1/4,\veps)$-policy cover relative to $\Pim$ in the $\cM$ for all layers.
	\end{lemma}
	\pref{lem:transfer} implies that for any $\veps$, letting $\eps\ldef{}\veps/2S$, if we can construct a set $\Psi$ that acts as a
	$(1/2,\eps)$-policy cover relative to
	$\Pibar_{\epsilon}$ in $\Mbar$, then $\Psi$ will also be a $(1/4,\veps)$-policy cover relative to $\Pim$ in the original BMDP $\cM$, as desired. This allows us to restrict our attention to the former goal going forward. }

	\subsection{Warm-Up: Multi-Step Inverse Kinematics for Tabular MDPs}
	\label{sec:warmup}
        \begin{algorithm}[htp]
          \caption{$\musiktab$: Multi-Step Inverse Kinematics (tabular variant)}
          \label{alg:musik_tabular}
          \begin{algorithmic}[1]\onehalfspacing
            \Require Number of samples $n$.
            \State Set $\Psi\ind{1}= \emptyset$.
            \For{$h=2\ldots, H$} 
            \State Let
            $\Psi\ind{h}=\ikdptab(\Psi\ind{1},\dots,\Psi\ind{h-1},
            n)$.\quad\algcommentlight{\pref{alg:IKDP-tab}.}
            \label{line:ikdp1}
            \EndFor
            \State \textbf{Return:} Policy covers $\Psi\ind{1},\dots,\Psi\ind{H}$. 
          \end{algorithmic}
	\end{algorithm}

In this section, we use the extended BMDP, truncated policy class, and relevant structural results introduced in
prequel to analyze a simplified version of \musik for the \emph{tabular} setting in which the state $\s_h$ is directly observed (a special case of the BMDP in which $\cX=\cS$ and $\bx_h=\bs_h$ almost surely). The tabular setting preserves the most important challenges in removing reachability, and will serve as a useful warm-up exercise for the full BMDP setting. Our analysis will also give a taste for how the multi-step inverse kinematics objective in \ikdp (\pref{line:inversenew2}) allows one to approximately implement dynamic programming.

\paragraph{$\musik$ and $\ikdp$ for tabular MDPs} \pref{alg:musik_tabular} (\musiktab) and \pref{alg:IKDP-tab} (\ikdptab) are simplified variants of \musik and \ikdp tailored to the tabular setting. \musiktab is identical to \musik, except that the subroutine \ikdp is replaced by \ikdptab. $\ikdptab$ has the same structure as $\ikdp$, but does not require access to a decoder class $\Phi$, since the states are observed directly. The algorithm takes advantage of a slightly simplified multi-step inverse kinematics objective (\cref{line:inversenew20} of \cref{alg:IKDP-tab}) which involves directly predicting actions based on the latent states. Recall that for iteration $t\in[h-1]$, the full version of $\ikdp$ uses observations to predict \emph{pairs} $(\a_t,\bi_t)$, where $\a_t$ is the action played at layer $t$ and $\bi_t\in\brk{S}$ is the (random) index of the partial policy executed after layer $t$. \ikdptab does
not require randomizing over the index $\bi_t$, and instead solves a separate regression problem for each state $i\in\brk{S}$ (representing the state being targeted at layer $h$), predicting only the action $\a_t$; we will highlight the need for the randomization over indices $\bi_t$ when we return to the BMDP setting in the sequel (\pref{sec:block}).

The following theorem, an analogue of \cref{thm:policycover} for tabular MDPs, provides the main guarantee for $\musiktab$.
\begin{theorem}[Main theorem for \musiktab]
	\label{thm:trend}
	Let $\veps,\delta \in(0,1)$ be given, and let  $n\geq 1$ be chosen such that
    \begin{align}
	n \geq \frac{c A^2 S^6 H^2 \left( S^2 A \ln n + \ln (S H^2/\delta)\right)}{\veps^2},\label{eq:condn} 
	\end{align} for some absolute constant $c>0$ independent of all problem parameters. Then, with probability at least $1-\delta$, the collections $\Psi\ind{1},\dots,\Psi\ind{H}$ produced by \musiktab are $(1/4,\veps)$-policy covers for layers $1$ through $H$.
\end{theorem}
\paragraph{Analysis by induction}
To prove \cref{thm:trend}, we proceed by induction over the layers $h=1,\dots,H$. Leveraging the extended MDP and truncated policy class, we will show that for each layer $h\in[H]$, if the collections $\Psi\ind{1},\dots, \Psi\ind{h-1}$ produced by $\ikdptab$ have the property that
\begin{align}
  \Psi\ind{1},\dots, \Psi\ind{h-1} \text{ are $(1/2,\epsilon)$-policy covers relative to $\Pibar_{\epsilon}$ in $\Mbar$ for layers $1$ through $h-1$,} \label{eq:inductouter}
\end{align}
then with high probability, the collection $\Psi_{h}$ produced by $\ikdptab(\Psi_{1:h-1},n)$ will be a  $(1/2,\epsilon)$-policy cover relative to $\Pibar_{\epsilon}$ in $\Mbar$ for layer $h$. Formally, we will prove the following result.
\begin{theorem}[Main theorem for $\ikdptab$]
	\label{thm:geniklemma0}
	Let $\epsilon,\delta \in(0,1)$ and $h\in[H]$ be given and define $\veps_{\stat}(n, \delta')\coloneqq  \sqrt{ n^{-1}\prn{S^2 A \ln n+ \ln (1/\delta')}}$. Assume that:
        \begin{enumerate}
        \item \emikdptab is invoked with $\Psi\ind{1},\ldots, \Psi\ind{h-1}$ satisfying \cref{eq:inductouter}.
        \item The policies in $\Psi\ind{1},\ldots, \Psi\ind{h-1}$ never take the terminal action $\afrak$.
        \item The parameter $n$ is chosen such that $8A  S^2 HC \cdot\veps_{\stat}(n,\frac{\delta}{SH^2})\leq \epsilon$ for some absolute constant $C>0$ independent of all problem parameters.
      \end{enumerate}
      Then, with probability at least $1-\frac{\delta}{H}$, the collection $\Psi\ind{h}$ produced by $\emikdptab(\Psi\ind{1},\dots,\Psi\ind{h-1},n)$ is an $(1/2,\epsilon)$-policy cover relative to $\Pibar_{\epsilon}$ in $\Mbar$ for layer $h$. In addition, $\Psi\ind{h}\subseteq \Pim^{1:h-1}$.
      \end{theorem}
      With this result in hand, the proof of \pref{thm:trend} follows swiftly.
      \begin{proof}[{Proof of \cref{thm:trend}}]
  Let $\delta,\veps\in(0,1)$ be given and let $\epsilon \coloneqq \veps/(2S)$. Let $\veps_\stat(\cdot, \cdot)$ and $C$ be as in \cref{thm:geniklemma0}; here $C$ is an absolute constant independent of all problem parameters. Let $\cE_h$ denote the event that \ikdptab succeeds as in \cref{thm:geniklemma0} for layer $h\in[H]$ with parameters $\delta$ and $\eps$, and define $\cE\coloneqq \bigcap_{h\in[H]}\cE_h$. Observe that by \pref{thm:geniklemma0} and the union bound, we have $\P[\cE]\geq 1 -\delta$. For $n$ large enough such that $8A  S^2 HC \cdot\veps_{\stat}(n,\frac{\delta}{SH^2})\leq \epsilon$ (which is implied by the condition on $n$ in the theorem's statement for $c=2^5 C$), \cref{thm:geniklemma0} implies that under $\cE$, the output $\Psi\ind{1},\dots,\Psi\ind{H}$ of $\musik$ are $(1/2,\epsilon)$-policy covers relative to $\Pibar_{\epsilon}$ in $\wbar \cM$ for layers 1 to $H$, respectively. We conclude by appealing to \cref{lem:transfer}, which now implies that $\Psi\ind{1},\ldots,\Psi\ind{H}$ are $(1/4,\veps)$-policy covers relative to $\Pim$ in $\cM$.

\begin{algorithm}[t]
  \caption{$\ikdptab:$ Inverse Kinematics for Dynamic Programming (tabular variant)}
  \label{alg:IKDP-tab}
  \begin{algorithmic}[1]\onehalfspacing
    \Require
    ~
    \begin{itemize}[leftmargin=*]
    \item Approximate covers
      $\Psi\ind{1},\dots,\Psi\ind{h-1}$ for layers $1$
      to $h-1$, where $\Psi\ind{t} \subseteq\Pim^{1:t-1}$.
    \item Number of samples $n$.
    \end{itemize}
    \For{$t=h-1,\ldots, 1$} \label{line:mainiter0}
    \State $\cD\ind{t}\leftarrow \emptyset$.
    \Statex[1]\algcommentbiglight{Collect data by rolling in with
      policy cover and rolling out with partial policy}
    \For{$i\in[S]$} \label{line:for2}
    \For{$n$ times}
    \State Sample $(\s_t, \a_t, \s_{h})\sim \unif(\Psi\ind{t}) \circ_t \unifa\circ_{t+1} \pihat ^{(i,t+1)}$.
    \State $\cD\ind{t} \leftarrow \cD\ind{t}\cup \{(\a_t, \s_t, \s_{h})\}$.
    \EndFor
    \Statex[2] \algcommentbiglight{Inverse kinematics}
    \State \label{line:inversenew20}
    \begin{equation}
      \fhat\ind{i,t} \in \mathrm{argmax}_{f\colon  [S]^2 \rightarrow \Delta_{A}} \sum_{(a,s,s')\in \cD\ind{t}} \ln  f(a \mid s, s').
      \label{eq:ik_tabular}
    \end{equation}
    \Statex[2] \algcommentbiglight{Update partial policy cover}
    \State Define $\ahat\ind{i,t}(s)\in \argmax_{a\in \cA} \fhat\ind{i,t}(a \mid s,i)$.
    \State  Define $\hat\pi^{(i,t)}\in \Pim^{t:h-1}$  via\label{line:policy_update_tabular}
    \begin{align}
      \hat  \pi^{(i,t)}(s)\coloneqq \left\{
      \begin{array}{ll}
        \ahat\ind{i,t}(s),&\quad s\in\cS_t,\\
        \pihat
        ^{(i,\tau)}(s),&\quad  s\in \cS_\tau,\;\;\tau\in\range{t+1}{h-1}.
      \end{array}
                         \right.
                         \label{eq:policy_update_tabular}
    \end{align}
    \EndFor
    \EndFor
    
    \State \textbf{Return:} Policy cover
    $\Psi\ind{h}=\{\hat\pi^{(i,1)}\colon i
    \in[S]\}\subseteq \Pim^{1:h-1}$ for layer $h$.
  \end{algorithmic}
\end{algorithm}

We now compute the total number of trajectories used by the algorithm.
Recall that when invoked with parameter $n\in\bbN$, \musiktab invokes $H-1$ instances of $\ikdptab$, each with parameter $n$. Each instance of $\ikdptab$ uses $n$ trajectories for each layer $t\in[h-1]$ and $i\in[S]$ (see \cref{line:mainiter0,line:for2} of \cref{alg:IKDP-tab}), so the total number of trajectories used by $\musiktab$ is at most
		\begin{align}
                  \bigoht(1)\cdot\frac{A^2 S^7 H^4 \left( S^2 A + \ln (|\Phi|S H^2/\delta)\right)}{\veps^2}.\nn 
		\end{align}
\end{proof}

\subsubsection{Proof Sketch for \creftitle{thm:geniklemma0}}
We now sketch the proof of \pref{thm:geniklemma0}. The most important feature of the proof is that the guarantee on which we induct, \pref{eq:inductouter}, is stated with respect to the extended MDP and truncated policy class. We work in the extended MDP throughout the proof, and only pass back to the original MDP $\cM$ and full policy class $\Pim$ in the proof of \pref{thm:trend} (see above) \emph{once the induction is completed}.

Let $h\in[H]$ and $\epsilon>0$ be fixed, and assume that \cref{eq:inductouter} holds (that is, $\Psi\ind{1},\dots, \Psi\ind{h-1}$ are $(1/2,\epsilon)$-policy covers relative to $\Pibar_{\epsilon}$ in $\Mbar$ for layers $1$ through $h-1$). We will prove that the collection $\Psi\ind{h}$ produced by $\emikdptab(\Psi\ind{1},\dots,\Psi\ind{h-1},n)$ is an $(1/2,\epsilon)$-policy cover relative to $\Pibar_{\epsilon}$ in $\Mbar$ for layer $h$. We first argue that proving \cref{thm:geniklemma0} reduces to showing the following lemma. To state the result, recall that $\cS_{h,\epsilon} \ldef  \crl*{  s\in \cS_h  \colon	\max_{\pi \in \Pibar_{\epsilon}}\dbar^{\pi}(s) \geq \epsilon}$ is the set of states that are $\eps$-reachable by $\Pibar_\eps$ in $\cMbar$.
\begin{lemma}
  \label{lem:prop}
  Assuming points 1.~and 2.~in \cref{thm:geniklemma0} hold, and if $n$ is chosen large enough such that $8A  S^2 HC \cdot\veps_{\stat}(n,\frac{\delta}{SH^2})\leq \epsilon$ for some absolute constant $C>0$ independent of all problem parameters, then for all $t\in[h-1]$, with probability at least $1-\delta/H^2$, the learned partial policies $\crl*{\pihat\ind{i,t}}_{i\in[S]}$ in $\ikdptab$ have the property that for all $i \in \cS_{h,\epsilon}$,
  \begin{align}
    \dbar^{\pi_\star\ind{i}\circ_{t+1} \pihat\ind{i,t+1}}(i)  - \dbar^{\pi_\star\ind{i}\circ_{t}\pihat\ind{i,t}}(i)\leq \frac{\epsilon}{2H}, \quad \text{where}\quad \pi_\star\ind{i} \in \argmax_{\pi \in \Pibar_\epsilon} \dbar^{\pi}(i).\label{eq:prop}
  \end{align}
\end{lemma}
For each $i\in\cS_{h,\eps}$, $\pi_\star\ind{i}$ in \pref{eq:prop} denotes the policy in the truncated class $\Pibar_{\epsilon}$ that maximizes the probability of visiting $i$ at layer $h$. Informally, \cref{eq:prop} states that if we execute $\pi_\star\ind{i}$ up to layer $t-1$ (inclusive), then switch to the learned partial policy $\pihat\ind{i,t}$ for the remaining steps (i.e.~execute $\pi_\star\ind{i}\circ_{t} \pihat\ind{i,t}$), then the probability of reaching state $i$ in layer $h$ is close to what is achieved by running $\pi_\star\ind{i}\circ_{t+1} \pihat\ind{i,t}$. In other words, $\pihat\ind{i,t}$ is near-optimal in an average-case sense. We now show that \cref{thm:geniklemma0} follows from \cref{lem:prop}.
\begin{proof}[Proof of \cref{thm:geniklemma0}]
	For $t \in [h-1]$, let $\cE_t$ denote the success event of \cref{lem:prop}. Let us condition on the event $\cE \coloneqq \bigcap_{t \in [h-1]} \cE_t$. Fix $\i \in \cS_{h,\epsilon}$. Summing the \lhs of \pref{eq:prop} over $t=1,\dots,h-1$ for $i=\i$ and telescoping, we have that
	\begin{align}
		\dbar^{\pihat\ind{\i,1}}(\i)  \geq 	\max_{\pi \in \Pibar_\epsilon} \dbar^{\pi}(\i) - \frac{\epsilon}{2} \geq  \frac{1}{2}	\max_{\pi \in \Pibar_\epsilon} \dbar^{\pi}(\i),\label{eq:newest}
	\end{align}
where the last inequality follows by the fact that $\max_{\pi \in \Pibar_\epsilon} \dbar^{\pi}(\i)\geq \epsilon$ (since $\i \in \cS_{h,\epsilon}$).
Since this conclusion holds uniformly for all $\i\in\cS_{h,\epsilon}$, we have that under the event $\cE$, the output $\Psi\ind{h} \coloneqq \{ \pihat\ind{i,1}\colon i \in[S]\}$ of \cref{alg:IKDP-tab} is a $(1/2,\epsilon)$-policy cover relative to $\Pibar_\epsilon$ for layer $h$. Finally, by a union bound, we have $\P[\cE] \geq 1 - \P[\cE^c] \geq 1 -\sum_{t\in[h-1]}\sum_{i\in[S]} \P[(\cE\ind{i}_t)^c]\geq 1 - \delta/H$, which completes the proof. \end{proof}

  \begin{remark}
    \label{rem:performance}
It is also possible to derive \pref{eq:newest} from \cref{lem:prop} using the performance difference lemma \citep{kakade2003sample} with a specific state-action value function; this perspective will be useful when we generalize our analysis from the tabular to the BMDP setting. To see how the performance difference lemma can be applied to obtain \pref{eq:newest}, fix $i\in\brk{S}$ and consider the \emph{state-action value function} ($Q$-function) at layer $t$ with respect to the partial policy $\pihat\ind{i,t}\in \Pim^{t:h-1}$ for the MDP $\wbar \cM$ with rewards $r\ind{i}_\tau(s)=\mathbf{1}\{s=i\} \cdot \mathbf{1}\{\tau = h\}$, for $\tau\in[h]$; that is, \loose
  \begin{align}
  	Q_t^{\hat \pi\ind{i,t}}(s,a;i) = r\ind{i}_{t}(s) + \Ebar^{\hat \pi\ind{i,t}}\left[\left.\sum_{\tau=t+1}^h r\ind{i}_\tau(\s_\tau) \ \right\mid\  \s_t = s,\a_t = a\right].\label{eq:Qfunction} 
  \end{align}
Thanks to the choice of reward functions, we have 
\begin{align}
	Q_t^{\hat \pi\ind{i,t}}(s,a;i)  &= \Pbar^{ \pihat\ind{i,t+1}}[\s_h=i\mid \s_t =s, \a_t = a], \label{eq:FK} \\  \shortintertext{and thus}    \dbar^{\pihat\ind{i,1}}(i) - \dbar^{\pi_\star\ind{i}}(i) &= \E\left[	Q_1^{\hat \pi\ind{i,t}}(\s_1,\hat \pi\ind{i,t}(\s_1);i) - Q_1^{\pi_\star\ind{i}}(\s_1,\pi_\star\ind{i}(\s_1);i)  \right]. \label{eq:lhs}
\end{align}
Thus, by the performance difference lemma, the \rhs of \eqref{eq:lhs} can be bounded by 
\begin{align}
	\sum_{t=1}^{h-1}\Ebar^{\pi_\star\ind{i}} \left[Q^{\pihat\ind{i,t}}_t(\s_t, \pi_\star\ind{i}(\s_t); i)  -Q^{\pihat\ind{i,t}}_t(\s_t, \pihat ^{(i,t)}(\s_t); i)\right]. \label{eq:cont}
	\end{align}
        Thanks to \pref{eq:FK}, the quantity in \eqref{eq:cont} is simply $\sum_{t=1}^{h-1} \big(\dbar^{\pi_\star\ind{i}\circ_{t+1} \pihat\ind{i,t+1}}(i)  - \dbar^{\pi_\star\ind{i}\circ_{t}\pihat\ind{i,t}}(i)\big)$, which can directly be bounded using \cref{lem:prop} to arrive at the conclusion in \pref{eq:newest}.
  \end{remark}

     It remains to prove \cref{lem:prop}.
  To prove the result, we first use the multi-step inverse kinematics objective to establish a certain ``local'' optimality guarantee. We combine this with the assumption that $\Psi\ind{1},\ldots,\Psi\ind{h-1}$ are policy covers, along with certain structural properties of the extended MDP $\cMbar$, to conclude the result.

\paragraph{A local optimality guarantee from multi-step inverse kinematics}
Fix $1\leq t<h$ and a state $i\in\cS_{h,\epsilon}$, and let $\{\pihat\ind{i,t+1}\colon i\in[S]\}$ be the partial policies constructed by $\ikdptab$ at layer $t+1$. As the first step toward constructing the policy $\pihat \ind{i,t}$, $\ikdptab$ computes an estimator $\fhat\ind{i,t}\colon \brk{S}^2\to\Delta(\cA)$ by solving the multi-step inverse kinematics objective in \pref{line:inversenew20}. This entails predicting the probability of the action $\a_t$ conditioned on the states $\s_t$ and $\s_h$, under the process $(\s_t,\a_t,\s_h) \sim \P^{\unif(\Psi\ind{t})\circ_t \unifa \circ_{t+1}\pihat\ind{i,t+1}}$.\footnote{{Note that $\unifa$ denotes the policy that samples $\a_t$ uniformly from $\cA$, not $\cAbar$.}} The following result gives a generalization guarantee for $\fhat\ind{i,t}$ under this process.
	\begin{lemma}[Conditional density estimation guarantee]
          \label{lem:regtab}
          Fix $t \in \brk{h-1}$. Let $n\geq 1$ and $\delta\in(0,1)$ be given, and define $\veps_\stat(n,\delta)\coloneqq n^{-1/2}\cdot  \sqrt{S^2 A\ln n + \ln (1/\delta)}$. Assume that the policies in $\Psi\ind{t}$ never take the terminal action $\afrak$. Then, there exists an absolute constant $C>0$ (independent of $t,h$, and other problem parameters) such that for all $i\in[S]$ the solution $\fhat\ind{i,t}$ to the conditional density estimation problem in \cref{line:inversenew20} of \cref{alg:IKDP-tab} has that with probability at least $1-\delta$,
		\begin{align}\label{eq:mle_tabular}
			\Ebar^{\unif(\Psi\ind{t})\circ_t \unifa \circ_{t+1} \pihat\ind{i,t+1}} \left[ \sum_{a\in\cA}\left(  \fhat\ind{i,t}(a\mid \s_t, \s_h)	-  P\ind{i,t}_{\bayes}(a\mid \s_t ,\s_h) \right)^2 \right]\leq   C^2 \cdot \veps_\stat^2(n,\delta) ,
		\end{align}
                where \begin{equation}\label{eq:bayes_tabular}
                  P\ind{i,t}_{\bayes}(a\mid s ,s')	\coloneqq
     \frac{\Pbar^{\pihat\ind{i,t+1}}[\s_h=s'\mid
            \s_t=s,\a_t=a]}{Z\ind{i,t}(s,s')}, \ \ \text{for}\ \ Z\ind{i,t}(s,s')\coloneqq  \sum_{a'\in\cA}\Pbar^{
            \pihat\ind{i,t+1}}[\s_h=s'\mid \s_t=s,\a_t=a'].
          \end{equation}
	\end{lemma}
        \pref{lem:regtab} is a consequence of a standard generalization bound for conditional density estimation. The \emph{Bayes-optimal regression function} $\Pbayes\ind{i,t}$ represents the true conditional probability for $\a_t$ under the process $(\s_t,\a_t,\s_h) \sim \P^{\unif(\Psi\ind{t})\circ_t \unifa \circ_{t+1}\pihat\ind{i,t+1}}$. This quantity is useful as a proxy for another quantity we refer to as \emph{forward kinematics}:
\begin{align}
\fk\ind{t}(i\mid s,a ) \coloneqq   \Pbar^{ \pihat\ind{i,t+1}}[\s_h=i\mid \s_t =s, \a_t = a]. \label{eq:fk}
\end{align}
The utility of forward kinematics is somewhat more immediate: It represents the probability that we reach state $i$ at layer $h$ if we start from $\s_t=s$, take action $\a_t=a$, and then roll out with $\pihat\ind{i,t+1}$; equivalently $\fk\ind{t}(i\mid{}s,a)$ is the Q-function for the reward function $\indic\crl{\s_h=i}$---see \pref{eq:FK}. Hence, by the principle of dynamic programming, it is natural to choose 
\begin{align}
\pihat\ind{i,t}(s)=\argmax_{a\in\cA}\fk\ind{t}(i\mid s,a ).\label{eq:fk_opt}
\end{align}
\ikdptab does not directly compute the forward kinematics, and hence cannot directly define $\pihat\ind{i,t}$ based on \pref{eq:fk_opt}. Instead, we compute
\begin{align}
\pihat\ind{i,t}(s)=\argmax_{a\in\cA}\Pbayes\ind{i,t}(a\mid s, i).\label{eq:bayes_pihat}
\end{align}
To see that this is equivalent, observe that $P\ind{i,t}_{\bayes}(a\mid s ,i)$ is a ratio of two quantities: The numerator is exactly $\fk\ind{t}(i\mid{}s,a)$, and the denominator is a ``constant'' whose value does not depend on $a$. With some manipulation, we can use this fact to relate suboptimality with respect to $\fk\ind{t}$ to the regression error in \pref{eq:mle_tabular}, leading to the following ``local'' optimality guarantee for $\pihat\ind{i,t}$ (see \cref{app:geniklemma0proof} for a proof).
\begin{lemma}[Local optimality guarantee]
\label{lem:proppost}
Consider the setting of \cref{thm:geniklemma0} and let $t\in[h-1]$. Then, there is an event $\cE_t$ of probability at least $1-\delta/H^2$ under which the learned partial policies $\crl*{\pihat\ind{i,t}}_{i\in\brk{S}}$ and $\crl{\pihat \ind{i,t+1}}_{i\in[S]}$ in $\ikdptab$ have the property that for all $i \in \cS_{h}$, 
\begin{align}
	\sum_{\pi \in \Psi\ind{t}}	\dbar^{\pi}(s_t)	\left(	\max_{a\in \cA}	\fk\ind{t}(i \mid s_t, a) - \fk\ind{t}(i \mid s_t, \pihat\ind{i,t}(s_t))\right) \leq 2S AC\veps_\stat(n, \delta /(SH^2)), \quad \forall s_t \in \cS_t,\label{eq:high}
	\end{align}
where $\veps_\stat(\cdot,\cdot)$ and $C>0$ are as in \cref{lem:regtab}; here $C>0$ is an absolute constant independent of problem parameters.
\end{lemma}

\begin{remark}
  For the tabular setting where $\s_h$ is observed, it is also possible to estimate the function $\fk\ind{t}(i\mid s,a )$ directly. However, in the BDMP setting, estimating forward kinematics is \emph{not possible} because states are not observed. We will see that in spite of this, the multi-step inverse kinematics objective used in \ikdp still serves as a useful proxy for the forward kinematics.
\end{remark}
We now use \cref{lem:proppost} to prove \cref{lem:prop}.

\begin{proof}[Proof of \cref{lem:prop}] To prove \cref{lem:prop}, we translate the local suboptimality guarantee in \pref{eq:high} to the global guarantee in \pref{eq:prop}. Fix $\i \in \cS_{h,\epsilon}$ and let us abbreviate $\veps_\stat' \equiv C\cdot \veps_\stat(n, \delta /(S H^2))$, where $\veps_\stat(\cdot,\cdot)$ and $C>0$ are as in \cref{lem:regtab}. Condition on the event $\cE_t$ of \cref{lem:proppost}. We begin by writing the \lhs of \cref{eq:prop} in a form that is closer to the \lhs of \pref{eq:high}:
\begin{align}
\dbar^{\pi_\star^{(\i)}\circ_{t+1} \pihat\ind{\i,t+1}}(\i)  - \dbar^{\pi_\star^{(\i)}\circ_{t}\pihat\ind{\i,t}}(\i)
  &= \sum_{s\in \cS_t\cup\{\tfrak_t\}}	\dbar^{\pi_\star^{(\i)}}(s)\cdot \left( \fk\ind{t}(\i\mid s, \pi_\star^{(\i)}(s) ) - \fk\ind{t}(\i \mid s, \pihat\ind{\i,t}(s))\right), \label{eq:propo}
\end{align}
where we use the convention that $\pihat\ind{\i,t}(\tfrak_t)=\afrak$; this equality follows by the definition of $\fk\ind{t}$ in \pref{eq:fk}.  
Now, we bound the \rhs of \cref{eq:propo} in terms of the \lhs of \pref{eq:high} by using that $\Psi\ind{t}$ is a relative policy cover. In particular, since $\Psi\ind{t}$ is an $(1/2,{\epsilon})$-policy cover relative to $\Pibar_{\epsilon}$ at layer $t$, and since $\pi_\star\ind{\i}\in\Pibar_\eps$, we have that for all $s_t \in \cS_{t,\epsilon}$, 
\begin{align}
&	\dbar^{\pi_\star^{(\i)}}(s_t) \left(\fk\ind{t}(\i \mid s_t,\pi_\star^{(\i)}(s_t)) - \fk\ind{t}(\i\mid s_t, \pihat\ind{\i,t}(s_t))  \right) \nn \\
	& \leq 	\dbar^{\pi_\star^{(\i)}}(s_t)	\left(\max_{a\in \wbar\cA} \fk\ind{t}(\i\mid s_t,a) - \fk\ind{t}(\i\mid s_t, \pihat\ind{\i,t}(s_t)) \right), \nn \\ 
	&= 	\dbar^{\pi_\star^{(\i)}}(s_t)	\left(\max_{a\in \cA} \fk\ind{t}(\i\mid s_t,a) - \fk\ind{t}(\i\mid s_t, \pihat\ind{\i,t}(s_t)) \right), \label{eq:india} \\ 
	& \leq  2	\sum_{\pi \in \Psi\ind{t}}	\dbar^{\pi}(s_t)	\left(\max_{a\in \cA} \fk\ind{t}(\i\mid s_t,a) - \fk\ind{t}(\i\mid s_t, \pihat\ind{\i,t}(s_t)) \right),\nn \\
	& \leq   4S A\veps_\stat',\label{eq:forget000}
\end{align}
where \cref{eq:india} follows from the fact $\fk\ind{t}(\i\mid s_t,\fraka)=0$ (since $\fraka$ is the action leading to the terminal state $\tfrak_{t+1}$ from any state at layer $t$), so that $\max_{a\in \wbar\cA} \fk\ind{t}(\i\mid s_t,a)=\max_{a\in \cA} \fk\ind{t}(\i\mid s_t,a)$; \cref{eq:forget000} follows from \cref{eq:high} in \cref{lem:proppost}. On the other hand, by \cref{lem:bla}, we have that for all $s_t \in \cS_t \setminus \cS_{t,\epsilon}$, $\pi_\star^{(\i)}(s_t)=\fraka$. Therefore,
\begin{align}
	\forall s_t \in \cS_t\setminus \cS_{t,\epsilon}, \quad			\fk\ind{t}(\i \mid s_t,\pi_\star^{(\i)}(s_t))  = 	\fk\ind{t}(\i \mid s_t,\afrak)  =0.\label{eq:xxx}
\end{align}
Using this together with the fact that $ \fk\ind{t}(\i\mid s_t, \pihat\ind{\i,t}(s_t))\geq 0$ and \cref{eq:forget000} implies that
\begin{align}
	\forall s_t \in \cS_{t},\quad 		\dbar^{\pi_\star^{(\i)}}(s_t) \left(\fk\ind{t}(\i \mid s_t,\pi_\star^{(\i)}(s_t)) - \fk\ind{t}(\i\mid s_t, \pihat\ind{\i,t}(s_t))  \right)\leq 4S A\veps_\stat'.\label{eq:tosum}
\end{align}
Now, by choosing $n$ large enough such that $8 HS^2 A C\veps_{\stat}(n,\frac{\delta}{ SH^2}) \leq  \epsilon$ (as in the lemma's statement), we have $8 HS^2 A \veps_\stat' \leq \epsilon$ by definition of $\veps_\stat'$. Using this and summing \eqref{eq:tosum} over $s_t\in \cS_t$ in \eqref{eq:tosum} we have that
\begin{align}
  \sum_{s\in \cS_t\cup\{\tfrak_t\}}	\dbar^{\pi_\star^{(\i)}}(s)\cdot \left( \fk\ind{t}(\i\mid s, \pi_\star^{(\i)}(s) ) - \fk\ind{t}(\i \mid s, \pihat\ind{\i,t}(s))\right)
  \leq \frac{\epsilon}{2H},  \label{eq:preperf}
\end{align}
where we have used that $\fk\ind{t}(\i\mid\term_t,\cdot)=0$.
\end{proof}

	\subsection{From Tabular MDPs to Block MDPs}
	\label{sec:block}

	We now give an overview of the proof of \pref{thm:policycover}. The proof builds on the techniques in \pref{sec:warmup} and follows the same structure, but requires non-trivial changes to accommodate the general BMDP setting. We highlight the most important similarities and differences below, with the full proof deferred to \pref{sec:BMDP}.

Recall that on the algorithmic side, the main change in moving from the tabular setting to the general BMDP setting is that the latent states $\s_h$ are unobserved. To address this, the multi-step inverse kinematics objective in \ikdp (\pref{line:inversenew2}) differs from the simplified version in \ikdptab by incorporating estimation of a decoder $\phihat\ind{t}\in\Phi$ at each step $t\in\brk{h-1}$. Here, a critical property of the multi-step inverse kinematics objective is that the Bayes-optimal regression function (the BMDP analogue of \pref{eq:bayes_tabular}) only depends on the observations $\x_t$ and $\x_h$ through $\phistar(\x_t)$ and $\phistar(\x_h)$, which ensures that the conditional density estimation problem in \pref{line:inversenew2} is always well-specified.

\paragraph{The need for non-Markovian policies}
\ikdp also differs from \ikdptab in how we construct the partial policy collection $\crl{\pihat\ind{i,t}}_{i\in\brk{S}}$ for layer $t\in\brk{h-1}$ from the collection $\crl{\pihat\ind{i,t+1}}_{i\in\brk{S}}$ learned at layer $t+1$. The construction in \cref{line:nonmark} of \ikdp, as discussed in \cref{sec:musik}, leads to policies that are \emph{non-Markovian} (that is, history-dependent). This complicates the analysis because we cannot appeal to the performance difference lemma in the same fashion \pref{sec:warmup} (see \cref{rem:performance}), where it was used to relate the local suboptimality for each policy to global suboptimality. Before giving an overview for how we overcome this challenge, we first give a more detailed explanation as to \emph{why} $\ikdp$ builds non-Markovian policies.

Fix $h\in\brk{H}$. Recall that in the tabular setting, for each backward step $t\in\brk{h-1}$, each partial policy $\pihat\ind{i,t}$ constructed in $\ikdptab$ is designed to target the state $i\in \cS_h$. In the BMDP setting, the states $\s_h$ are unobserved, and it is no longer the case that the partial policy $\pihat\ind{i,t} \in \Pinm^{t:h-1}$ constructed in $\ikdp$ targets the state $i\in\cS_h$. Indeed, while we will show that each partial policy $\pihat\ind{i,t}$ (approximately) targets \emph{some} state in $\cS_h$, the algorithm has no way of knowing which one.\footnote{Unless additional assumptions are added, the latent representation may only be learned up to an unknown permutation.} An additional challenge, which motivates the composition rule in \pref{line:nonmark} of \ikdp, is that for each $i\in\brk{S}$, the suffix policy $\pihat\ind{i,t+1}$ and the one-step policy $\ahat\ind{i,t}$ learned in \pref{line:second} may target \emph{different latent states}, so it does not suffice to simply construct $\pihat\ind{i,t}$ by composing them.
This motivates the second key difference between the multi-step inverse kinematics objectives used in \ikdp versus \ikdptab. The objective in $\ikdp$ predicts both actions and \emph{indices of roll-out policies} (instead of just actions, as in the tabular case) in order to learn to associate partial policies at successive layers. In particular, recall that \pref{line:second} of \ikdp defines
\[
  (\ahat\ind{i,t}(x), \iotahat\ind{i,t}(x))=
  \argmax_{(a,j)} \fhat\ind{t}((a,j) \mid
  \phihat\ind{t}(x),i), \quad x\in \cX_t.
\]
As described in \cref{sec:musik}, one should interpret $j=\iotahat\ind{i,t}(x)$ as the \emph{most likely} (or most closely associated) roll-out policy $\pihat\ind{j,t+1}$ given that the (decoded) latent state at layer $h$ is $i\in\brk{S}$ and $x\in\cX_t$ is the current observation at layer $t$. With this in mind, the composition rule in \pref{line:nonmark} constructs $\pihat\ind{i,t}$ via
\[
  \hat  \pi^{(i,t)}(x_{t:\tau})\coloneqq\left\{
\begin{array}{ll}
                            \ahat\ind{i,t}(x_t),&\quad \tau=t,\;\;
                                                  x_t\in\cX_t,\\
                            \pihat
                            ^{(\iotahat\ind{i,t}(x_t),t+1)}(x_{t+1:\tau}),&\quad\tau\in\range{t+1}{h-1},\;\;
                                                                   x_{t:\tau}\in \cX_t\times \dots \times\cX_{\tau}.
                          \end{array}
\right.
\]
For layers $t+1,\ldots,h-1$, this construction follows the policy $\pihat\ind{\iotahat\ind{i,t}(x_t),t+1}$ which---per the discussion above---is most associated with the decoded state $i\in\brk{S}$. At layer $t$, we select $a_t=\ahat\ind{i,t}(x_t)$, which maximizes the probability of reaching the decoded latent state $i\in\brk{S}$ when we roll-out with $\pihat\ind{\iotahat\ind{i,t}(x_t),t+1}$. The construction, while intuitive, is non-Markovian, since for layers $t+1$ and onward the policy depends on $x_t$ through $\iotahat\ind{i,t}(x_t)$.

\paragraph{Analysis by induction}
The proof of \pref{thm:policycover} follows the same high-level structure as \pref{thm:trend} (\musiktab), and we use the same induction strategy: For each layer $h\in\brk{H}$, we assume that $\Psi\ind{1},\dots, \Psi\ind{h-1}$ are
approximate policy covers relative to $\Pibar_\eps$ for $\cMbar$, then show that the collection $\Psi\ind{h}$ produced by $\ikdp(\Psi\ind{1},\dots,\Psi\ind{h-1}, \Phi,n)$ is an approximate cover with high probability whenever this holds. As with the tabular setting, a key component in our proof is to work with the extended BMDP and truncated policy class throughout the induction, and only pass back to the original BMDP at the end.

The following result (proven in \pref{app:geniklemmaproof}) is our main theorem concerning the performance of \ikdp, and serves as the BMDP analogue of \cref{thm:geniklemma0}.
	\begin{theorem}[Main Theorem for \ikdp]
		\label{thm:geniklemma}
                Let $\epsilon,\delta \in(0,1)$ and $h\in[H]$ be given, and define $\veps_\stat(n,\delta)\coloneqq n^{-1/2}  \sqrt{S^3 A\ln n + \ln (|\Phi|/\delta)}$. Assume that:
                \begin{enumerate}
                \item \ikdp is invoked with $\Psi\ind{1},\ldots, \Psi\ind{h-1}$ satisfying \cref{eq:inductouter}.
        \item The policies in $\Psi\ind{1},\ldots, \Psi\ind{h-1}$ never take the terminal action $\afrak$.
        \item The parameter $n$ is chosen such that $8 A  S^4 HC\veps_{\stat}(n,\frac{\delta}{ H^2})\leq \epsilon$, for some absolute constant $C>0$ independent of $h$ and other problem parameters.
      \end{enumerate}
      Then, with probability at least $1-\frac{\delta}{H}$, the collection $\Psi\ind{h}$ produced by $\ikdp(\Psi\ind{1},\dots,\Psi\ind{h-1},\Phi,n)$ is an $(1/2,\epsilon)$-policy cover relative to $\Pibar_{\epsilon}$ in $\Mbar$ for layer $h$. In addition, $\Psi\ind{h}\subseteq \Pinm^{1:h-1}$.
	\end{theorem}
        We close the section by highlighting some key differences between the proof of this result and its tabular counterpart (\cref{thm:geniklemma0}).\loose
       
\paragraph{An alternative to \cref{lem:prop}}
Recall that in the tabular setting, the proof of \cref{thm:geniklemma0} relied on \cref{lem:prop} and the performance difference lemma (see \cref{rem:performance}). In the BMDP setting, \cref{lem:prop} does not necessarily hold since, unlike in the tabular setting, successive partial policies $\pihat\ind{i,t} \in \Pinm^{t:h-1}$ and $\pihat\ind{i,t+1} \in \Pinm^{t+1:h-1}$ may target different states at layer $h$ despite sharing the same index $i\in[S]$. For this reason, we use a modified version of \cref{lem:prop}, together with a generalized version of the performance difference lemma.  
\begin{lemma}[BMDP counterpart to \cref{lem:prop}]
	\label{lem:propnew}
	There is an absolute constant $C>0$ such that for all $t\in[h-1]$, with probability at least $1-\delta/H^2$, the learned partial policies $\crl*{\pihat \ind{i,t}}_{i\in\brk{S}}$ and $\crl*{\pihat \ind{i,t+1}}_{i\in[S]}$ in $\ikdp$ have the property that for all $s_h \in \cS_{h,\epsilon}$, there exists $\i\in [S]$ such that
	\begin{align}
0	\leq 	\sum_{\pi \in \Psi\ind{t}}	\bar{d}^{\pi}(s_t) \Ebar_{\x_t \sim q(\cdot \mid s_t)}	\left[	\max_{a\in \cA, j \in[S]}	Q_t^{\pihat\ind{j,t+1}}(\x_t,a;s_h) - V_t^{\pihat\ind{\i,t}}(\x_t;s_h) \right] \leq 2S^3 A C \veps_\stat(n, \tfrac{\delta}{H^2}), \ \ \forall s_t \in \cS_t, \label{eq:propnew}
	\end{align}
where $Q_t^{\pihat\ind{j,t+1}}(x_t,a;s_h)\coloneqq \Pbar^{\pihat\ind{j,t}}[\s_h = s_h \mid \x_t = x_t, \a_t = a]$, $V_t^{\pihat\ind{\i,t}}(x_t;s_h) \coloneqq Q_t^{\pihat\ind{\i,t+1}}(x_t,\pihat\ind{\i,t}(x_t);s_h)$, and $\veps_\stat(n,\delta')\coloneqq n^{-1/2}  \sqrt{S^3 A\ln n + \ln (|\Phi|/\delta')}$.
\end{lemma}
This result is proven in \cref{sec:errorbound}. To see the similarity between \cref{lem:propnew} and \cref{lem:prop}, note that the main quantity that the latter bounds (i.e.~the quantity on the \rhs of \pref{eq:prop}) can also be written as a difference between $Q$; see \cref{rem:performance}. 
 Once \cref{lem:propnew} is established, it can be shown to imply \cref{thm:geniklemma} using a generalized variant of the performance difference lemma (\cref{lem:performancediffbmdp}).

\paragraph{Establishing \cref{eq:propnew} using multi-step inverse kinematics}
To show that \pref{eq:propnew} holds, we use the structure of the multi-step inverse kinematics objective in \pref{line:inversenew2} of \ikdp, as well as the non-Markov policy construction outlined in the prequel. In particular, we show that the multi-step inverse kinematics objective acts as a proxy for the forward kinematics given by 
\begin{align}
\P^{\pihat\ind{i,t+1}}[\s_h=\phi_\star(x_h)\mid \s_t = \phi_\star(x_t), \a_t = a], \nn
\end{align}
for $i\in[S]$, $x_t \in \cX_t$ and $x_h \in \cX_h$. We use this to show that up to statistical error, the partial policies $\crl{\pihat\ind{i,t}}_{i\in[S]}$ constructed from $\crl{\pihat\ind{i,t+1}}_{i\in[S]}$ i) identify (using observations at layer $t$) the best action at layer $t$,  and ii) identify the best partial policy from $\crl{\pihat\ind{j,t+1}}_{j\in[S]}$ to switch to from layer $t+1$ onwards.

Beyond the multi-step inverse kinematics objective and non-Markov policy construction, the proof of \cref{thm:geniklemma} uses the extended BMDP in a similar fashion to the tabular setting. We make use of the fact that for each layer $t\in\brk{h-1}$,  the policies in $\Pibar_{\epsilon}$ always play the terminal action $\afrak$ on observations emitted from states in $\cS_{t,\epsilon}$, and the generalized performance difference lemma (\cref{lem:performancediffbmdp}) is specifically designed to take advantage of this. This allows us to ``write off'' these states (analogous to \cref{eq:xxx} in the proof of \cref{lem:prop}), and use the policy cover property for $\Psi\ind{1},\ldots,\Psi\ind{h-1}$ to control the error for states in $\cS_{t,\eps}$; see \cref{sec:proofbmdp} for details. However, there is some added complexity stemming from the non-Markovian nature of $\crl{\pihat\ind{i,t}}_{i\in\brk{S}}$.

  }

 \icml{
   	 	\section{Proof Techniques}
 	\label{sec:key}

\icml{
        We find it somewhat surprising that \musik attains rate-optimal sample complexity in spite of forgoing optimism. The proof of \pref{thm:policycover}, which we sketch  in \pref{sec:overview},
        has two main components. For the first component, we prove that the multi-step inverse kinematics objective learns a decoder that can be used to drive exploration; this formalizes the intuition in \pref{sec:musik}. With this established, proving that \musik succeeds under minimum reachability (\pref{def:reachability}) is somewhat straightforward, but proving that the algorithm 1) succeeds in  absence of this assumption, and 2) achieves optimal sample complexity is more involved. For this component of the proof, we use a new analysis tool we refer to as an \emph{extended BMDP} which, in tandem with another tool we refer to as the \emph{truncated policy class}, allows one to emulate certain consequences of reachability even when the condition does not hold. These techniques, which we anticipate will find broader use in the context of non-optimistic algorithms based on policy covers, appear to be new even for tabular reinforcement learning.

\textbf{\emph{Due to space limitations, an in-depth overview of the analysis of \musik is deferred to \pref{sec:overview}.}} 

As a teaser, in this section we introduce the most important technical tools used in the proof of \cref{thm:policycover}, the extended BMDP and truncated policy class, which, play a key role in providing tight guarantees for \musik (and more broadly, non-optimistic algorithms) in the absence of minimum reachability.
\emph{We recommend reading the full overview in \pref{sec:overview} before diving into the full proof of \pref{thm:policycover} (\pref{sec:BMDP}). }
        }

\arxiv{
  In this section, we give a teaser for the most important technical tools used in the proof of \cref{thm:policycover}; see \cref{sec:overview} for an in-depth overview of the analysis that builds on these ideas. In the sequel, we will introduce two analysis tools, the \emph{extended BMDP} and \emph{truncated policy class}, which, play a key role in providing tight guarantees for \musik (and more broadly, non-optimistic algorithms) in the absence of minimum reachability.
  }

\paragraph{Analysis in Extended BMDP}
\musik proceeds by inductively building a sequence of policy covers $\Psi\ind{1},\ldots,\Psi\ind{H}$. A key invariant maintained by the algorithm is that for each layer $h$, $\Psi\ind{1},\ldots,\Psi\ind{h-1}$ provide good coverage for layers $1,\ldots,h-1$, and thus can be used to gather data which will allow us to efficiently learn $\Psi\ind{h}$. Prior approaches that build policy covers in a similar fashion \citep{du2019provably,misra2020kinematic} require the minimum reachability assumption (\pref{def:reachability}) to ensure that for each $h$, $\Psi\ind{h}$ \emph{uniformly} covers all states in $\cS_h$. In the absence of reachability, we inevitably must sacrifice certain hard-to-reach states, which requires a more refined analysis. In particular, we must show that the effect of ignoring hard-to-reach states at earlier layers do not compound as the algorithm progresses.

\zm{Double the notation $\tfrak_h$ is consistent.}
To provide such an analysis, we make use of an extended BMDP $\Mbar$. The extended BMDP $\Mbar$ augments $\cM$ by adding $H$ ``terminal'' states $\tfrak_{1:H}$ and one additional \emph{terminal} action $\afrak$ as follows:
\icml{{\bf I})  The latent state space is $\wbar \cS \coloneqq \bigcup_{h=1}^H \wbar \cS_h$, where $\wb{\cS}_h\ldef\cS_h\cup\crl{\tfrak_h}$; {\bf II}) the action space is $\cAbar\ldef\cA \cup\crl{\afrak}$, where $\afrak$ is an action that deterministically transitions to $\term_{h+1}$ from every state at layer $h\in[H-1]$; and {\bf III)} For $h\in[H-1]$, taking any action in $\wbar\cA$ at state $\tfrak_h$ transitions to $\tfrak_{h+1}$ deterministically.}
We assume the state $\tfrak_h$ emits itself as an observation and we write $\wb{\cX}_h \coloneqq  \cX_h \cup\{\tfrak_h\}$, for all $h\in[H]$. The dynamics of $\cMbar$ are otherwise identical to $\cM$, and for any policy $\pi \in \Pibar_{\nm} \coloneqq \left\{\pi \colon \bigcup_{h=1}^H (\wbar \cX_{1}\times \dots \times \wbar\cX_h) \rightarrow \Abar\right\}$, we define $\Pbar^{\pi}\coloneqq \P^{\Mbar,\pi}$, $\Ebar^{\pi}\coloneqq \E^{\Mbar,\pi}$, and $\dbar^{\pi}(s)\coloneqq \P^{\Mbar,\pi}[\s_h = s]$, for all $s\in \cS_h$ and $h \in[H]$.
\arxiv{
\begin{align*}
	\Pbar^{\pi}\coloneqq \P^{\Mbar,\pi}, \quad  \Ebar^{\pi}\coloneqq \E^{\Mbar,\pi},\quad  \text{and} \quad \dbar^{\pi}(s)\coloneqq \P^{\Mbar,\pi}[\s_h = s], \quad \text{for all } s\in \cS_h \text{ and } h \in[H].
\end{align*}
}
\paragraph{Truncated policy class} On its own, the extended BMDP is not immediately useful. The most important idea behind our analysis is to combine it with a restricted sub-class of policies we refer to as the truncated policy class. Define $\Pibarm  \coloneqq \{ \pi\colon \bigcup_{h=1}^H \wbar\cX_h \rightarrow \Abar\}$. For $\epsilon \in(0,1)$, we define a sequence of policy classes $\Pibar_{0, \epsilon},\dots,\Pibar_{H, \epsilon}$, inductively, starting from $\Pibar_{0, \epsilon}=\Pibarm$ and letting $\Pibar_{t, \epsilon}$ be the set for which $\pi \in \Pibar_{t, \epsilon}$ if and only if
\icml{
\begin{align}
&\exists \pi'\in \Pibar_{t-1, \epsilon}\ \text{ such that } \forall h \in[H], \forall s\in \wbar\cS_h, \forall x\in \phi^{-1}_\star(s), \nn \\	&\pi(x) = \left\{\begin{array}{ll} \fraka, & \text{if } h=t \text{ and } \max_{\tilde \pi\in \Pibar_{t-1, \epsilon}} \dbar^{\tilde \pi}(s)< \epsilon, \\   \pi'(x), & \text{otherwise}.    \end{array}    \right. \nn %
\end{align}
}
\arxiv{
	\begin{align}
\exists \pi'\in \Pibar_{t-1, \epsilon}, \forall h \in[H], \forall s\in \cS_h, \forall x\in \phi^{-1}_\star(s), \ 	\pi(x) = \left\{\begin{array}{ll} \fraka, & \text{if } h=t \text{ and } \max_{\tilde \pi\in \Pibar_{t-1, \epsilon}} \dbar^{\tilde \pi}(s)< \epsilon, \\   \pi'(x), & \text{otherwise}.    \end{array}    \right.  \label{eq:define00}
	\end{align}
}
Restated informally, the class $\Pibar_{t,\eps}$ is identical to $\Pibar_{t-1,\eps}$, except that at layer $t$, all policies in the class take the terminal action $\afrak$ in latent states $s$ for which $\max_{\tilde \pi\in \Pibar_{t-1, \epsilon}} \dbar^{\tilde \pi}(s)< \epsilon$.

We define the \emph{truncated policy class} as $\Pibar_{\epsilon} \coloneqq \Pibar_{H,\epsilon}$. The truncated policy class satisfies two fundamental technical properties. First, by construction, all policies in the class take the terminal action $\afrak$ when they encounter states that are not $\eps$-reachable by $\Pibar_\eps$. The next lemma formalizes this.
\begin{lemma}
	\label{lem:bla}
	Let $\epsilon\in(0,1)$ be given, and define $\cS_{h,\epsilon} \ldef  \{  s\in \cS_h  \colon	\max_{\pi \in \Pibar_{\epsilon}}\dbar^{\pi}(s) \geq \epsilon\}.$
	Then for all $h \in[H]$ if $s\in \cS_{h }\setminus\cS_{h,\epsilon}$, then $\pi(x)=\fraka$, for all $x\in \phi^{-1}_\star(s)$ and $\pi \in \Pibar_\epsilon$.
\end{lemma}
Second, in spite of the fact that policies in $\Pibar_{\epsilon}$ always take the terminal action on states with low visitation probability, they can still achieve near-optimal visitation probability for all states in $\cMbar$ (up to additive error).
\begin{lemma}[Approximation for truncated policies]
	\label{lem:critical}
	Let $\epsilon\in(0,1)$ be given. For all $h \in[H]$ and $s\in \cS_h$, 
	\label{lem:fall}
	\begin{align}
		\max_{\pi \in \Pibarm} \dbar^{\pi}(s) \leq \max_{\pi \in \Pibar_{\epsilon}} \dbar^{\pi}(s) +S \epsilon.\nn 
	\end{align}
\end{lemma}
The proofs for these results (and other results in this subsection) are elementary, and are given in \pref{app:structural}.       
Building on these properties, our proof of \pref{thm:policycover} makes use of two key ideas:
\begin{enumerate}[leftmargin=*]
	\item Even though the extended BMDP $\cMbar$ does not necessarily enjoy minimum reachability (\pref{def:reachability}), if we restrict ourselves to competing against policies in $\Pibar_\eps$, \pref{lem:bla} will allow us to ``emulate'' certain properties enjoyed by $\eps$-reachable MDPs. This in turn will imply that if we only wish to learn a policy cover that has good coverage relative to $\Pibar_\eps$, \pref{alg:GenIk} will succeed.
	\item By \pref{lem:critical}, we lose little by restricting our attention to the class $\Pibar_\eps$. This will allow us to transfer any guarantees we achieve with respect to the extended BMDP $\cMbar$ and truncated policy class $\Pibar_\eps$ back to the original BMDP $\cM$ and unrestricted policy class $\Pim$.
\end{enumerate}
\arxiv{
We will make the first point precise in the sections that follow (\cref{sec:warmup,sec:block}).}
\icml{
We will make the first point precise in \cref{sec:warmup,sec:block}.} For now, we formalize the second point via another technical result, \pref{lem:transfer}. To do so, we introduce the notion of a \emph{relative policy cover} (generalizing \cref{def:polcover101}).
\begin{definition}[Relative policy cover]
	\label{def:polcover}
	Let $\alpha, \veps\in[0,1)$ be given. Consider a BMDP $\cM'$, and let $\Pi$ and $\Psi$ be two sets of policies. We say that $\Psi$ is an $(\alpha,\veps)$-policy cover relative to $\Pi$ for layer $h$ in $\cM'$, if
	\arxiv{
	\begin{align}
		\max_{\pi \in \Psi} d^{\cM',\pi}(s) \geq  \alpha \cdot \max_{\pi \in \Pi} d^{\cM',\pi}(s) \quad \text{for all } \  s\in \cS_h \ \text{ such that }\  \max_{\pi \in \Pi
		} d^{\cM',\pi}(s)  \geq \veps. \nn
	\end{align}
}
	\icml{\begin{align}
		\max_{\pi \in \Psi} d^{\cM',\pi}(s) \geq  \alpha \cdot \max_{\pi \in \Pi} d^{\cM',\pi}(s), \nn 
		\end{align} 
	for all $s\in \cS_h$ such that $\max_{\pi \in \Pi
} d^{\cM',\pi}(s)  \geq \veps$. 
}
\end{definition} 

\begin{lemma}[Policy cover transfer]
	\label{lem:transfer}
	Let $\veps\in(0,1)$ be given, and define $\eps\ldef\veps/(2S)$.
	Let $\Psi$ be a set of policies for $\cMbar$ that never take the terminal action $\afrak$. If $\Psi$ is a $(1/2,\epsilon)$-policy cover relative to $\Pibar_\epsilon$ in $\Mbar$ for all layers, then $\Psi$ is a $(1/4,\veps)$-policy cover relative to $\Pim$ in the $\cM$ for all layers.
\end{lemma}
\pref{lem:transfer} implies that for any $\veps$, letting $\eps\ldef{}\veps/2S$, if we can construct a set $\Psi$ that acts as a
$(1/2,\eps)$-policy cover relative to
$\Pibar_{\epsilon}$ in $\Mbar$, then $\Psi$ will also be a $(1/4,\veps)$-policy cover relative to $\Pim$ in the original BMDP $\cM$, which is ultimately
what we wish to accomplish. \arxiv{This allows us to restrict our attention to the former goal going forward.}
\icml{This allows us to restrict our attention to the former goal in the analysis.}

We refer the reader to \cref{sec:overview} for the overview of the analysis of \pref{thm:policycover}, which builds on the tools presented in this section, and to \cref{sec:BMDP} for the full proof. We anticipate that these techniques will find broader use in RL.\arxiv{: For example, they can be used to show that $\homer$ (without any modifications) successfully learns an $\veps$-optimal policy using $O(1/\veps^3$) samples (which is better than $\texttt{BRIEE}$ but worse than $\musik$) \emph{without} requiring minimum reachability. }

  }

	\section{Experiments}
	\label{sec:experiments}

\newcommand{\cbr}[1]{\{#1\}}
\newcommand{\rbr}[1]{(#1)}
\newcommand{\actions}{\{a_i\}}
\newcommand{\action}{a}
\newcommand{\states}{\{s_i\}}
\newcommand{\state}{s}
\newcommand{\policy}{\pi}
\newcommand{\obs}{x}
\newcommand{\briee}{\texttt{BRIEE}}

\arxiv{
As a validation experiment, we evaluate the
performance of \musik on
the challenging ``diabolical combination lock'' (``\comblock'')
environment \citep{misra2020kinematic,zhang2022efficient}, which
combines high-dimensional observations with anti-shaped, sparse
rewards, necessitating representation learning and systematic
exploration\footnote{The code for our experiments is available at \url{https://github.com/zmhammedi/MusIK.git}.}. We compare $\musik$ to $\homer$ \citep{misra2020kinematic} and
$\briee$ \citep{zhang2022efficient} which, amongst provably efficient
algorithms, have the best known empirical performance
\citep{zhang2022efficient}. 

\paragraph{Environment} The \comblock environment of
\citet{misra2020kinematic,zhang2022efficient} is parameterized
by the horizon $H$. There are $A=10$ actions, and at each layer $h\in[H]$, there are $N=3$ states $s_{h,1}, s_{h,2}, s_{h,3}\in \cS_h$, where $s_{h,1},s_{h,2}$
are ``good'' states and $s_{h,3}$ is a ``bad'' terminal state. For each layer $h\in[H]$, there exists a pair
of ``good'' actions $u_{h,1}, u_{h,2} \in \cA$ such that taking action
$u_{h,j}$ in state $s_{h,j}$, for $j\in\{1,2\}$, leads to one of the
good states $\{s_{h+1,1}, s_{h+1,2}\}$ at next layer with equal
probability.  All actions $a_h\not\in \{u_{h,1}, u_{h,2}\}$ lead to
the bad state $s_{h+1,3}$ deterministically. The sequences of good actions $u_{1:H,1}$ and $u_{1:H,2}$ are sampled uniformly at random from the set of actions $\cA$ when the environment is initialized, and are unknown to the learner. 

For $h=H$, the agent receives a reward of $1$ if action $u_{H,j}$ is
taken in state $s_{H,j}$, for $j\in\{1,2\}$, and receives reward of 0
otherwise. For $h<H$, the agent receives an anti-shaped reward of
$0.1$ for choosing any action $a_h \neq u_{h,j}$ in state $s_{h,j}$, for
$j\in\{1,2\}$, and receives a reward of 0 otherwise (in particular,
the agent never receives a reward in the bad state $s_{h,3}$). This anti-shaped
reward encourages the agent to take actions that lead to the bad state
$s_{h+1,3}$, from which it is not possible to reach the good states
$\{s_{H,1}, s_{H,2}\}$ at layer $H$ and achieve the optimal reward of
$1$.

The agent does not observe the states $\{\s_h\}$ directly, and instead
receives observations $\{\x_h\}$. For each $h$, the observation $\x_h$
is a $d$-dimensional vector, where $d\ldef{}2^{\ceil{\log_2
    (H+N+1)}}$, obtained by concatenating the one-hot vector of the latent state $\s_h$ and the one hot vector of the layer index $h$, followed by adding noise sampled from $\cN(0,0.1)$ in one dimension, padding with zeros if necessary, and multiplying with a Hadamard matrix. Strictly speaking, the CombLock
environment is more challenging than a Block MDP, since two latent states can emit the
same observation due to the addition of the Gaussian noise in the observation process.

Since the good actions $\{u_{h,1}, u_{h,2}\}_{h\in[H]}$ are not known
to the agent, deliberate exploration is required to learn a policy
that maximizes the reward function (note that it is only possible to achieve reward $1$ if the agent selects a good action for all
$h\in\brk{H}$). For example, when the horizon is set to $H=100$, the probability
of finding the optimal policy through naive uniform exploration is $10^{-100}$. In
addition, representation learning is required to recover the latent
state $\s_h$ from the observation $\x_h$ at each layer, with the best
decoder depending on the layer $h$.

\paragraph{$\musik$ implementation} We use \musik to compute a
policy cover for the \comblock, then apply \psdp with the cover to
optimize the reward function. Following the approach taken with
$\homer$ in
\cite{misra2020kinematic}, we take advantage of the fact that in the CombLock
environment, the optimal policy can be learned by composing optimal
homing policies for each layer (this is
not true for general BMDPs). In particular, we take advantage of this composability property to implement a more sample-efficient
version of $\musik$ (see \cref{alg:IKDPcomp} in
\cref{app:experiments}). We parametrize all models with neural networks and optimize $\musik$ and $\psdp$'s objectives
using stochastic gradient descent via \texttt{PyTorch}. 

\paragraph{Baselines} 
As baselines, we use $\homer$ \citep{misra2020kinematic} and
$\briee$ \citep{zhang2022efficient}. Amongst provably efficient
algorithms, these methods are known to have the best empirical performance
\citep{zhang2022efficient} on the \comblock environment\footnote{We compare only against other model-free
	methods, and do not consider model-based approaches \citep{uehara2022,zhang2022making,ren2022latent}.}. The \homer{} algorithm has the same
structure as \musik: it first learns a policy cover, then
uses the cover within $\psdp$ to learn a near-optimal policy. The
$\briee$ algorithm does not explicitly learn a policy cover, but
rather interleaves exploration and exploitation using optimism. We do not reproduce
  $\briee$ and $\homer$, and instead report the results from
  \citet{zhang2022efficient}.

  Additional details for the experimental setup are
given in \pref{app:experiments}.
\icml{
	\begin{figure}
		\centering
		\includegraphics[width=\linewidth]{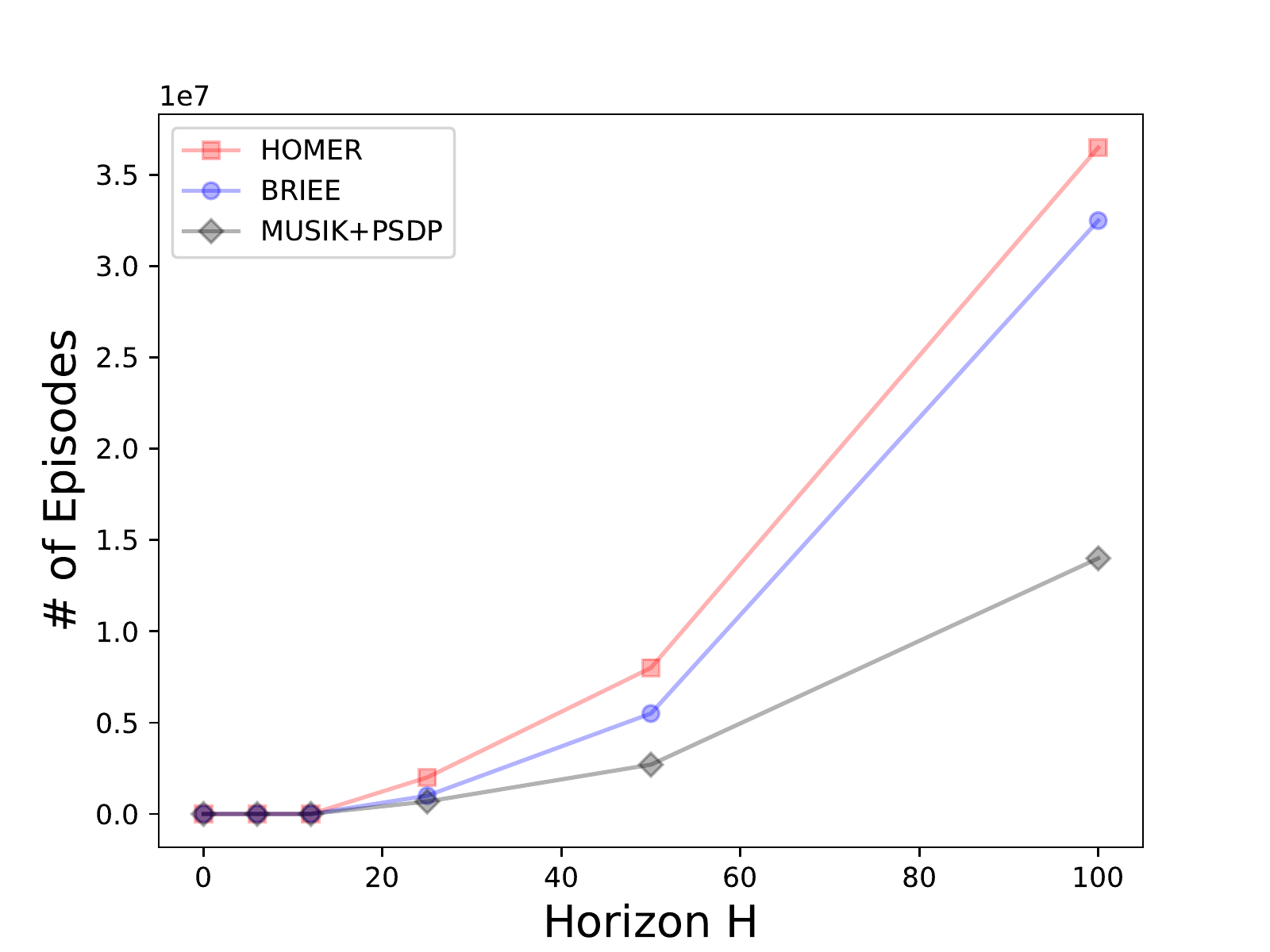}
		\caption{Number of episodes required to identify the optimal
			policy, as a function of the horizon $H$.}
		\label{fig:horizon-plot} 
	\end{figure}
}
\arxiv{
		\begin{figure}
		\centering
		\includegraphics[width=.5\linewidth]{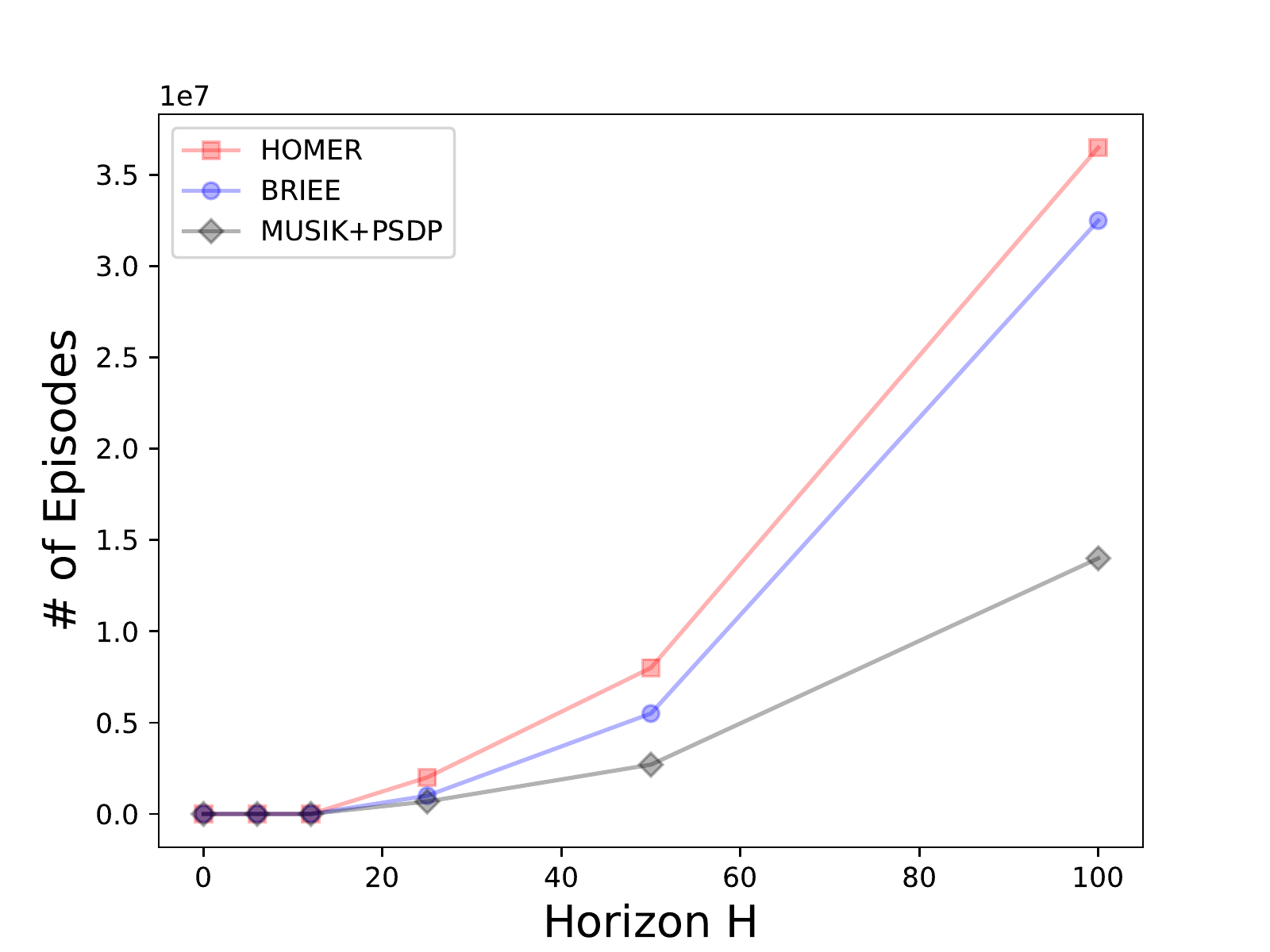}
		\caption{Number of episodes required to identify the optimal
			policy, as a function of the horizon $H$.}
		\label{fig:horizon-plot} 
	\end{figure}
}
\paragraph{Evaluation and results} 
\cref{fig:horizon-plot} reports the number of episodes (or, number of
sampled trajectories) required for each method to identify the optimal
policy, as a function of the horizon $H$.\footnote{We declare the returned policy $\pihat$ to be
  optimal if the average reward over $50$ trajectories is $1$.} For $\musik$, we plot the number of episodes required to find the optimal policy across five different initialization
seeds. For $\briee$ and $\homer$, we report the number of episodes required to find the optimal policy across three out of five different initialization seeds; note that this only improves the results for
the baseline methods compared to \musik.

We find that for small values of $H$, all methods
have similar performance, but for large horizon, $\musik$
outperforms the baselines. For $H=100$, $\musik$ is able to
find the optimal policy using almost three times fewer episodes than
$\homer$ and $\briee$.
This suggests that the
multi-step inverse kinematics objective in \musik may indeed carry
practical (as opposed to just theoretical)
benefits over alternative representation learning approaches. 
}

\icml{
	As a validation experiment, we evaluate the
	performance of \musik on
	the challenging ``diabolical combination lock'' (``\comblock'')
	environment \cite{misra2020kinematic,zhang2022efficient}, which
	combines high-dimensional observations with anti-shaped, sparse
	rewards, necessitating representation learning and systematic
	exploration.

	\paragraph{Environment} We adopt the diabolical combination
        lock (\comblock) environment from
	\citet{misra2020kinematic,zhang2022efficient}, which is parameterized
	by the horizon $H$ and number of actions $A=10$. At each layer $h$, there are $N=3$ states $s_{h,1}, s_{h,2}, s_{h,3}\in \cS_h$, where $s_{h,1},s_{h,2}$
	are ``good'' states and $s_{h,3}$ is a ``bad'' terminal state. For each layer $h$, there exists a pair
	of ``good'' actions $u_{h,1}, u_{h,2} \in \cA$ such that taking action
	$u_{h,j}$ in state $s_{h,j}$, for $j\in\{1,2\}$, leads to one of the
	good states $\{s_{h+1,1}, s_{h+1,2}\}$ at next layer with equal
	probability.  All actions $a_h\not\in \{u_{h,1}, u_{h,2}\}$ lead to
	the bad state $s_{h+1,3}$ deterministically. The sequences of
        good actions $u_{1:H,1}$ and $u_{1:H,2}$ are sampled uniformly
        at random from the set of actions $\cA$ when the environment
        is initialized and are unknown to the learner. The optimal reward of 1 can only be achieved at the states $s_{H,1}$ and
        $s_{H,2}$ (we postpone the details of the reward and observation processes to \cref{app:experiments}).

	Since the good actions $\{u_{h,1}, u_{h,2}\}_{h\in[H]}$ are not known
	to the agent, deliberate exploration is required to learn a policy
	that maximizes the reward function; note it is only possible
	achieve reward $1$ if the agent selects a good action for all
	$h\in\brk{H}$. For example, when the horizon is set to $H=100$, the probability
	of finding the optimal policy through naive uniform exploration is $10^{-100}$. In
	addition, representation learning is required to recover the latent
	state $\s_h$ from the observation $\x_h$ at each layer, with the best
	decoder depending on the layer $h$.

		\begin{figure}
			\centering
			\includegraphics[width=\linewidth]{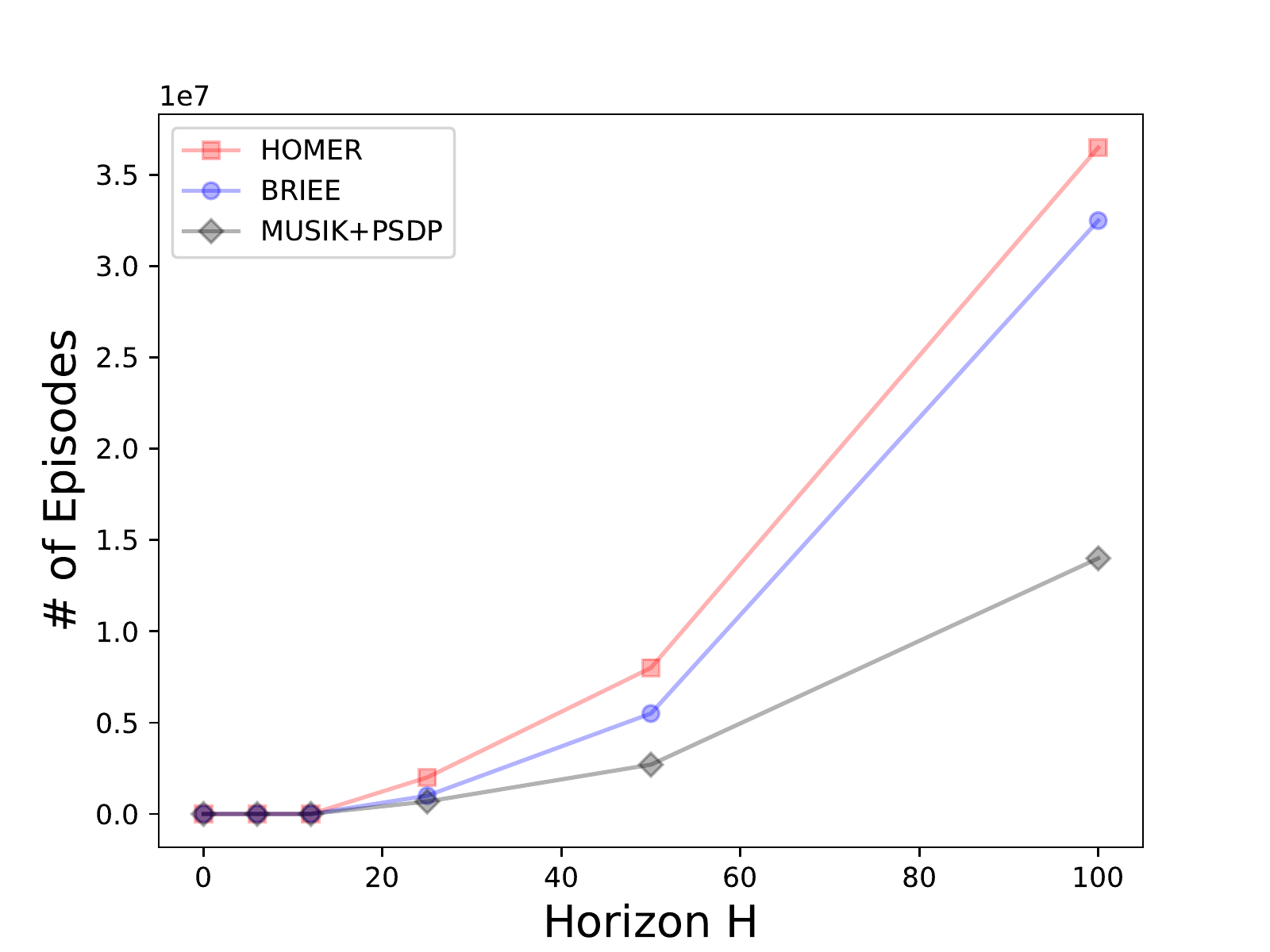}
			\caption{Number of episodes required to identy the optimal
				policy, as a function of the horizon $H$ for the CombLock experiment.}
			\label{fig:horizon-plot} 
		\end{figure}
	\paragraph{Evaluation and results} 
	We compare $\musik$ to $\homer$ \citep{misra2020kinematic} and
	$\briee$ \citep{zhang2022efficient} which, amongst provably efficient
	algorithms, have the best known empirical
        performance.\footnote{We compare only against other model-free
          methods, and do not consider model-based approaches \citep{uehara2022,zhang2022making,ren2022latent}.} For \musik, we adopt the decoder class $\Phi \coloneqq
\{\phi_W\colon x \mapsto\argmax_{i\in[N]} W x  \mid W\in
\reals^{N\times d} \}$, which is the same as that used in \citep{misra2020kinematic}. See \cref{app:experiments} for implementation details.
        We do not reproduce $\briee$ and $\homer$, and instead report the results from \citet{zhang2022efficient}. %

	\cref{fig:horizon-plot} reports the number of episodes (or, number of
	sampled trajectories) required for each method to identify the optimal
	policy, as a function of the horizon $H$; we declare the returned policy $\pihat$ to be
		optimal if the average reward over $50$ trajectories is $1$.
For $\musik$, we plot the
\emph{worst-case} number of episodes across 5 different initialization
seeds. For $\briee$ and $\homer$, we only report the median (instead of
the worst-case) number of trajectories over 5 different seeds required to
	find the optimal policy; note that this only improves the results for
	the baseline methods compared to \musikp.
We find that for small values of $H$, all methods
	have similar performance, but for large horizon, $\musik$
	outperforms the baselines. In particular, for $H=100$, $\musik$ is able to
	find the optimal policy using almost three times fewer episodes than
	$\homer$ and $\briee$.
	This suggests that the
	multi-step inverse kinematics objective in \musik may indeed carry
	practical (as opposed to just theoretical)
	benefits over the alternative representation learning approaches used
	in \homer{} and \briee{}. Performing a large scale evaluation is a
	promising direction for future research.
	
}

\arxiv{        
	\section{Discussion}
	\label{sec:discussion}
Our results suggest a number of exciting directions for future
research. First, while \musik attains rate-optimal sample complexity
with respect to the accuracy parameter $\veps>0$, it has loose
dependence on the parameters $H$, $S$, and $A$, similar to other
efficient algorithms
\citep{misra2020kinematic,zhang2022efficient}. We anticipate that
achieving efficiently achieving \emph{minimax optimal} sample
complexity will require further algorithmic improvements, as well as
refinements to our analysis techniques. More broadly, we are excited to explore whether our
techniques can be applied beyond the basic BMDP model, with possible
examples including representation learning with linear function approximation
\citep{agarwal2020flambe,modi2021model,uehara2022}, and Block MDPs
with exogenous noise \citep{efroni2021provable,efroni2022sample}.

\subsection*{Acknowledgements}
We thank John Langford, Dipendra Misra, and Akshay Krishnamurthy for
several helpful discussions. ZM and AR acknowledge support from the ONR through awards N00014-20-1-2336 and N00014-20-1-2394.

         }
	
	\clearpage
	
	\bibliography{refs}
	
	\clearpage
	
	\appendix
\icml{
	\onecolumn
\renewcommand{\contentsname}{Contents of Appendix}
\addtocontents{toc}{\protect\setcounter{tocdepth}{2}}
{
  \hypersetup{hidelinks}
  \tableofcontents
}
\clearpage
}

\icml{
\part{Additional Details and Results}

\section{Omitted Tables and Pseudocode}
\label{sec:omitted}

        \begin{table}[tp]
        	               \caption{Comparison of sample complexity required learn an $\veps$-optimal
        		policy.
        		For approaches that require a minimum
        		reachability assumption $\eta_{\min} \coloneqq \min_{s\in
        			\cS}\max_{\pi \in \Pim}d^{\pi}(s)$ denotes
        		the reachability parameter. $\Phi$ and $\Psi$ denote the decoder and model classes, respectively. 
        	}
        	\label{tb:resultscomp}
          \renewcommand{\arraystretch}{1.6}
		\fontsize{9}{10}\selectfont
		\centering 
		\begin{tabular}{ccccc}
			\hline
			& Sample complexity & Model-free &
                                                           Comp. efficient
                  &\makecell{Rate-optimal \\ $1/\veps^2$-sample
                  comp.} \\
			\hline
			$\olive$ \citep{jiang2017contextual} & $\frac{A^2 H^3 S^3 \ln |\Phi|}{\veps^2}$ & Yes & No & Yes \\
			$\texttt{MOFFLE}$ \citep{modi2021model} & $\frac{A^{13 }H^8 S^7 \ln |\Phi|}{(\veps^2\eta_{\min} \wedge \eta_{\min}^5)}$ & Yes& Yes &  No \\
			$\homer$ \citep{misra2020kinematic} &
			$\frac{A H S^6 (S^2 A^3+\ln |\Phi|)}{(\veps^2 \wedge \eta_{\min}^3)}$ &Yes & Yes &  No\\
			$\texttt{Rep-UCB}$ \citep{uehara2022} &$\frac{ A^2 H^5 S^4 \ln (|\Phi|\textcolor{red!70!black}{|\Psi|}) }{\veps^2}$ & No & Yes&  Yes \\
			$\texttt{BRIEE}$ \citep{zhang2022efficient} & $ \frac{A^{14} H^9 S^8 \ln |\Phi|}{\veps^4}$& Yes & Yes & No  \\
			$\musik$ (this paper) & $\frac{A^2 H^4 S^{10} (A S^3 + \ln |\Phi|)}{\veps^2}$& \textbf{Yes} &  \textbf{Yes} &  \textbf{Yes}\\
			\hline
		\end{tabular}
	\end{table}

 \arxiv{
 	\begin{algorithm}[t]
 	\caption{Execute non-Markov partial policy
          produced by $\musik$.}
 	\label{alg:gen}
 	\begin{algorithmic}[1]\onehalfspacing
 		\Require 
                ~
        \begin{itemize}[leftmargin=*]
        \item Indices $t,h\in[H]$ such that $t<h$.
        \item Index $i\in\brk{S}$.
          \hfill\algcommentlight{Index for policy
            $\pihat\ind{i,t}\in\Pinm^{t:h}$ produced in \cref{line:nonmark} of \cref{alg:IKDP}.}
                    \item                 Initial observation $x_t\in
        \cX_t$.
                      \item Functions $(\hat f\ind{t},\phihat\ind{t}),\dots, (\hat
        f\ind{h-1}, \phihat\ind{h-1})$ produced in 
        \cref{line:inversenew2} of \cref{alg:IKDP}.
                  \end{itemize}
 	    \State Set $\bj_{t-1} = i$.
 		\For{$\tau=t,\ldots, h-1$}
 		\State Compute $(\a_\tau, \bm{j}_\tau)\in \argmax_{(a,j)} \fhat\ind{\tau}((a,j) \mid \phihat\ind{\tau}(\x_\tau),\bm{j}_{\tau-1})$. \label{line:argmax}
 		\State Play action $\a_\tau$ at layer $\tau$ and
                observe $\x_{\tau+1}$.
 		\EndFor
                \State \textbf{Return:} Partial trajectory
                $(\a_{t:h-1},\x_{t:h})$ generated by $\pihat\ind{i,t}\in
                \Pinm^{t:h-1}$ (\cref{line:nonmark}).
 	\end{algorithmic}
 \end{algorithm}	

}

\icml{
 
\begin{algorithm}[ht]
	\caption{Execute non-Markov partial policy
		produced by $\musik$.}
	\label{alg:gen}
	\begin{algorithmic}[1]\onehalfspacing
		\Require 
		~Indices $t,h\in [H]$ and $i\in\brk{S}$ (Indexes policy
			$\pihat\ind{i,t}\in\Pinm^{t:h-1}$ produced in \cref{line:nonmark} of \cref{alg:IKDP}).
		Initial observation $x_t\in
		\cX_t$. Functions $(\hat f\ind{t},\phihat\ind{t}),\dots, (\hat
		f\ind{h-1}, \phihat\ind{h-1})$ produced in 
		\cref{line:inversenew2}.
		\State Set $\bj_{t-1} = i$.
		\For{$\tau=t,\ldots, h-1$}
		\State $(\a_\tau, \bm{j}_\tau)\gets \argmax\limits_{(a,i)} \fhat\ind{\tau}((a,i) \mid \phihat\ind{\tau}(\x_\tau),\bm{j}_{\tau-1})$ \label{line:argmax}
		\State Play action $\a_\tau$ at layer $\tau$ and
		observe $\x_{\tau+1}$.
		\EndFor
		\State \textbf{Return:} Partial trajectory
		$(\a_{t:h-1},\x_{t:h})$ generated by $\pihat\ind{i,t}\in
		\Pinm^{t:h-1}$ (\cref{line:nonmark} of \cref{alg:IKDP}).
	\end{algorithmic}
\end{algorithm}

}

	\section{Application to Reward-Based RL: Planning with an Approximate Cover}
	\label{sec:rewardsetting} 
\arxiv{
 To conclude the section, we show how the policy cover learned by $\musik$ can be used to optimize any downstream reward function of interest.}
\icml{In this section, we show how the policy cover learned by $\musik$ can be used to optimize any downstream reward function of interest. }
 For the results that follow, we assume that at each layer $h\in\brk{H}$, the learner observes a reward $\br_h\in\brk*{0,1}$ in addition to the observation $\x_h\in\cX$, so that trajectories take the form $(\x_1,\a_1,\br_1),\ldots,(\x_H,\a_H,\br_H)$.  We will make the following standard BMDP assumption \citep{misra2020kinematic,zhang2022efficient}, which asserts that the mean reward function depends only on the latent state, not the full observation.
\begin{assumption}[Realizability]
	\label{assum:reward}
	For all $h\in[H]$, there exists $\rbar_h\colon \cS \times \cA \rightarrow [0,1]$ such that $\E[\br_h \mid \x_h=x,\a_h=a]=\rbar_h(\phi_\star(x),a)$. 
\end{assumption}

\begin{algorithm}
	\caption{$\texttt{PSDP}$: Policy Search by Dynamic Programming
		(variant of \citet{bagnell2003policy})}
	\label{alg:PSDP}
	\begin{algorithmic}[1]\onehalfspacing
		\Require Policy cover $\Psi\ind{1},\dots,\Psi\ind{H}$. Decoder class $\Phi$. Number of samples $n$.
		\For{$h=H, \dots, 1$} 
		\State $\cD\ind{h} \gets\emptyset$. 
		\For{$n$ times}
		\State Sample $(\x_h, \a_h, \bm{r}_{h:H})\sim
		\unif(\Psi\ind{h})\circ_h \unif (\cA) \circ_{h+1} \pihat \ind{h+1}$.
		\State Update dataset: $\cD\ind{h} \gets \cD\ind{h} \cup \left\{\left(\x_h, \a_h, \sum_{t =h}^H \br_{t}\right)\right\}$.
		\EndFor
		\State Solve regression:
		\[(\fhat\ind{h},\phihat\ind{h}) \gets\argmin_{f \colon [S]\times \cA\rightarrow [0,H-h +1], \phi \in \Phi}  \sum_{(x, a, R)\in\cD} (f(\phi(x),a)-R)^2.\] \label{eq:mistake}
		\State Define $\pihat\ind{h}\in\Pim^{h:H}$ via
		\[
		\pihat \ind{h}(x) = \left\{
		\begin{array}{ll}
			\argmax_{a\in \cA} \fhat\ind{h}(\phihat\ind{h}(x),a),&\quad x\in\cX_h,\\
			\pihat\ind{h}(x),& \quad x\in \cX_{t},\;\;  t\in [h+1 \ldotst H].
		\end{array}
		\right.
		\]
		\EndFor
		\State \textbf{Return:} Near-optimal policy $\pihat \ind{1}\in \Pim$. 
	\end{algorithmic}
\end{algorithm}

\paragraph{The PSDP algorithm}
To optimize rewards, we take a somewhat standard approach and appeal to a variant of the Policy Search by Dynamic Programming (\psdp) algorithm of \citet{bagnell2003policy,misra2020kinematic}. $\psdp$ uses the approximate policy cover produced by \musik as part of a dynamic programming scheme, which constructs a near-optimal policy in a layer-by-layer fashion. In particular, starting from layer $H$, $\psdp$ first constructs a partial policy $\hat
\pi\ind{H}\in \Pim^{H:H}$ using data collected with $\Psi\ind{H}$, then moves back a layer and constructs a partial policy
$\pihat \ind{H-1}\in \Pim^{H-1:H}$ using data collected with $\Psi\ind{H-1}$ and $\pihat\ind{H}$,
and so on, until the first layer is reached. The variant of $\psdp$ we present here differs slightly from the original version in \cite{bagnell2003policy,misra2020kinematic}, with the main difference being that instead of using a policy optimization sub-routine to compute the policy for each layer, we appeal to least-squares regression (see \cref{eq:mistake} of \cref{alg:PSDP}) to estimate a $Q$-function, and then select the greedy policy this function induces.

The following result, proven in \cref{app:PSDPthmproof}, provides the main sample complexity guarantee for \psdp.\footnote{This result does not immediately follow from prior work \citep{misra2020kinematic} because it allows for an $(\alpha,\veps)$-policy cover with $\veps>0$; previous work only handles the case where $\veps=0$.}
\begin{theorem}
	\label{thm:PSDPthm}
	Let $\alpha$, $\veps$, $\delta \in(0,1)$ be given. Suppose that \cref{assum:real,assum:reward} hold, and that for all $h\in\brk{H}$:
	\begin{enumerate}
		\item $\Psi\ind{h}$ is a $(\alpha, \eps)$-approximate cover for layer $h$, where $\eps\coloneqq \veps/(2SH^2)$.
		\item $|\Psi\ind{h}|\leq S$.
	\end{enumerate}
	Then, for appropriately chosen $n\in\bbN$, the policy $\pihat\ind{1}$ returned by \pref{alg:PSDP} satisfies
	\begin{align}
		\E^{\pihat\ind{1}}\left[\sum_{h=1}^H \br_h\right]\geq    \max_{\pi \in \Pim}  \E^{\pi}\left[\sum_{h=1}^H \br_h\right] - \veps \nn
	\end{align}
	with probability at least $1-\delta$. Furthermore, the total number of sampled trajectories used by the algorithm is bounded by
	\begin{align}
	\nn
		\bigoht(1)\cdot \frac{ H^5 S^6 (S A + \ln (|\Phi|/\delta))}{ \alpha^2 \veps^2}.
	\end{align}   
\end{theorem}

\paragraph{Sample complexity to find an $\veps$-suboptial policy with $\musik+\psdp$} From \cref{thm:PSDPthm}, to find an $\veps$-suboptimal policy, $\psdp$ requires an $(\alpha,\eps)$-approximate cover for all layers, where $\eps \ldef \veps/(2 SH^2)$. Focusing only on dependence on the accuracy parameter $\veps$, it follows from the results in \pref{sec:main_theorem} that $\musik$ can generate an $(1/4, \eps)$-approximate cover using $\wtilde O(1/\veps^2)$ trajectories (see \eqref{eq:trajectories}). Thus, the total number of trajectories required to find an $\veps$-suboptimal policy in reward-based RL using $\musik+\psdp$ scales with $\wtilde O(1/\veps^2)$. To the best of our knowledge, this is the first computationally efficient approach that gives $\bigoht(1/\veps^2)$ sample complexity for reward-based reinforcement learning in BMDPs (without reachability).

\section{Details for Experiments}
\label{app:experiments}
In this section, we give the details of the experiments. We provide the full code in the supplementary material.

\paragraph{Environment} We adopt the \comblock environment from
\citet{misra2020kinematic,zhang2022efficient}, which is parameterized
by the horizon $H$ and number of actions $A=10$. At each layer $h\in[H]$, there are $N=3$ states $s_{h,1}, s_{h,2}, s_{h,3}\in \cS_h$, where $s_{h,1},s_{h,2}$
are ``good'' states and $s_{h,3}$ is a ``bad'' terminal state. For each layer $h\in[H]$, there exists a pair
of ``good'' actions $u_{h,1}, u_{h,2} \in \cA$ such that taking action
$u_{h,j}$ in state $s_{h,j}$ (for $j\in\{1,2\}$) leads to one of the
good states $\{s_{h+1,1}, s_{h+1,2}\}$ at next layer with equal
probability.  All actions $a_h\not\in \{u_{h,1}, u_{h,2}\}$ lead to
the bad state $s_{h+1,3}$ deterministically. The sequences of good actions $u_{1:H,1}$ and $u_{1:H,2}$ are sampled uniformly at random from the set of actions $\cA$ when the environment is initialized and are unknown to the learner. 

For $h=H$, the agent receives a reward of $1$ if action $u_{H,j}$ is
taken in state $s_{H,j}$ (for $j\in\{1,2\}$), and receives reward of 0
otherwise. For $h<H$, the agent receives an anti-shaped reward of
$0.1$ for choosing any action $a_h \neq u_{h,j}$ in state $s_{h,j}$, for
$j\in\{1,2\}$, and receives a reward of 0 otherwise (in particular,
the agent never receives a reward in the bad state $s_{h,3}$). This anti-shaped
reward encourages the agent to take actions that lead to the bad state
$s_{h+1,3}$, from which it is not possible to reach the good states
$\{s_{H,1}, s_{H,2}\}$ at layer $H$ and achieve the optimal reward of
$1$.

The agent does not observe the states $\{\s_h\}$ directly, and instead
receives observations $\{\x_h\}$. For each $h$, the observation $\x_h$
is a $d$-dimensional vector, where $d\ldef{}2^{\ceil{\log_2
		(H+N+1)}}$, obtained by concatenating the one-hot vector of the latent state $\s_h$ and the one hot vector of the layer index $h$, followed by adding noise sampled from $\cN(0,0.1)$ in one dimension, padding with zeros if necessary, and multiplying with a Hadamard matrix.   
Strictly speaking, the CombLock
environment is more challenging than a Block MDP, since two latent states can emit the
same observation due to the addition of the Gaussian noise in the observation process.

Since the good actions $\{u_{h,1}, u_{h,2}\}_{h\in[H]}$ are not known
to the agent, deliberate exploration is required to learn a policy
that that maximizes the reward function (note it is only possible
achieve reward $1$ if the agent selects a good action for all
$h\in\brk{H}$). For example, when the horizon is set to $H=100$, the probability
of finding the optimal policy through naive uniform exploration is $10^{-100}$. In
addition, representation learning is required to recover the latent
state $\s_h$ from the observation $\x_h$ at each layer, with the best
decoder depending on the layer $h$. 

\paragraph{Implementation of \musik} We use $\musik$ to learn a policy
cover that we then use in $\psdp$ to find a near-optimal policy in the
CombLock environment. In this environment, the optimal policy cover can be
learned by composing optimal policy covers at each layer (though this is
not true in general, many problems share this property). We follow an
approach taken with $\homer$ in \cite{misra2020kinematic}, and take
advantage of this composability property to implement a more sample-efficient
version of $\musik$, where during the call to the $\ikdp$ subroutine
at layer $h$, we only learn $\fhat\ind{h-1},\phihat\ind{h-1}$
(i.e.~the $\ikdp$ for-loop stops at $t=h-1$); this is exactly what was done in \cite{misra2020kinematic}. This version of $\musik$, which we name $\musik.\texttt{comp}$, is displayed in \cref{alg:IKDPcomp} (we write the full algorithm without a reference to an external ($\ikdp$) subroutine).

We use $\Phi \coloneqq
\{\phi_W\colon x \mapsto\argmax_{i\in[N]} W x  \mid W\in \reals^{N\times d} \}$ for the
decoder class, where we recall that $N$ is the number of latent states per layer in the CombLock environment---this is exactly the same decoder class as the one used by \cite{misra2020kinematic} for $\homer$. Given the observation process in the CombLock environment, there exists a matrix $W_\star \in \reals^{N\times d}$ such that the true decoder $\phi_\star$ is given by $\phi_{W_\star}$. To learn $W_\star$, we use the set of differentiable maps $\Phi'\coloneqq \{ x \mapsto\mathrm{softmax} (W x)  \mid W\in \reals^{N\times d}\}$ during training (this is reflected in the objective in the next display). Further, we make a slight
(empirically-motivated) modification to the conditional density
estimation problem in Line~\ref{line:inversenew2} of $\ikdp$, where we
instead solve 
\begin{align}
	\fhat\ind{h-1} ,
	\hat{\psi}\ind{h-1}\gets \argmax_{f: \cX \times [N]  \rightarrow
		\Delta(\cA) , \psi \in \Phi'}\  \sum_{(a_h,x_{h-1},x_h)\in \cD\ind{h-1}} \log \left(\sum_{i=1}^N f(a_{h-1} \mid x_{h-1}, i) \cdot [\psi(x_h)]_i  \right).\label{eq:newobjective}
\end{align}
Compared to the original objective of $\ikdp$ in Line \ref{line:inversenew2} of \cref{alg:IKDP}, we no longer need to predict the index $i_{h-1}$ of the future roll-out policies (since the for-loop of $\ikdp$ now stops at $t=h-1$, there are no future roll-outs). Another difference is that we do not use a decoder at layer $h-1$; we use $f(a\mid x,j)$ instead of $f(a\mid \phi(x),j)$ (this helps with the training). For each $j\in[N]$, we instantiate $f(\cdot \mid \cdot, j)$ with a two-layer neural network with $\tanh$ activation, input dimension $d$ and output dimension $A$, where the output is pushed through a softmax so that $f(\cdot \mid x, j)$ is a distribution over actions for any $x\in \cX$. We use $\texttt{Adam}$ for the optimization problem in \eqref{eq:newobjective}. We specify the choices of hyperparameters in the sequel.

With $\hat{\psi}\ind{h-1}$ as in \eqref{eq:newobjective}, the learned decoder is given by $\phihat\ind{h-1}(x)\coloneqq \argmax_{i\in[N]}[\hat{\psi}\ind{h-1}x]_i$. Further, for $\fhat\ind{h-1}$ as in \eqref{eq:newobjective}, the $h$th layer policy cover $\Psi\ind{h} =\{\pihat\ind{j,h}\}_{ j \in[N]}$ constructed by $\musik.\texttt{comp}$ is essentially given by: 
\begin{align}
	\pihat\ind{j,h} =  \pihat \circ_{h-1} \ahat\ind{j,h-1} , \quad \text{where} \quad \pihat\in \argmax_{\pi \in \Psi\ind{h-1}}  \P^{\pi \circ_{h-1} \unifa}\left[\phihat\ind{h-1}(\x_h)=j\right], \label{eq:argmax}
\end{align}
and $\ahat\ind{j,h-1}(x) \coloneqq
\argmax_{a\in\cA}\fhat\ind{h-1}(a\mid x,j).$ That is, the policy
$\pihat\ind{j,h}$ is the composition of the best partial policy $\pihat$ among the partial policies in $\Psi\ind{h-1}$ (the policy
cover at the previous layer) and the best action at layer $h-1$ to
maximize to probability of reaching the `abstract state'
$j\in[N]$. Technically, computing $\pihat$ requires
estimating $\P^{\pi \circ_{h-1} \unifa}\left[\phihat\ind{h-1}(\x_h)=j\right]$,
for all $\pi \in \Psi\ind{h-1}$. For this, we reuse the dataset
$\cD\ind{h-1}$ from \eqref{eq:newobjective} and solve another conditional density estimation problem---see \eqref{eq:anothercond} in \cref{alg:IKDPcomp}\footnote{Technically, the solution of the conditional estimation problem in \eqref{eq:anothercond} does not yield an estimator of $\P^{\pi \circ_{h-1} \unifa}\left[\hat\phi\ind{h-1}(\x_h)=j\right]$ per se. But it gives us a proxy for a function whose argmax $\pihat$ in \eqref{eq:argmax}.}. 

\paragraph{$\psdp$ implementation} The only modification we make to the $\psdp$
algorithm is that we use $f(a\mid x)$ instead of $f(a \mid \phi(x))$
in the objective \eqref{eq:mistake} (i.e.~we do not use a decoder). We instantiate $f(\cdot \mid \cdot)$ with a two-layer neural network with input dimension $d$, hidden dimension of 400, and output dimension of 1. We use the $\tanh$ activation function at all layers.

\paragraph{Hyper-parameters} For each $j\in[N]$, we instantiate $f(\cdot \mid \cdot, j)$ in \eqref{eq:newobjective} with a two-layer neural network with $\tanh$ activation, input dimension $d$, hidden dimension of size $N_{\texttt{hidden}}$, and output dimension $A$, where the output is run through the softmax activation function (with temperature 1) so that $f(\cdot \mid x, j)$ is a distribution over actions for any $x\in \cX$. We also instantiate $g(\cdot \mid \cdot)$ in \eqref{eq:anothercond} with a two-layer neural network with $\tanh$ activation. input dimension $N$, hidden dimension of size 400, and output dimension $N$, where the output is pushed through a softmax (with temperature 1) so that $g(\cdot \mid j)$ is a distribution over $[N]$ for any $j\in [N]$. For the choice of hidden size $N_{\texttt{hidden}}$, we searched over the grid $\{100,200, 400\}$. The results reported in \cref{fig:horizon-plot} are for $N_{\texttt{hidden}}=400$.

We optimize the parameters of $(f,\theta)$ [resp.~$g$] in \eqref{eq:newobjective} [resp.~\eqref{eq:anothercond}] using $\texttt{Adam}$ with the default parameters in $\texttt{PyTorch}$. We use a batch size of $\min(n,N_{\texttt{batch}})$, where $n$ is as in \cref{alg:IKDPcomp}, and perform $N_{\texttt{update}}$ gradient updates. For the batch size $N_{\texttt{batch}}$ and number of updates $N_{\texttt{updates}}$, we searched over the girds $\{512,1024,2048,4096,8196\}$ and $\{64, 128, 256\}$, respectively. The results reported in \cref{fig:horizon-plot} are for $N_{\texttt{batch}}=8196$ and $N_{\texttt{updates}}=128$.

\paragraph{Baselines} 
As baselines, we use $\homer$ \citep{misra2020kinematic} and
$\briee$ \citep{zhang2022efficient}. Amongst provably efficient
algorithms, these methods are known to have the best empirical performance
\citep{zhang2022efficient} on the \comblock environment. The \homer{} algorithm has the same
structure as \musik: it first learns a policy cover, then
uses the cover within $\psdp$ to learn a near-optimal policy. The
$\briee$ algorithm does not explicitly learn a policy cover, but
rather interleaves exploration and exploitation using optimism.
We do not reproduce
$\briee$ and $\homer$, and instead report the results from
\citet{zhang2022efficient}.

\begin{algorithm}[htp]
	\caption{$\musik.\texttt{comp}$: Variant of $\musik$ for composable policy covers (version of $\musik$ used in the experiments).}
	\label{alg:IKDPcomp}
	\begin{algorithmic}[1]\onehalfspacing
		\Require
		~
		\begin{itemize}[leftmargin=*]
			\item Dimension of the observation space $d$.
			\item Number of latent states per layer $N$.
			\item Number of samples $n$.
		\end{itemize}
		\State Set $\Psi\ind{1} = \{\unifa,\dots, \unifa\}$ with $|\Psi\ind{1}|=N$.
		\For{$h=2,\ldots, H$} 
		\State $\cD\ind{h}\leftarrow \emptyset$.
		\State Let $\iota\ind{h-1} \colon \Psi\ind{h-1}\rightarrow [N]$ be any one-to-one mapping.
		\Statex[1]\algcommentbiglight{Collect data by rolling in with
			policy cover}
		\For{$n$ times}
		\State Sample $\pihat  \sim \unif(\Psi\ind{h-1})$.  
		\State Sample $(\x_{h-1}, \a_{h-1}, \x_{h})\sim \hat\pi \circ_{h-1} \unifa$. 
		\State $\cD\ind{h-1} \leftarrow \cD\ind{h-1}\cup \{(\iota\ind{h-1}(\hat\pi),\a_{h-1}, \x_{h-1}, \x_{h})\}$.
		\EndFor
		\Statex[1] \algcommentbiglight{Inverse kinematics}
		\State For $\Phi'\coloneqq \{ x \mapsto\mathrm{softmax} (W x)  \mid W\in \reals^{N\times d}\}$, solve
		\begin{equation}
			\fhat\ind{h-1},	\hat{\psi}\ind{h-1}\gets \argmax_{f\colon \cX \times[N] \rightarrow
				\Delta(\cA) , \psi\in \Phi'}
			\sum_{(-,a,x,x')\in \cD\ind{h-1}}  \ln   \left(\sum_{j\in[N]} f( a \mid
			x, j) \cdot [\psi(x)]_j\right).\label{eq:newobjective} 
		\end{equation}
		\Statex[1] \algcommentbiglight{Inverse Kinematics to learn associations between policies at subsequent layers}
		\State Solve
		\begin{equation}
			\hat g\ind{h-1}\gets \argmax_{g\colon [N] \rightarrow
				\Delta([N]) }
			\sum_{(i,-,-,x')\in \cD\ind{h-1}} \ln    g\left( i \, \left|\,
			\argmax_{i\in[N]}[\hat{\psi}\ind{h-1}(x')]_i \right. \right). \label{eq:anothercond} 
		\end{equation}
		\Statex[1] \algcommentbiglight{Update partial policy cover}
		\State For each $j\in\brk{S}$, define
		\begin{align}
			\ahat\ind{j,h-1}(x) &=
			\argmax_{a\in \cA} \fhat\ind{h-1}(a \mid
			x,j), \quad x\in \cX_t.\nn \\
			\iotahat\ind{j,h-1}(x)	& = 	\argmax_{i\in [N]}  \hat g \ind{h-1}(i\mid j).\nn 
		\end{align}
		\State  For each $j \in[S]$, define $\hat
		\pi^{(j,h)}\in \Pim^{1:h-1}$ via
		\begin{align}
			\hat  \pi^{(j,h)}(x_{\tau})\coloneqq \left\{
			\begin{array}{ll}
				\ahat\ind{j,h-1}(x_\tau),&\quad \tau=h-1,\\
				\pihat
				^{(\iotahat\ind{j,h-1},h-1)}(x_{\tau}),&\quad\tau\in[h-2],
			\end{array}
			\right. \quad x_\tau \in \cX_\tau.\nn 
		\end{align}
	\State Define $\Psi\ind{h} =\{ \pihat\ind{j,h} \colon j \in[N] \}$ \algcommentlight{Policy cover for layer $h$.}
	\EndFor
	\State \textbf{Return:} Policy covers
	$\Psi\ind{1}, \dots, \Psi\ind{H}$.
\end{algorithmic}
\end{algorithm}

 }

\icml{
\newpage
\part{Analysis}
\label{part:analysis}

  \section{Organization}
  \cref{part:analysis} of the appendix contains the proof of our main
  result, \cref{thm:policycover}, as well as other proofs. This
  section is organized as follows.
  \begin{itemize}
  \item First, in \cref{sec:overview}, we give an informal overview of
    the analysis of \cref{thm:policycover}, using the tools introduced
    in \cref{sec:key} as a starting point. In particular:
    \begin{itemize}
    \item \cref{sec:warmup} introduces and analyzes a simplified
      version of \musik intended for tabular reinforcement learning as
      a warm-up exercise.
    \item \cref{sec:block} builds on this development to showcase the
      main ideas behind the proof of \cref{thm:policycover}.
    \end{itemize}
  \item \cref{app:structural} provides proofs for the structural
    results introduced in \cref{sec:key}.
  \item \cref{sec:tabular} contains proofs for the tabular warm-up
    exercise in \cref{sec:overview}
  \item \cref{sec:BMDP} contains the proof of our main result,
    \cref{thm:policycover}. For background on the key ideas, we recommend reading the overview in
    \cref{sec:overview}.
  \item \cref{app:PSDPthmproof} contains proofs for the extensions to
    reward-based RL in \cref{sec:rewardsetting}.
  \end{itemize}

  }

       \icml{
       	\section{Overview of Analysis}
       	\label{sec:overview}
\icml{
  In this section, we give an overview of the analysis of our main result, \cref{thm:policycover}, with the full proof deferred to  \cref{sec:BMDP}. First, in \cref{sec:warmup} we show how to analyze a simplified version of \musik for the \emph{tabular} setting in which the state $\s_h$ is directly observed. Then, in \cref{sec:block}, we build on these developments to give a proof sketch for the full Block MDP setting.
  }

\arxiv{
  In this section, we give an overview of the analysis of our main result for \musik, \cref{thm:policycover}, with the full proof deferred to  \cref{sec:BMDP}. First, in \cref{sec:keytools} we introduce two analysis tools, the \emph{extended BMDP} and \emph{truncated policy class}, which play a key role in providing tight guarantees for \musik (and more broadly, non-optimistic algorithms) in the absence of minimum reachability. Then, as a warmup (\cref{sec:warmup}), we show how to analyze a simplified version of \musik for the \emph{tabular} setting in which the state $\s_h$ is directly observed (i.e., $\cX=\cS$ and $\bx_h=\bs_h$ almost surely). Finally, in \cref{sec:block}, we build on these developments to give a proof sketch for the full Block MDP setting.
  }

\arxiv{\subsection{Key Analysis Tools: Extended BMDP and Truncated Policy Class}
\label{sec:keytools}
Recall that \musik proceeds by inductively building a sequence of policy covers $\Psi\ind{1},\ldots,\Psi\ind{H}$. A key invariant maintained by the algorithm is that for each layer $h$, the previous covers $\Psi\ind{1},\ldots,\Psi\ind{h-1}$ provide good coverage for layers $1,\ldots,h-1$, and thus can be used to efficiently gather data to build the next cover $\Psi\ind{h}$. Prior approaches that build policy covers in this inductive fashion \citep{du2019provably,misra2020kinematic} require the assumption of minimum reachability (\pref{def:reachability}) to ensure that for each $h$, $\Psi\ind{h}$ \emph{uniformly} covers all possible states in $\cS_h$. In the absence of reachability, we inevitably must sacrifice certain hard-to-reach states, which necessitates a more refined analysis. In particular, we must show that the effects of ignoring hard-to-reach states at earlier layers do not compound as the algorithm proceeds forward.

To provide such an analysis, we make use of a tool we refer to as the \emph{extended BMDP} $\Mbar$. The extended BMDP $\Mbar$ augments $\cM$ by adding a set of $H$ \emph{terminal states} $\tfrak_{1:H}$ and one additional \emph{terminal action} $\afrak$ as follows:
	\begin{enumerate}
		\item The latent state space is $\wbar \cS \coloneqq \bigcup_{h=1}^H \wbar \cS_h$, where $\wb{\cS}_h\ldef\cS_h\cup\crl{\tfrak_h}$. 
		\item The action space is $\cAbar\ldef\cA \cup\crl{\afrak}$. Here, $\afrak$ is a ``terminal action'' that causes the latent state to deterministically transition to $\term_{h+1}$ from every state at layer $h\in[H-1]$.
		\item For $h\in[H-1]$, taking any action in $\wbar\cA$ at latent state $\tfrak_h$ transitions to $\tfrak_{h+1}$ deterministically.\footnote{The reason we introduce $H$ states $\tfrak_{1:H}$ instead of a single self-looping state $\tfrak$ is to keep the convention that the state space is layered.}
                \end{enumerate}
            The dynamics of $\cMbar$ (including the initial state distribution) are otherwise identical to $\cM$. We assume the state $\tfrak_h$ emits itself as an observation and we write $\wb{\cX}_h \coloneqq  \cX_h \cup\{\tfrak_h\}$, for all $h\in[H]$. We will use the convention that $\phi(\tfrak_h)=\tfrak_h$, for all $\phi\in\Phi$ and $h\in[H]$. For any policy $\pi \in \Pibar_{\nm} \coloneqq \left\{\pi \colon \bigcup_{h=1}^H (\wbar\cX_{1}\times \dots \times \wbar\cX_h) \rightarrow \Abar\right\}$, we define 
	\begin{align*}
		\Pbar^{\pi}\coloneqq \P^{\Mbar,\pi}, \quad  \Ebar^{\pi}\coloneqq \E^{\Mbar,\pi},\quad  \text{and} \quad \dbar^{\pi}(s)\coloneqq \P^{\Mbar,\pi}[\s_h = s], \quad \text{for all } s\in \cS_h \text{ and } h \in[H].
	\end{align*}
	\paragraph{Truncated policy class} On its own, the extended BMDP is not immediately useful. The main idea behind our analysis is to combine it with a restricted sub-class of policies we refer to as the \emph{truncated policy class.} Define $\Pibarm  \coloneqq \{ \pi\colon \bigcup_{h=1}^H \wbar\cX_h \rightarrow \Abar\}$. For $\epsilon \in(0,1)$, we define a sequence of policy classes $\Pibar_{0, \epsilon},\dots,\Pibar_{H, \epsilon}$, inductively, starting from $\Pibar_{0, \epsilon}=\Pibarm$:
	\begin{align}
		\pi \in \Pibar_{t, \epsilon} & \iff \nn \\  &  \hspace{-.1cm} \exists \pi'\in \Pibar_{t-1, \epsilon}: \forall h \in[H], \forall s\in \wbar\cS_h, \forall x\in \phi^{-1}_\star(s), \ 	\pi(x) = \left\{\begin{array}{ll} \fraka, & \text{if } h=t \text{ and } \max_{\tilde \pi\in \Pibar_{t-1, \epsilon}} \dbar^{\tilde \pi}(s)< \epsilon, \\   \pi'(x), & \text{otherwise}.    \end{array}    \right.  \label{eq:define0}
	\end{align}
	Restated informally, the class $\Pibar_{t,\eps}$ is identical to $\Pibar_{t-1,\eps}$, except that at layer $t$, all policies in the class $\Pibar_{t,\eps}$ take the terminal action $\afrak$ in latent states $s$ for which $\max_{\tilde \pi\in \Pibar_{t-1, \epsilon}} \dbar^{\tilde \pi}(s)< \epsilon$.
	
	We define the \emph{truncated policy class} as $\Pibar_{\epsilon} \coloneqq \Pibar_{H,\epsilon}$. The truncated policy class satisfies two fundamental technical properties. First, by construction, all policies in the class take the terminal action $\afrak$ when they encounter states that are not $\eps$-reachable by $\Pibar_\eps$. 	Second, in spite of the fact that policies in $\Pibar_{\epsilon}$ always take the terminal action on states with low visitation probability, they can still achieve near-optimal visitation probability for all states in $\cMbar$ (up to additive error). The following lemmas formalize these properties.
	\begin{lemma}[Behavior on low-reachability states]
		\label{lem:bla}
		Let $\epsilon\in(0,1)$ be given, and define
		\begin{align}
		  \label{eq:truncated_reachable}
			\cS_{h,\epsilon} \ldef  \crl*{  s\in \cS_h  \colon	\max_{\pi \in \Pibar_{\epsilon}}\dbar^{\pi}(s) \geq \epsilon}.
		\end{align}
		Then, for all $h \in[H]$ if $s\in \cS_{h }\setminus\cS_{h,\epsilon}$, then for all $\pi\in\Pibar_\eps$, $\pi(x)=\fraka$ for all $x\in \phi^{-1}_\star(s)$.
	\end{lemma}
	\begin{lemma}[Approximation for truncated policies]
		\label{lem:critical}
		Let $\epsilon\in(0,1)$ be given. For all $h \in[H]$ and $s\in \cS_h$, 
		\begin{align}
		         \label{eq:truncated_approximation}
			\max_{\pi \in \Pibarm} \dbar^{\pi}(s) \leq \max_{\pi \in \Pibar_{\epsilon}} \dbar^{\pi}(s) +S \epsilon.
		\end{align}
	\end{lemma}
	The proofs for these results (and other results in this subsection) are elementary, and are given in \pref{app:structural}.       
	Building on these properties, our proof of \pref{thm:policycover} makes use of two key ideas:
	\begin{enumerate}
		\item Even though the extended BMDP $\cMbar$ does not necessarily enjoy minimum reachability (\pref{def:reachability}), if we restrict ourselves to competing against policies in $\Pibar_\eps$, \pref{lem:bla} will allow us to ``emulate'' certain properties enjoyed by $\eps$-reachable BMDPs. This in turn will imply that if are satisfied with learning a policy cover with good coverage ``relative'' to $\Pibar_\eps$, \pref{alg:GenIk} will succeed.
		\item By \pref{lem:critical}, we lose little by restricting our attention to the class $\Pibar_\eps$. This will allow us to transfer any guarantees we achieve with respect to the extended BMDP $\cMbar$ and truncated policy class $\Pibar_\eps$ back to the original BMDP $\cM$ and unrestricted policy class $\Pim$.
	\end{enumerate}
	We make the first point precise in the sections that follow (\cref{sec:warmup,sec:block}). Before proceeding, we formalize the second point via another technical result, \pref{lem:transfer}. To do so, we introduce the notion of a \emph{relative policy cover}.
	\begin{definition}[Relative policy cover]
		\label{def:polcover}
		Let $\alpha, \veps\in[0,1)$ be given. Consider a BMDP $\cM'$, and let $\Pi$ and $\Psi$ be two sets of policies. We say that $\Psi$ is an $(\alpha,\veps)$-policy cover relative to $\Pi$ in $\cM'$ for layer $h$ if
		\begin{align}
			\max_{\pi \in \Psi} d^{\cM',\pi}(s) \geq  \alpha \cdot \max_{\pi \in \Pi} d^{\cM',\pi}(s) \quad \text{for all } \  s\in \cS_h \ \text{ such that }\  \max_{\pi \in \Pi} d^{\cM',\pi}(s)  \geq \veps. \nn
		\end{align}
	\end{definition} 
	\begin{lemma}[Policy cover transfer]
	\label{lem:transfer}
		Let $\veps\in(0,1)$ be given, and define $\eps\ldef\veps/2S$.
		Let $\Psi$ be a set of policies for $\cMbar$ that never take the terminal action $\afrak$. If $\Psi$ is a $(1/2,\epsilon)$-policy cover relative to $\Pibar_\epsilon$ in $\Mbar$ for all layers, then $\Psi$ is a $(1/4,\veps)$-policy cover relative to $\Pim$ in the $\cM$ for all layers.
	\end{lemma}
	\pref{lem:transfer} implies that for any $\veps$, letting $\eps\ldef{}\veps/2S$, if we can construct a set $\Psi$ that acts as a
	$(1/2,\eps)$-policy cover relative to
	$\Pibar_{\epsilon}$ in $\Mbar$, then $\Psi$ will also be a $(1/4,\veps)$-policy cover relative to $\Pim$ in the original BMDP $\cM$, as desired. This allows us to restrict our attention to the former goal going forward. }

	\subsection{Warm-Up: Multi-Step Inverse Kinematics for Tabular MDPs}
	\label{sec:warmup}
        \begin{algorithm}[htp]
          \caption{$\musiktab$: Multi-Step Inverse Kinematics (tabular variant)}
          \label{alg:musik_tabular}
          \begin{algorithmic}[1]\onehalfspacing
            \Require Number of samples $n$.
            \State Set $\Psi\ind{1}= \emptyset$.
            \For{$h=2\ldots, H$} 
            \State Let
            $\Psi\ind{h}=\ikdptab(\Psi\ind{1},\dots,\Psi\ind{h-1},
            n)$.\quad\algcommentlight{\pref{alg:IKDP-tab}.}
            \label{line:ikdp1}
            \EndFor
            \State \textbf{Return:} Policy covers $\Psi\ind{1},\dots,\Psi\ind{H}$. 
          \end{algorithmic}
	\end{algorithm}

In this section, we use the extended BMDP, truncated policy class, and relevant structural results introduced in
prequel to analyze a simplified version of \musik for the \emph{tabular} setting in which the state $\s_h$ is directly observed (a special case of the BMDP in which $\cX=\cS$ and $\bx_h=\bs_h$ almost surely). The tabular setting preserves the most important challenges in removing reachability, and will serve as a useful warm-up exercise for the full BMDP setting. Our analysis will also give a taste for how the multi-step inverse kinematics objective in \ikdp (\pref{line:inversenew2}) allows one to approximately implement dynamic programming.

\paragraph{$\musik$ and $\ikdp$ for tabular MDPs} \pref{alg:musik_tabular} (\musiktab) and \pref{alg:IKDP-tab} (\ikdptab) are simplified variants of \musik and \ikdp tailored to the tabular setting. \musiktab is identical to \musik, except that the subroutine \ikdp is replaced by \ikdptab. $\ikdptab$ has the same structure as $\ikdp$, but does not require access to a decoder class $\Phi$, since the states are observed directly. The algorithm takes advantage of a slightly simplified multi-step inverse kinematics objective (\cref{line:inversenew20} of \cref{alg:IKDP-tab}) which involves directly predicting actions based on the latent states. Recall that for iteration $t\in[h-1]$, the full version of $\ikdp$ uses observations to predict \emph{pairs} $(\a_t,\bi_t)$, where $\a_t$ is the action played at layer $t$ and $\bi_t\in\brk{S}$ is the (random) index of the partial policy executed after layer $t$. \ikdptab does
not require randomizing over the index $\bi_t$, and instead solves a separate regression problem for each state $i\in\brk{S}$ (representing the state being targeted at layer $h$), predicting only the action $\a_t$; we will highlight the need for the randomization over indices $\bi_t$ when we return to the BMDP setting in the sequel (\pref{sec:block}).

The following theorem, an analogue of \cref{thm:policycover} for tabular MDPs, provides the main guarantee for $\musiktab$.
\begin{theorem}[Main theorem for \musiktab]
	\label{thm:trend}
	Let $\veps,\delta \in(0,1)$ be given, and let  $n\geq 1$ be chosen such that
    \begin{align}
	n \geq \frac{c A^2 S^6 H^2 \left( S^2 A \ln n + \ln (S H^2/\delta)\right)}{\veps^2},\label{eq:condn} 
	\end{align} for some absolute constant $c>0$ independent of all problem parameters. Then, with probability at least $1-\delta$, the collections $\Psi\ind{1},\dots,\Psi\ind{H}$ produced by \musiktab are $(1/4,\veps)$-policy covers for layers $1$ through $H$.
\end{theorem}
\paragraph{Analysis by induction}
To prove \cref{thm:trend}, we proceed by induction over the layers $h=1,\dots,H$. Leveraging the extended MDP and truncated policy class, we will show that for each layer $h\in[H]$, if the collections $\Psi\ind{1},\dots, \Psi\ind{h-1}$ produced by $\ikdptab$ have the property that
\begin{align}
  \Psi\ind{1},\dots, \Psi\ind{h-1} \text{ are $(1/2,\epsilon)$-policy covers relative to $\Pibar_{\epsilon}$ in $\Mbar$ for layers $1$ through $h-1$,} \label{eq:inductouter}
\end{align}
then with high probability, the collection $\Psi_{h}$ produced by $\ikdptab(\Psi_{1:h-1},n)$ will be a  $(1/2,\epsilon)$-policy cover relative to $\Pibar_{\epsilon}$ in $\Mbar$ for layer $h$. Formally, we will prove the following result.
\begin{theorem}[Main theorem for $\ikdptab$]
	\label{thm:geniklemma0}
	Let $\epsilon,\delta \in(0,1)$ and $h\in[H]$ be given and define $\veps_{\stat}(n, \delta')\coloneqq  \sqrt{ n^{-1}\prn{S^2 A \ln n+ \ln (1/\delta')}}$. Assume that:
        \begin{enumerate}
        \item \emikdptab is invoked with $\Psi\ind{1},\ldots, \Psi\ind{h-1}$ satisfying \cref{eq:inductouter}.
        \item The policies in $\Psi\ind{1},\ldots, \Psi\ind{h-1}$ never take the terminal action $\afrak$.
        \item The parameter $n$ is chosen such that $8A  S^2 HC \cdot\veps_{\stat}(n,\frac{\delta}{SH^2})\leq \epsilon$ for some absolute constant $C>0$ independent of all problem parameters.
      \end{enumerate}
      Then, with probability at least $1-\frac{\delta}{H}$, the collection $\Psi\ind{h}$ produced by $\emikdptab(\Psi\ind{1},\dots,\Psi\ind{h-1},n)$ is an $(1/2,\epsilon)$-policy cover relative to $\Pibar_{\epsilon}$ in $\Mbar$ for layer $h$. In addition, $\Psi\ind{h}\subseteq \Pim^{1:h-1}$.
      \end{theorem}
      With this result in hand, the proof of \pref{thm:trend} follows swiftly.
      \begin{proof}[{Proof of \cref{thm:trend}}]
  Let $\delta,\veps\in(0,1)$ be given and let $\epsilon \coloneqq \veps/(2S)$. Let $\veps_\stat(\cdot, \cdot)$ and $C$ be as in \cref{thm:geniklemma0}; here $C$ is an absolute constant independent of all problem parameters. Let $\cE_h$ denote the event that \ikdptab succeeds as in \cref{thm:geniklemma0} for layer $h\in[H]$ with parameters $\delta$ and $\eps$, and define $\cE\coloneqq \bigcap_{h\in[H]}\cE_h$. Observe that by \pref{thm:geniklemma0} and the union bound, we have $\P[\cE]\geq 1 -\delta$. For $n$ large enough such that $8A  S^2 HC \cdot\veps_{\stat}(n,\frac{\delta}{SH^2})\leq \epsilon$ (which is implied by the condition on $n$ in the theorem's statement for $c=2^5 C$), \cref{thm:geniklemma0} implies that under $\cE$, the output $\Psi\ind{1},\dots,\Psi\ind{H}$ of $\musik$ are $(1/2,\epsilon)$-policy covers relative to $\Pibar_{\epsilon}$ in $\wbar \cM$ for layers 1 to $H$, respectively. We conclude by appealing to \cref{lem:transfer}, which now implies that $\Psi\ind{1},\ldots,\Psi\ind{H}$ are $(1/4,\veps)$-policy covers relative to $\Pim$ in $\cM$.

\begin{algorithm}[t]
  \caption{$\ikdptab:$ Inverse Kinematics for Dynamic Programming (tabular variant)}
  \label{alg:IKDP-tab}
  \begin{algorithmic}[1]\onehalfspacing
    \Require
    ~
    \begin{itemize}[leftmargin=*]
    \item Approximate covers
      $\Psi\ind{1},\dots,\Psi\ind{h-1}$ for layers $1$
      to $h-1$, where $\Psi\ind{t} \subseteq\Pim^{1:t-1}$.
    \item Number of samples $n$.
    \end{itemize}
    \For{$t=h-1,\ldots, 1$} \label{line:mainiter0}
    \State $\cD\ind{t}\leftarrow \emptyset$.
    \Statex[1]\algcommentbiglight{Collect data by rolling in with
      policy cover and rolling out with partial policy}
    \For{$i\in[S]$} \label{line:for2}
    \For{$n$ times}
    \State Sample $(\s_t, \a_t, \s_{h})\sim \unif(\Psi\ind{t}) \circ_t \unifa\circ_{t+1} \pihat ^{(i,t+1)}$.
    \State $\cD\ind{t} \leftarrow \cD\ind{t}\cup \{(\a_t, \s_t, \s_{h})\}$.
    \EndFor
    \Statex[2] \algcommentbiglight{Inverse kinematics}
    \State \label{line:inversenew20}
    \begin{equation}
      \fhat\ind{i,t} \in \mathrm{argmax}_{f\colon  [S]^2 \rightarrow \Delta_{A}} \sum_{(a,s,s')\in \cD\ind{t}} \ln  f(a \mid s, s').
      \label{eq:ik_tabular}
    \end{equation}
    \Statex[2] \algcommentbiglight{Update partial policy cover}
    \State Define $\ahat\ind{i,t}(s)\in \argmax_{a\in \cA} \fhat\ind{i,t}(a \mid s,i)$.
    \State  Define $\hat\pi^{(i,t)}\in \Pim^{t:h-1}$  via\label{line:policy_update_tabular}
    \begin{align}
      \hat  \pi^{(i,t)}(s)\coloneqq \left\{
      \begin{array}{ll}
        \ahat\ind{i,t}(s),&\quad s\in\cS_t,\\
        \pihat
        ^{(i,\tau)}(s),&\quad  s\in \cS_\tau,\;\;\tau\in\range{t+1}{h-1}.
      \end{array}
                         \right.
                         \label{eq:policy_update_tabular}
    \end{align}
    \EndFor
    \EndFor
    
    \State \textbf{Return:} Policy cover
    $\Psi\ind{h}=\{\hat\pi^{(i,1)}\colon i
    \in[S]\}\subseteq \Pim^{1:h-1}$ for layer $h$.
  \end{algorithmic}
\end{algorithm}

We now compute the total number of trajectories used by the algorithm.
Recall that when invoked with parameter $n\in\bbN$, \musiktab invokes $H-1$ instances of $\ikdptab$, each with parameter $n$. Each instance of $\ikdptab$ uses $n$ trajectories for each layer $t\in[h-1]$ and $i\in[S]$ (see \cref{line:mainiter0,line:for2} of \cref{alg:IKDP-tab}), so the total number of trajectories used by $\musiktab$ is at most
		\begin{align}
                  \bigoht(1)\cdot\frac{A^2 S^7 H^4 \left( S^2 A + \ln (|\Phi|S H^2/\delta)\right)}{\veps^2}.\nn 
		\end{align}
\end{proof}

\subsubsection{Proof Sketch for \creftitle{thm:geniklemma0}}
We now sketch the proof of \pref{thm:geniklemma0}. The most important feature of the proof is that the guarantee on which we induct, \pref{eq:inductouter}, is stated with respect to the extended MDP and truncated policy class. We work in the extended MDP throughout the proof, and only pass back to the original MDP $\cM$ and full policy class $\Pim$ in the proof of \pref{thm:trend} (see above) \emph{once the induction is completed}.

Let $h\in[H]$ and $\epsilon>0$ be fixed, and assume that \cref{eq:inductouter} holds (that is, $\Psi\ind{1},\dots, \Psi\ind{h-1}$ are $(1/2,\epsilon)$-policy covers relative to $\Pibar_{\epsilon}$ in $\Mbar$ for layers $1$ through $h-1$). We will prove that the collection $\Psi\ind{h}$ produced by $\emikdptab(\Psi\ind{1},\dots,\Psi\ind{h-1},n)$ is an $(1/2,\epsilon)$-policy cover relative to $\Pibar_{\epsilon}$ in $\Mbar$ for layer $h$. We first argue that proving \cref{thm:geniklemma0} reduces to showing the following lemma. To state the result, recall that $\cS_{h,\epsilon} \ldef  \crl*{  s\in \cS_h  \colon	\max_{\pi \in \Pibar_{\epsilon}}\dbar^{\pi}(s) \geq \epsilon}$ is the set of states that are $\eps$-reachable by $\Pibar_\eps$ in $\cMbar$.
\begin{lemma}
  \label{lem:prop}
  Assuming points 1.~and 2.~in \cref{thm:geniklemma0} hold, and if $n$ is chosen large enough such that $8A  S^2 HC \cdot\veps_{\stat}(n,\frac{\delta}{SH^2})\leq \epsilon$ for some absolute constant $C>0$ independent of all problem parameters, then for all $t\in[h-1]$, with probability at least $1-\delta/H^2$, the learned partial policies $\crl*{\pihat\ind{i,t}}_{i\in[S]}$ in $\ikdptab$ have the property that for all $i \in \cS_{h,\epsilon}$,
  \begin{align}
    \dbar^{\pi_\star\ind{i}\circ_{t+1} \pihat\ind{i,t+1}}(i)  - \dbar^{\pi_\star\ind{i}\circ_{t}\pihat\ind{i,t}}(i)\leq \frac{\epsilon}{2H}, \quad \text{where}\quad \pi_\star\ind{i} \in \argmax_{\pi \in \Pibar_\epsilon} \dbar^{\pi}(i).\label{eq:prop}
  \end{align}
\end{lemma}
For each $i\in\cS_{h,\eps}$, $\pi_\star\ind{i}$ in \pref{eq:prop} denotes the policy in the truncated class $\Pibar_{\epsilon}$ that maximizes the probability of visiting $i$ at layer $h$. Informally, \cref{eq:prop} states that if we execute $\pi_\star\ind{i}$ up to layer $t-1$ (inclusive), then switch to the learned partial policy $\pihat\ind{i,t}$ for the remaining steps (i.e.~execute $\pi_\star\ind{i}\circ_{t} \pihat\ind{i,t}$), then the probability of reaching state $i$ in layer $h$ is close to what is achieved by running $\pi_\star\ind{i}\circ_{t+1} \pihat\ind{i,t}$. In other words, $\pihat\ind{i,t}$ is near-optimal in an average-case sense. We now show that \cref{thm:geniklemma0} follows from \cref{lem:prop}.
\begin{proof}[Proof of \cref{thm:geniklemma0}]
	For $t \in [h-1]$, let $\cE_t$ denote the success event of \cref{lem:prop}. Let us condition on the event $\cE \coloneqq \bigcap_{t \in [h-1]} \cE_t$. Fix $\i \in \cS_{h,\epsilon}$. Summing the \lhs of \pref{eq:prop} over $t=1,\dots,h-1$ for $i=\i$ and telescoping, we have that
	\begin{align}
		\dbar^{\pihat\ind{\i,1}}(\i)  \geq 	\max_{\pi \in \Pibar_\epsilon} \dbar^{\pi}(\i) - \frac{\epsilon}{2} \geq  \frac{1}{2}	\max_{\pi \in \Pibar_\epsilon} \dbar^{\pi}(\i),\label{eq:newest}
	\end{align}
where the last inequality follows by the fact that $\max_{\pi \in \Pibar_\epsilon} \dbar^{\pi}(\i)\geq \epsilon$ (since $\i \in \cS_{h,\epsilon}$).
Since this conclusion holds uniformly for all $\i\in\cS_{h,\epsilon}$, we have that under the event $\cE$, the output $\Psi\ind{h} \coloneqq \{ \pihat\ind{i,1}\colon i \in[S]\}$ of \cref{alg:IKDP-tab} is a $(1/2,\epsilon)$-policy cover relative to $\Pibar_\epsilon$ for layer $h$. Finally, by a union bound, we have $\P[\cE] \geq 1 - \P[\cE^c] \geq 1 -\sum_{t\in[h-1]}\sum_{i\in[S]} \P[(\cE\ind{i}_t)^c]\geq 1 - \delta/H$, which completes the proof. \end{proof}

  \begin{remark}
    \label{rem:performance}
It is also possible to derive \pref{eq:newest} from \cref{lem:prop} using the performance difference lemma \citep{kakade2003sample} with a specific state-action value function; this perspective will be useful when we generalize our analysis from the tabular to the BMDP setting. To see how the performance difference lemma can be applied to obtain \pref{eq:newest}, fix $i\in\brk{S}$ and consider the \emph{state-action value function} ($Q$-function) at layer $t$ with respect to the partial policy $\pihat\ind{i,t}\in \Pim^{t:h-1}$ for the MDP $\wbar \cM$ with rewards $r\ind{i}_\tau(s)=\mathbf{1}\{s=i\} \cdot \mathbf{1}\{\tau = h\}$, for $\tau\in[h]$; that is, \loose
  \begin{align}
  	Q_t^{\hat \pi\ind{i,t}}(s,a;i) = r\ind{i}_{t}(s) + \Ebar^{\hat \pi\ind{i,t}}\left[\left.\sum_{\tau=t+1}^h r\ind{i}_\tau(\s_\tau) \ \right\mid\  \s_t = s,\a_t = a\right].\label{eq:Qfunction} 
  \end{align}
Thanks to the choice of reward functions, we have 
\begin{align}
	Q_t^{\hat \pi\ind{i,t}}(s,a;i)  &= \Pbar^{ \pihat\ind{i,t+1}}[\s_h=i\mid \s_t =s, \a_t = a], \label{eq:FK} \\  \shortintertext{and thus}    \dbar^{\pihat\ind{i,1}}(i) - \dbar^{\pi_\star\ind{i}}(i) &= \E\left[	Q_1^{\hat \pi\ind{i,t}}(\s_1,\hat \pi\ind{i,t}(\s_1);i) - Q_1^{\pi_\star\ind{i}}(\s_1,\pi_\star\ind{i}(\s_1);i)  \right]. \label{eq:lhs}
\end{align}
Thus, by the performance difference lemma, the \rhs of \eqref{eq:lhs} can be bounded by 
\begin{align}
	\sum_{t=1}^{h-1}\Ebar^{\pi_\star\ind{i}} \left[Q^{\pihat\ind{i,t}}_t(\s_t, \pi_\star\ind{i}(\s_t); i)  -Q^{\pihat\ind{i,t}}_t(\s_t, \pihat ^{(i,t)}(\s_t); i)\right]. \label{eq:cont}
	\end{align}
        Thanks to \pref{eq:FK}, the quantity in \eqref{eq:cont} is simply $\sum_{t=1}^{h-1} \big(\dbar^{\pi_\star\ind{i}\circ_{t+1} \pihat\ind{i,t+1}}(i)  - \dbar^{\pi_\star\ind{i}\circ_{t}\pihat\ind{i,t}}(i)\big)$, which can directly be bounded using \cref{lem:prop} to arrive at the conclusion in \pref{eq:newest}.
  \end{remark}

     It remains to prove \cref{lem:prop}.
  To prove the result, we first use the multi-step inverse kinematics objective to establish a certain ``local'' optimality guarantee. We combine this with the assumption that $\Psi\ind{1},\ldots,\Psi\ind{h-1}$ are policy covers, along with certain structural properties of the extended MDP $\cMbar$, to conclude the result.

\paragraph{A local optimality guarantee from multi-step inverse kinematics}
Fix $1\leq t<h$ and a state $i\in\cS_{h,\epsilon}$, and let $\{\pihat\ind{i,t+1}\colon i\in[S]\}$ be the partial policies constructed by $\ikdptab$ at layer $t+1$. As the first step toward constructing the policy $\pihat \ind{i,t}$, $\ikdptab$ computes an estimator $\fhat\ind{i,t}\colon \brk{S}^2\to\Delta(\cA)$ by solving the multi-step inverse kinematics objective in \pref{line:inversenew20}. This entails predicting the probability of the action $\a_t$ conditioned on the states $\s_t$ and $\s_h$, under the process $(\s_t,\a_t,\s_h) \sim \P^{\unif(\Psi\ind{t})\circ_t \unifa \circ_{t+1}\pihat\ind{i,t+1}}$.\footnote{{Note that $\unifa$ denotes the policy that samples $\a_t$ uniformly from $\cA$, not $\cAbar$.}} The following result gives a generalization guarantee for $\fhat\ind{i,t}$ under this process.
	\begin{lemma}[Conditional density estimation guarantee]
          \label{lem:regtab}
          Fix $t \in \brk{h-1}$. Let $n\geq 1$ and $\delta\in(0,1)$ be given, and define $\veps_\stat(n,\delta)\coloneqq n^{-1/2}\cdot  \sqrt{S^2 A\ln n + \ln (1/\delta)}$. Assume that the policies in $\Psi\ind{t}$ never take the terminal action $\afrak$. Then, there exists an absolute constant $C>0$ (independent of $t,h$, and other problem parameters) such that for all $i\in[S]$ the solution $\fhat\ind{i,t}$ to the conditional density estimation problem in \cref{line:inversenew20} of \cref{alg:IKDP-tab} has that with probability at least $1-\delta$,
		\begin{align}\label{eq:mle_tabular}
			\Ebar^{\unif(\Psi\ind{t})\circ_t \unifa \circ_{t+1} \pihat\ind{i,t+1}} \left[ \sum_{a\in\cA}\left(  \fhat\ind{i,t}(a\mid \s_t, \s_h)	-  P\ind{i,t}_{\bayes}(a\mid \s_t ,\s_h) \right)^2 \right]\leq   C^2 \cdot \veps_\stat^2(n,\delta) ,
		\end{align}
                where \begin{equation}\label{eq:bayes_tabular}
                  P\ind{i,t}_{\bayes}(a\mid s ,s')	\coloneqq
     \frac{\Pbar^{\pihat\ind{i,t+1}}[\s_h=s'\mid
            \s_t=s,\a_t=a]}{Z\ind{i,t}(s,s')}, \ \ \text{for}\ \ Z\ind{i,t}(s,s')\coloneqq  \sum_{a'\in\cA}\Pbar^{
            \pihat\ind{i,t+1}}[\s_h=s'\mid \s_t=s,\a_t=a'].
          \end{equation}
	\end{lemma}
        \pref{lem:regtab} is a consequence of a standard generalization bound for conditional density estimation. The \emph{Bayes-optimal regression function} $\Pbayes\ind{i,t}$ represents the true conditional probability for $\a_t$ under the process $(\s_t,\a_t,\s_h) \sim \P^{\unif(\Psi\ind{t})\circ_t \unifa \circ_{t+1}\pihat\ind{i,t+1}}$. This quantity is useful as a proxy for another quantity we refer to as \emph{forward kinematics}:
\begin{align}
\fk\ind{t}(i\mid s,a ) \coloneqq   \Pbar^{ \pihat\ind{i,t+1}}[\s_h=i\mid \s_t =s, \a_t = a]. \label{eq:fk}
\end{align}
The utility of forward kinematics is somewhat more immediate: It represents the probability that we reach state $i$ at layer $h$ if we start from $\s_t=s$, take action $\a_t=a$, and then roll out with $\pihat\ind{i,t+1}$; equivalently $\fk\ind{t}(i\mid{}s,a)$ is the Q-function for the reward function $\indic\crl{\s_h=i}$---see \pref{eq:FK}. Hence, by the principle of dynamic programming, it is natural to choose 
\begin{align}
\pihat\ind{i,t}(s)=\argmax_{a\in\cA}\fk\ind{t}(i\mid s,a ).\label{eq:fk_opt}
\end{align}
\ikdptab does not directly compute the forward kinematics, and hence cannot directly define $\pihat\ind{i,t}$ based on \pref{eq:fk_opt}. Instead, we compute
\begin{align}
\pihat\ind{i,t}(s)=\argmax_{a\in\cA}\Pbayes\ind{i,t}(a\mid s, i).\label{eq:bayes_pihat}
\end{align}
To see that this is equivalent, observe that $P\ind{i,t}_{\bayes}(a\mid s ,i)$ is a ratio of two quantities: The numerator is exactly $\fk\ind{t}(i\mid{}s,a)$, and the denominator is a ``constant'' whose value does not depend on $a$. With some manipulation, we can use this fact to relate suboptimality with respect to $\fk\ind{t}$ to the regression error in \pref{eq:mle_tabular}, leading to the following ``local'' optimality guarantee for $\pihat\ind{i,t}$ (see \cref{app:geniklemma0proof} for a proof).
\begin{lemma}[Local optimality guarantee]
\label{lem:proppost}
Consider the setting of \cref{thm:geniklemma0} and let $t\in[h-1]$. Then, there is an event $\cE_t$ of probability at least $1-\delta/H^2$ under which the learned partial policies $\crl*{\pihat\ind{i,t}}_{i\in\brk{S}}$ and $\crl{\pihat \ind{i,t+1}}_{i\in[S]}$ in $\ikdptab$ have the property that for all $i \in \cS_{h}$, 
\begin{align}
	\sum_{\pi \in \Psi\ind{t}}	\dbar^{\pi}(s_t)	\left(	\max_{a\in \cA}	\fk\ind{t}(i \mid s_t, a) - \fk\ind{t}(i \mid s_t, \pihat\ind{i,t}(s_t))\right) \leq 2S AC\veps_\stat(n, \delta /(SH^2)), \quad \forall s_t \in \cS_t,\label{eq:high}
	\end{align}
where $\veps_\stat(\cdot,\cdot)$ and $C>0$ are as in \cref{lem:regtab}; here $C>0$ is an absolute constant independent of problem parameters.
\end{lemma}

\begin{remark}
  For the tabular setting where $\s_h$ is observed, it is also possible to estimate the function $\fk\ind{t}(i\mid s,a )$ directly. However, in the BDMP setting, estimating forward kinematics is \emph{not possible} because states are not observed. We will see that in spite of this, the multi-step inverse kinematics objective used in \ikdp still serves as a useful proxy for the forward kinematics.
\end{remark}
We now use \cref{lem:proppost} to prove \cref{lem:prop}.

\begin{proof}[Proof of \cref{lem:prop}] To prove \cref{lem:prop}, we translate the local suboptimality guarantee in \pref{eq:high} to the global guarantee in \pref{eq:prop}. Fix $\i \in \cS_{h,\epsilon}$ and let us abbreviate $\veps_\stat' \equiv C\cdot \veps_\stat(n, \delta /(S H^2))$, where $\veps_\stat(\cdot,\cdot)$ and $C>0$ are as in \cref{lem:regtab}. Condition on the event $\cE_t$ of \cref{lem:proppost}. We begin by writing the \lhs of \cref{eq:prop} in a form that is closer to the \lhs of \pref{eq:high}:
\begin{align}
\dbar^{\pi_\star^{(\i)}\circ_{t+1} \pihat\ind{\i,t+1}}(\i)  - \dbar^{\pi_\star^{(\i)}\circ_{t}\pihat\ind{\i,t}}(\i)
  &= \sum_{s\in \cS_t\cup\{\tfrak_t\}}	\dbar^{\pi_\star^{(\i)}}(s)\cdot \left( \fk\ind{t}(\i\mid s, \pi_\star^{(\i)}(s) ) - \fk\ind{t}(\i \mid s, \pihat\ind{\i,t}(s))\right), \label{eq:propo}
\end{align}
where we use the convention that $\pihat\ind{\i,t}(\tfrak_t)=\afrak$; this equality follows by the definition of $\fk\ind{t}$ in \pref{eq:fk}.  
Now, we bound the \rhs of \cref{eq:propo} in terms of the \lhs of \pref{eq:high} by using that $\Psi\ind{t}$ is a relative policy cover. In particular, since $\Psi\ind{t}$ is an $(1/2,{\epsilon})$-policy cover relative to $\Pibar_{\epsilon}$ at layer $t$, and since $\pi_\star\ind{\i}\in\Pibar_\eps$, we have that for all $s_t \in \cS_{t,\epsilon}$, 
\begin{align}
&	\dbar^{\pi_\star^{(\i)}}(s_t) \left(\fk\ind{t}(\i \mid s_t,\pi_\star^{(\i)}(s_t)) - \fk\ind{t}(\i\mid s_t, \pihat\ind{\i,t}(s_t))  \right) \nn \\
	& \leq 	\dbar^{\pi_\star^{(\i)}}(s_t)	\left(\max_{a\in \wbar\cA} \fk\ind{t}(\i\mid s_t,a) - \fk\ind{t}(\i\mid s_t, \pihat\ind{\i,t}(s_t)) \right), \nn \\ 
	&= 	\dbar^{\pi_\star^{(\i)}}(s_t)	\left(\max_{a\in \cA} \fk\ind{t}(\i\mid s_t,a) - \fk\ind{t}(\i\mid s_t, \pihat\ind{\i,t}(s_t)) \right), \label{eq:india} \\ 
	& \leq  2	\sum_{\pi \in \Psi\ind{t}}	\dbar^{\pi}(s_t)	\left(\max_{a\in \cA} \fk\ind{t}(\i\mid s_t,a) - \fk\ind{t}(\i\mid s_t, \pihat\ind{\i,t}(s_t)) \right),\nn \\
	& \leq   4S A\veps_\stat',\label{eq:forget000}
\end{align}
where \cref{eq:india} follows from the fact $\fk\ind{t}(\i\mid s_t,\fraka)=0$ (since $\fraka$ is the action leading to the terminal state $\tfrak_{t+1}$ from any state at layer $t$), so that $\max_{a\in \wbar\cA} \fk\ind{t}(\i\mid s_t,a)=\max_{a\in \cA} \fk\ind{t}(\i\mid s_t,a)$; \cref{eq:forget000} follows from \cref{eq:high} in \cref{lem:proppost}. On the other hand, by \cref{lem:bla}, we have that for all $s_t \in \cS_t \setminus \cS_{t,\epsilon}$, $\pi_\star^{(\i)}(s_t)=\fraka$. Therefore,
\begin{align}
	\forall s_t \in \cS_t\setminus \cS_{t,\epsilon}, \quad			\fk\ind{t}(\i \mid s_t,\pi_\star^{(\i)}(s_t))  = 	\fk\ind{t}(\i \mid s_t,\afrak)  =0.\label{eq:xxx}
\end{align}
Using this together with the fact that $ \fk\ind{t}(\i\mid s_t, \pihat\ind{\i,t}(s_t))\geq 0$ and \cref{eq:forget000} implies that
\begin{align}
	\forall s_t \in \cS_{t},\quad 		\dbar^{\pi_\star^{(\i)}}(s_t) \left(\fk\ind{t}(\i \mid s_t,\pi_\star^{(\i)}(s_t)) - \fk\ind{t}(\i\mid s_t, \pihat\ind{\i,t}(s_t))  \right)\leq 4S A\veps_\stat'.\label{eq:tosum}
\end{align}
Now, by choosing $n$ large enough such that $8 HS^2 A C\veps_{\stat}(n,\frac{\delta}{ SH^2}) \leq  \epsilon$ (as in the lemma's statement), we have $8 HS^2 A \veps_\stat' \leq \epsilon$ by definition of $\veps_\stat'$. Using this and summing \eqref{eq:tosum} over $s_t\in \cS_t$ in \eqref{eq:tosum} we have that
\begin{align}
  \sum_{s\in \cS_t\cup\{\tfrak_t\}}	\dbar^{\pi_\star^{(\i)}}(s)\cdot \left( \fk\ind{t}(\i\mid s, \pi_\star^{(\i)}(s) ) - \fk\ind{t}(\i \mid s, \pihat\ind{\i,t}(s))\right)
  \leq \frac{\epsilon}{2H},  \label{eq:preperf}
\end{align}
where we have used that $\fk\ind{t}(\i\mid\term_t,\cdot)=0$.
\end{proof}

	\subsection{From Tabular MDPs to Block MDPs}
	\label{sec:block}

	We now give an overview of the proof of \pref{thm:policycover}. The proof builds on the techniques in \pref{sec:warmup} and follows the same structure, but requires non-trivial changes to accommodate the general BMDP setting. We highlight the most important similarities and differences below, with the full proof deferred to \pref{sec:BMDP}.

Recall that on the algorithmic side, the main change in moving from the tabular setting to the general BMDP setting is that the latent states $\s_h$ are unobserved. To address this, the multi-step inverse kinematics objective in \ikdp (\pref{line:inversenew2}) differs from the simplified version in \ikdptab by incorporating estimation of a decoder $\phihat\ind{t}\in\Phi$ at each step $t\in\brk{h-1}$. Here, a critical property of the multi-step inverse kinematics objective is that the Bayes-optimal regression function (the BMDP analogue of \pref{eq:bayes_tabular}) only depends on the observations $\x_t$ and $\x_h$ through $\phistar(\x_t)$ and $\phistar(\x_h)$, which ensures that the conditional density estimation problem in \pref{line:inversenew2} is always well-specified.

\paragraph{The need for non-Markovian policies}
\ikdp also differs from \ikdptab in how we construct the partial policy collection $\crl{\pihat\ind{i,t}}_{i\in\brk{S}}$ for layer $t\in\brk{h-1}$ from the collection $\crl{\pihat\ind{i,t+1}}_{i\in\brk{S}}$ learned at layer $t+1$. The construction in \cref{line:nonmark} of \ikdp, as discussed in \cref{sec:musik}, leads to policies that are \emph{non-Markovian} (that is, history-dependent). This complicates the analysis because we cannot appeal to the performance difference lemma in the same fashion \pref{sec:warmup} (see \cref{rem:performance}), where it was used to relate the local suboptimality for each policy to global suboptimality. Before giving an overview for how we overcome this challenge, we first give a more detailed explanation as to \emph{why} $\ikdp$ builds non-Markovian policies.

Fix $h\in\brk{H}$. Recall that in the tabular setting, for each backward step $t\in\brk{h-1}$, each partial policy $\pihat\ind{i,t}$ constructed in $\ikdptab$ is designed to target the state $i\in \cS_h$. In the BMDP setting, the states $\s_h$ are unobserved, and it is no longer the case that the partial policy $\pihat\ind{i,t} \in \Pinm^{t:h-1}$ constructed in $\ikdp$ targets the state $i\in\cS_h$. Indeed, while we will show that each partial policy $\pihat\ind{i,t}$ (approximately) targets \emph{some} state in $\cS_h$, the algorithm has no way of knowing which one.\footnote{Unless additional assumptions are added, the latent representation may only be learned up to an unknown permutation.} An additional challenge, which motivates the composition rule in \pref{line:nonmark} of \ikdp, is that for each $i\in\brk{S}$, the suffix policy $\pihat\ind{i,t+1}$ and the one-step policy $\ahat\ind{i,t}$ learned in \pref{line:second} may target \emph{different latent states}, so it does not suffice to simply construct $\pihat\ind{i,t}$ by composing them.
This motivates the second key difference between the multi-step inverse kinematics objectives used in \ikdp versus \ikdptab. The objective in $\ikdp$ predicts both actions and \emph{indices of roll-out policies} (instead of just actions, as in the tabular case) in order to learn to associate partial policies at successive layers. In particular, recall that \pref{line:second} of \ikdp defines
\[
  (\ahat\ind{i,t}(x), \iotahat\ind{i,t}(x))=
  \argmax_{(a,j)} \fhat\ind{t}((a,j) \mid
  \phihat\ind{t}(x),i), \quad x\in \cX_t.
\]
As described in \cref{sec:musik}, one should interpret $j=\iotahat\ind{i,t}(x)$ as the \emph{most likely} (or most closely associated) roll-out policy $\pihat\ind{j,t+1}$ given that the (decoded) latent state at layer $h$ is $i\in\brk{S}$ and $x\in\cX_t$ is the current observation at layer $t$. With this in mind, the composition rule in \pref{line:nonmark} constructs $\pihat\ind{i,t}$ via
\[
  \hat  \pi^{(i,t)}(x_{t:\tau})\coloneqq\left\{
\begin{array}{ll}
                            \ahat\ind{i,t}(x_t),&\quad \tau=t,\;\;
                                                  x_t\in\cX_t,\\
                            \pihat
                            ^{(\iotahat\ind{i,t}(x_t),t+1)}(x_{t+1:\tau}),&\quad\tau\in\range{t+1}{h-1},\;\;
                                                                   x_{t:\tau}\in \cX_t\times \dots \times\cX_{\tau}.
                          \end{array}
\right.
\]
For layers $t+1,\ldots,h-1$, this construction follows the policy $\pihat\ind{\iotahat\ind{i,t}(x_t),t+1}$ which---per the discussion above---is most associated with the decoded state $i\in\brk{S}$. At layer $t$, we select $a_t=\ahat\ind{i,t}(x_t)$, which maximizes the probability of reaching the decoded latent state $i\in\brk{S}$ when we roll-out with $\pihat\ind{\iotahat\ind{i,t}(x_t),t+1}$. The construction, while intuitive, is non-Markovian, since for layers $t+1$ and onward the policy depends on $x_t$ through $\iotahat\ind{i,t}(x_t)$.

\paragraph{Analysis by induction}
The proof of \pref{thm:policycover} follows the same high-level structure as \pref{thm:trend} (\musiktab), and we use the same induction strategy: For each layer $h\in\brk{H}$, we assume that $\Psi\ind{1},\dots, \Psi\ind{h-1}$ are
approximate policy covers relative to $\Pibar_\eps$ for $\cMbar$, then show that the collection $\Psi\ind{h}$ produced by $\ikdp(\Psi\ind{1},\dots,\Psi\ind{h-1}, \Phi,n)$ is an approximate cover with high probability whenever this holds. As with the tabular setting, a key component in our proof is to work with the extended BMDP and truncated policy class throughout the induction, and only pass back to the original BMDP at the end.

The following result (proven in \pref{app:geniklemmaproof}) is our main theorem concerning the performance of \ikdp, and serves as the BMDP analogue of \cref{thm:geniklemma0}.
	\begin{theorem}[Main Theorem for \ikdp]
		\label{thm:geniklemma}
                Let $\epsilon,\delta \in(0,1)$ and $h\in[H]$ be given, and define $\veps_\stat(n,\delta)\coloneqq n^{-1/2}  \sqrt{S^3 A\ln n + \ln (|\Phi|/\delta)}$. Assume that:
                \begin{enumerate}
                \item \ikdp is invoked with $\Psi\ind{1},\ldots, \Psi\ind{h-1}$ satisfying \cref{eq:inductouter}.
        \item The policies in $\Psi\ind{1},\ldots, \Psi\ind{h-1}$ never take the terminal action $\afrak$.
        \item The parameter $n$ is chosen such that $8 A  S^4 HC\veps_{\stat}(n,\frac{\delta}{ H^2})\leq \epsilon$, for some absolute constant $C>0$ independent of $h$ and other problem parameters.
      \end{enumerate}
      Then, with probability at least $1-\frac{\delta}{H}$, the collection $\Psi\ind{h}$ produced by $\ikdp(\Psi\ind{1},\dots,\Psi\ind{h-1},\Phi,n)$ is an $(1/2,\epsilon)$-policy cover relative to $\Pibar_{\epsilon}$ in $\Mbar$ for layer $h$. In addition, $\Psi\ind{h}\subseteq \Pinm^{1:h-1}$.
	\end{theorem}
        We close the section by highlighting some key differences between the proof of this result and its tabular counterpart (\cref{thm:geniklemma0}).\loose
       
\paragraph{An alternative to \cref{lem:prop}}
Recall that in the tabular setting, the proof of \cref{thm:geniklemma0} relied on \cref{lem:prop} and the performance difference lemma (see \cref{rem:performance}). In the BMDP setting, \cref{lem:prop} does not necessarily hold since, unlike in the tabular setting, successive partial policies $\pihat\ind{i,t} \in \Pinm^{t:h-1}$ and $\pihat\ind{i,t+1} \in \Pinm^{t+1:h-1}$ may target different states at layer $h$ despite sharing the same index $i\in[S]$. For this reason, we use a modified version of \cref{lem:prop}, together with a generalized version of the performance difference lemma.  
\begin{lemma}[BMDP counterpart to \cref{lem:prop}]
	\label{lem:propnew}
	There is an absolute constant $C>0$ such that for all $t\in[h-1]$, with probability at least $1-\delta/H^2$, the learned partial policies $\crl*{\pihat \ind{i,t}}_{i\in\brk{S}}$ and $\crl*{\pihat \ind{i,t+1}}_{i\in[S]}$ in $\ikdp$ have the property that for all $s_h \in \cS_{h,\epsilon}$, there exists $\i\in [S]$ such that
	\begin{align}
0	\leq 	\sum_{\pi \in \Psi\ind{t}}	\bar{d}^{\pi}(s_t) \Ebar_{\x_t \sim q(\cdot \mid s_t)}	\left[	\max_{a\in \cA, j \in[S]}	Q_t^{\pihat\ind{j,t+1}}(\x_t,a;s_h) - V_t^{\pihat\ind{\i,t}}(\x_t;s_h) \right] \leq 2S^3 A C \veps_\stat(n, \tfrac{\delta}{H^2}), \ \ \forall s_t \in \cS_t, \label{eq:propnew}
	\end{align}
where $Q_t^{\pihat\ind{j,t+1}}(x_t,a;s_h)\coloneqq \Pbar^{\pihat\ind{j,t}}[\s_h = s_h \mid \x_t = x_t, \a_t = a]$, $V_t^{\pihat\ind{\i,t}}(x_t;s_h) \coloneqq Q_t^{\pihat\ind{\i,t+1}}(x_t,\pihat\ind{\i,t}(x_t);s_h)$, and $\veps_\stat(n,\delta')\coloneqq n^{-1/2}  \sqrt{S^3 A\ln n + \ln (|\Phi|/\delta')}$.
\end{lemma}
This result is proven in \cref{sec:errorbound}. To see the similarity between \cref{lem:propnew} and \cref{lem:prop}, note that the main quantity that the latter bounds (i.e.~the quantity on the \rhs of \pref{eq:prop}) can also be written as a difference between $Q$; see \cref{rem:performance}. 
 Once \cref{lem:propnew} is established, it can be shown to imply \cref{thm:geniklemma} using a generalized variant of the performance difference lemma (\cref{lem:performancediffbmdp}).

\paragraph{Establishing \cref{eq:propnew} using multi-step inverse kinematics}
To show that \pref{eq:propnew} holds, we use the structure of the multi-step inverse kinematics objective in \pref{line:inversenew2} of \ikdp, as well as the non-Markov policy construction outlined in the prequel. In particular, we show that the multi-step inverse kinematics objective acts as a proxy for the forward kinematics given by 
\begin{align}
\P^{\pihat\ind{i,t+1}}[\s_h=\phi_\star(x_h)\mid \s_t = \phi_\star(x_t), \a_t = a], \nn
\end{align}
for $i\in[S]$, $x_t \in \cX_t$ and $x_h \in \cX_h$. We use this to show that up to statistical error, the partial policies $\crl{\pihat\ind{i,t}}_{i\in[S]}$ constructed from $\crl{\pihat\ind{i,t+1}}_{i\in[S]}$ i) identify (using observations at layer $t$) the best action at layer $t$,  and ii) identify the best partial policy from $\crl{\pihat\ind{j,t+1}}_{j\in[S]}$ to switch to from layer $t+1$ onwards.

Beyond the multi-step inverse kinematics objective and non-Markov policy construction, the proof of \cref{thm:geniklemma} uses the extended BMDP in a similar fashion to the tabular setting. We make use of the fact that for each layer $t\in\brk{h-1}$,  the policies in $\Pibar_{\epsilon}$ always play the terminal action $\afrak$ on observations emitted from states in $\cS_{t,\epsilon}$, and the generalized performance difference lemma (\cref{lem:performancediffbmdp}) is specifically designed to take advantage of this. This allows us to ``write off'' these states (analogous to \cref{eq:xxx} in the proof of \cref{lem:prop}), and use the policy cover property for $\Psi\ind{1},\ldots,\Psi\ind{h-1}$ to control the error for states in $\cS_{t,\eps}$; see \cref{sec:proofbmdp} for details. However, there is some added complexity stemming from the non-Markovian nature of $\crl{\pihat\ind{i,t}}_{i\in\brk{S}}$.

        }

	\section{Proofs for Structural Results for Extended BMDP}
	\label{app:structural}

In this section, we prove the main structural results concerning the extended BMDP and truncated policy class introduced in \cref{sec:keytools}. We first recall the definition of the truncated policy class. For $\epsilon \in(0,1)$, let $\Pibar_{0, \epsilon},\dots,\Pibar_{H, \epsilon}$ be the policies defined recursively as follows: $\Pibar_{0, \epsilon}=\Pibarm$ and for all $t\in [H]$, $\pi \in \Pibar_{t, \epsilon}$ if and only if there exists $\pi'\in \Pibar_{t-1, \epsilon} $ such that for all $h\in[H]$, $s\in \wbar\cS_h$, and $x\in \phi^{-1}_\star(s)$,
\begin{align}
	\pi(x) \coloneqq \left\{\begin{array}{ll} \fraka, & \text{if } h=t \text{ and } \max_{\tilde \pi\in \Pibar_{t-1, \epsilon}} \dbar^{\tilde \pi}(s)< \epsilon, \\   \pi'(x), & \text{otherwise}.    \end{array}    \right.  \label{eq:define101}
\end{align}
Finally, we let $\Pibar_{\epsilon} \coloneqq \Pibar_{H,\epsilon}$.

The proofs in this section make use of the following lemma.
        \begin{lemma}
          \label{lem:pih_max}
          For all $h \in[H]$, it holds that
          \begin{align}
          	\forall s\in \cS_h, \quad  \max_{\pi\in \Pibar_{h-1,\epsilon}} \dbar^{\pi}(s) =  \max_{\pi\in \Pibar_{\epsilon}} \dbar^{\pi}(s) .\label{eq:key}
          \end{align}
        \end{lemma}
        \begin{proof}[\pfref{lem:pih_max}]
We will show that for all $t\in[h\ldotst{}H]$,
\begin{align}
	\forall s\in \cS_h, \quad \max_{\pi \in \Pibar_{t-1,\epsilon}}\dbar(s)= 	\max_{\pi \in \Pibar_{t,\epsilon}}\dbar(s). \label{eq:tele}
\end{align}
This implies \cref{eq:key} by summing both sides of \cref{eq:tele} over $t=h,\dots, H$, telescoping, and using that $\Pibar_{\epsilon}=\Pibar_{H, \epsilon}$. To prove the result, let $t\in[h\ldotst{}H]$, $s\in \cS_h$, and $\tilde\pi \in \argmax_{\pi'\in \Pibar_{t-1,\epsilon}} \dbar^{\pi'}(s)$. Further, let $\pi\in \Pibar_{t, \epsilon}$ be as in \cref{eq:define101} with $\pi'=\tilde \pi$. In this case, by \cref{eq:define101}, we have $\tilde \pi(x)=\pi(x)$, for all $x\in \phi^{-1}(s')$, $s'\in \cS_{\tau}$, and $\tau \leq [t-1]$. Using this and the fact that $s\in \cS_h$ and $t\geq h$, we have 
\begin{align}
	\max_{\breve \pi\in \Pibar_{t-1,\epsilon}} \dbar^{\breve \pi}(s) =\dbar^{\tilde \pi}(s)= \dbar^{\pi}(s) \leq \max_{\breve \pi \in \Pibar_{t, \epsilon}} \dbar^{\breve \pi}(s).\nn
\end{align} 
We now show the inequality in the other direction. Let $t\in[h\ldotst{}H]$, $s\in \cS_h$, and $\tilde \pi \in \argmax_{ \breve\pi\in \Pibar_{t,\epsilon}} \dbar^{\breve \pi}(s)$. Further, let $\pi'\in \Pibar_{t-1, \epsilon}$ be as in \cref{eq:define101} for $\pi = \tilde \pi$. In this case, by \cref{eq:define101}, we have $\tilde \pi(x)=\pi'(x)$, for all $x\in \phi^{-1}(s')$, $s'\in \cS_{\tau}$, and $\tau \in [t-1]$. Using this and the fact that $s\in \cS_h$ and $t\geq h$, we have 
\begin{align}
	\max_{\breve \pi\in \Pibar_{t,\epsilon}} \dbar^{\breve \pi}(s) =\dbar^{\tilde \pi}(s)= \dbar^{\pi'}(s) \leq \max_{\breve \pi \in \Pibar_{t-1, \epsilon}} \dbar^{\breve \pi}(s).\nn
\end{align} 
This shows \cref{eq:tele} and completes the proof.
        \end{proof}

\subsection{Proof of \creftitle{lem:bla}}
	\label{app:blaproof}

	\begin{proof}[\pfref{lem:bla}]
		Fix $h \in[H]$. %
		We proceed by induction on $t$ to show that 
		\begin{align}
			\forall t \in [h\ldotst{}H], \forall \pi\in \Pibar_{t, \epsilon}, \exists \pi'\in \Pibar_{h,\epsilon}\colon  \forall  s\in \cS_{h}, \forall x\in \phi^{-1}_\star(s),  \quad \pi(x)=\pi'(x). \label{eq:induct}
		\end{align}
		For $t=h$, \cref{eq:induct} holds trivially. Now, we suppose that \cref{eq:induct} holds for $t \in [h\ldotst{}H-1]$, and show that it holds for $t+1$. By definition of $\Pibar^{(t+1)}_{\epsilon}$ (\cref{eq:define101}), there exists $\tilde \pi\in \Pibar_{t, \epsilon}$ such that 
		\begin{align}
	\forall \tau\in [t], 	\forall s\in \cS_{\tau},  \forall  x\in \phi^{-1}_\star(s), \quad 	\pi(x)=\tilde \pi(x). \label{eq:tocombine}
		\end{align} 
		Now, by the induction hypothesis, there exists $\pi'\in \Pibar_{h,\epsilon}$ such that $\tilde \pi(x)=\pi'(x)$, for all $x\in \phi_\star^{-1}(s)$ and $s\in \cS_{h}$. Combining this with \cref{eq:tocombine} and the fact that $t\geq h$ implies that \cref{eq:induct} holds for $t+1$, which concludes the induction. 
		
		Now, by instantiating \cref{eq:induct} with $t=H$ and recalling that $\Pibar_{\epsilon}=\Pibar_{\epsilon}^{(H)}$ (by definition), we get that 
		\begin{align}
			\forall \pi\in \Pibar_{\epsilon}, \exists \pi'\in \Pibar_{h,\epsilon}\colon \forall s\in \cS_{h}, \forall x\in \phi^{-1}(s), \quad \pi(x)=\pi'(x). \label{eq:postinduct}
		\end{align}
By \cref{lem:pih_max}, this implies that for any $\sfrak\in \cS_h \setminus \cS_{h,\epsilon}$, $\max_{\pi\in \Pibar^{(h-1)}_{\epsilon}} \dbar^{\pi}(\sfrak)\leq \epsilon$. It follows that for all $\pi' \in \Pibar^{(h)}_\epsilon$ and $x\in \phi^{-1}_\star(\sfrak)$, we have $\pi'(x)=\afrak$, by definition of $\Pibar_{h,\epsilon}$; see \cref{eq:define101}. 
		This together with \cref{eq:postinduct} implies that $\pi(x)=\afrak$ for all $x\in \phi^{-1}_\star(\sfrak)$ and $\pi \in \Pibar_\epsilon$, as desired.
	\end{proof}

        	\subsection{Proof of \creftitle{lem:critical}
                  (Approximation for Truncated Policy Class)}
	\label{app:criticalproof}

	\begin{proof}[\pfref{lem:critical}]
		We will show that for all $t \in [H]$, $h \in[H]$, and $s\in \cS_h$, 
		\begin{align}
			\max_{\pi \in \Pibar_{t-1, \epsilon}} \dbar^{\pi}(s) \leq \max_{\pi \in \Pibar_{t, \epsilon}} \dbar^{\pi}(s) + |\cS_t| \epsilon.\label{eq:foal}
		\end{align}
		With this established, summing \cref{eq:foal} over $t$, telescoping, and using that $\sum_{t\in[H]}|\cS_t|=S$ implies the desired result. 
		
		Let $t\in[H]$, $h \in [H]$ and $s\in \cS_h$.  Further, let $\tilde{\pi} \in \argmax_{\pi \in \Pibar_{t-1, \epsilon}} \dbar^{\pi}(s)$ and let $\pi \in \Pibar_{t, \epsilon}$ be as in \cref{eq:define101} for $\pi'=\tilde \pi$. First, suppose that $h \leq t$. Then, $\dbar^{\pi}(s) = \dbar^{\tilde{\pi}}(s)$ since $\pi|_{\cX_\tau}\equiv \tilde{\pi}|_{\cX_\tau}$ for all $\tau<t$, and so by our choice of $\tilde{\pi}$ and that $\pi\in \Pibar_{t,\epsilon}$, we have 
		\begin{align}
			\max_{\breve \pi \in \Pibar_{t-1, \epsilon}}\dbar^{\breve \pi}(s)= \dbar^{\tilde{\pi}}(s)= \dbar^{\pi}(s)\leq \max_{\breve \pi \in \Pibar_{t, \epsilon}}\dbar^{\breve \pi}(s).\nn
		\end{align}
		Now suppose that $h> t$. We will use that I) $\pi|_{\cX_\tau}= \tilde{\pi}|_{\cX_\tau}$, for all $\tau \neq t$; and II)  $\pi(x)= \tilde{\pi}(x)$, for all $x\in \phi^{-1}_\star(s')$ and $s'\in \cS_{t,\epsilon}$, by definition of $\cS_{t,\epsilon}$ and $\pi$, and \cref{lem:pih_max}. We note that I) implies that $\dbar^{\tilde{\pi}}(s')=\dbar^{\pi}(s')$, for all $s' \in \cS_t$, and the combination of I) and II) implies that $\Pbar^{\pi}[\s_h = s\mid \s_t= s'] =\Pbar^{\tilde{\pi}}[\s_h = s\mid \s_t= s']$, for all $s'\in \cS_{t,\epsilon}$. Using these facts, we have 
		\begin{align}
			\dbar^{\tilde{\pi}}(s) & =\sum_{s'\in \cS_t} \Pbar^{\tilde{\pi}}[\s_h = s\mid \s_t= s'] \cdot \dbar^{\tilde{\pi}}(s'),\nn \\
			& =\sum_{s'\in \cS_{t,\epsilon}} \Pbar^{\tilde{\pi}}[\s_h = s\mid \s_t= s'] \cdot \dbar^{\tilde{\pi}}(s')+ \sum_{s'\in \cS_t\setminus \cS_{t,\epsilon}} \Pbar^{\tilde{\pi}}[\s_h = s\mid \s_t= s'] \cdot \dbar^{\tilde{\pi}}(s'),\nn \\
			& = \sum_{s'\in \cS_{t,\epsilon}} \Pbar^{\pi}[\s_h = s\mid \s_t= s'] \cdot \dbar^{\pi}(s') + \sum_{s'\in \cS_t \setminus \cS_{t,\epsilon}} \Pbar^{\tilde{\pi}}[\s_h = s\mid \s_t= s'] \cdot \dbar^{\tilde{\pi}}(s'),\nn \\
			& \leq  \dbar^{\pi}(s)+ \sum_{s'\in \cS_t \setminus \cS_{t,\epsilon}} \Pbar^{\tilde{\pi}}[\s_h = s\mid \s_t= s'] \cdot \dbar^{\tilde{\pi}}(s'),\nn \\
			& \leq \dbar^{\pi}(s) + |\cS_t| \epsilon,\label{eq:lastone}
		\end{align}
		where the last inequality follows because $\dbar^{\tilde{\pi}}(s')\leq \epsilon$ for all $s' \in \cS_t\setminus \cS_{t,\epsilon}$, which follows from the definition of $\cS_{t,\epsilon}$, \cref{lem:pih_max}, and $\tilde{\pi}\in \Pibar_{t-1, \epsilon}$. %
                Combining \cref{eq:lastone} with the fact that $\tilde{\pi} \in \argmax_{\pi \in \Pibar_{t-1, \epsilon}} \dbar^{\pi}(s)$ and $\pi\in \Pibar_{t,\epsilon}$, we have that
                \begin{align}
			\max_{\breve \pi \in \Pibar_{t-1, \epsilon}}\dbar^{\breve \pi}(s)= \dbar^{\tilde{\pi}}(s)\leq  \dbar^{\pi}(s) + |\cS_t|\epsilon\leq \max_{\breve \pi \in \Pibar_{t, \epsilon}}\dbar^{\breve \pi}(s)+|\cS_t|\epsilon.\nn
		\end{align}
	\end{proof}

\subsection{Proof of \creftitle{lem:transfer}}
\label{app:transferproof}

\begin{proof}[\pfref{lem:transfer}]
	Let $\Pibarm$ be as in \cref{lem:critical}, and note that since $\Pim \subseteq \Pibarm$, we have for any $s\in\cS$,
	\begin{align}
		\max_{\pi \in \Pim} d^{\pi}(s) \leq  \max_{\pi \in \Pibarm} \dbar^{\pi}(s)\leq \max_{\pi \in \Pibar_{\epsilon}} \bar d^{\pi}(s) + S\epsilon. \label{eq:sweet}
	\end{align}
	where the last inequality follows by \cref{lem:critical}. Now, fix $s\in \cS$ such that $	\max_{\pi \in \Pim} d^{\pi}(s) \geq \veps$. Using \cref{eq:sweet} and that $\epsilon =\veps/(2S)$, we have $\max_{\pi \in \Pibar_{\epsilon}} \bar d^{\pi}(s)\geq \veps/2\geq \epsilon.$ Thus, since $\Psi$ is a $(1/2,\epsilon)$-policy cover relative to $\Pibar_\epsilon$ in $\Mbar$ for all layers, there exists $\pi\ind{s}\in \Psi$ such that 
	\begin{align}
		\frac{1}{2}	\max_{\pi \in \Pibar_\epsilon} \bar d^{\pi}(s) \leq  	 \bar d^{\pi\ind{s}}(s) =  d^{\pi\ind{s}}(s), \nn 
	\end{align}
	where the equality follows from the assumption that policies in $\Psi$ never take the terminal action $\afrak$. Combining this inequality with \pref{eq:sweet}, we have that
        \[
          \max_{\pi \in \Pim} d^{\pi}(s) \leq 2{}d^{\pi\ind{s}}(s)
          + S\epsilon.
        \]
        Since $S\eps=\frac{\veps}{2}\leq{}\frac{1}{2}\max_{\pi \in \Pim} d^{\pi}(s)$, rearranging gives
                \[
                  \frac{1}{4}\max_{\pi \in \Pim} d^{\pi}(s) \leq {}d^{\pi\ind{s}}(s),
                \]
                which concludes the proof.
              \end{proof}

              \begin{remark}
                The proof of \pref{lem:transfer} actually gives a result slightly stronger than what is stated in the lemma. Namely, it suffices for $\Psi$ to be a $(1/2, \veps/2)$-policy cover relative to $\Pibar_\eps$ in $\cMbar$ (as opposed to a $(1/2, \eps)$-policy cover). We state the weaker result because our analysis of \pref{alg:IKDP} does not take advantage of the stronger result.
              \end{remark}

	 \section{Proofs for Tabular MDPs}
	 \label{sec:tabular}

\subsection{Proof of \creftitle{lem:regtab} (MLE Guarantee for Tabular MDPs)}
\label{app:regtabproof}

To prove a guarantee for the minimizer $\fhat\ind{j,t}$ for the conditional density estimation problem in \cref{line:inversenew20} of \cref{alg:IKDP-tab}, we first derive the expression of the Bayes-optimal solution of this problem.
\begin{lemma}
	\label{lem:bayes0}
	Let $h\in [H]$, $t\in[h-1]$, $i\in[S]$, and consider the Bayes-optimal solution $P\ind{i,t}_{\bayes}$ of the problem in \cref{line:inversenew20} of \cref{alg:IKDP-tab}; that is,
	\begin{align}
		P_{\bayes}\ind{i,t} \in \argmax_{P\colon \cS_t \times \cS_h \rightarrow \Delta(\cA\times [S])} \E^{\unif(\Psi\ind{t})\circ_t \unifa \circ_{t+1}  \pihat\ind{i,t+1}} \left[ \ln   P(\a_t\mid \s_t, \s_h)\right]. \label{eq:normal0}
	\end{align}
	Then, for any $a\in\cA$, $s \in \cS_t$, and $s' \in \cS_h$, $P_{\bayes}\ind{i,t}$ satisfies
	\begin{align}
		P\ind{i,t}_{\bayes}(a\mid s, s') \coloneqq \frac{\Pbar^{\pihat\ind{i,t+1}}[\s_h=s' \mid \s_t=s,\a_t=a]}{ \sum_{a'\in \cA} \Pbar^{  \pihat\ind{i,t+1}}[\s_h=s' \mid \s_t=s,\a_t=a']}. \nn 
	\end{align}
\end{lemma}
\begin{proof}[\pfref{lem:bayes0}]
	Fix $a\in\cA$ and $(s,s')\in \cS_{t}\times \cS_{h}$. The solution $P_{\bayes}\ind{i,t} $ of the problem in \cref{eq:normal0} satisfies 
	\begin{align}
		P_{\bayes}\ind{i,t} (a\mid s,s') & = 	 \P^{\unifa \circ_{t+1} \pihat\ind{i,t+1}}[\a_t=a\mid \s_t=s,\s_h=s'], \nn \\
		& =  \frac{	\P^{\pihat\ind{i,t+1}}[\s_h=s' \mid \s_t=s,\a_t=a] \cdot \P^{\unifa \circ_{t+1} \pihat\ind{i,t+1}}[\a_t=a \mid \s_t=s]}{\sum_{a'\in \cA} \P^{\pihat\ind{i,t+1}}[\s_h=s' \mid \s_t=x,\a_t=a'] \cdot \P^{\unifa\circ_{t+1}  \pihat\ind{i,t+1}}[\a_t=a'\mid \s_t=s]}, \label{eq:pass0}
	\end{align} 
	where the last equality follows by Bayes Theorem; in particular the fact that \begin{align*}
		\mu [A\mid B, C]=\frac{\mu[B\mid A, C]\cdot \mu[A\mid C]}{\mu[B\mid C]}
                                                                                      \end{align*} applied with $A=\{\a_t =a\}$, $B= \{\s_h=s' \}$, $C= \{\s_t = s\}$, and $\mu\equiv \Pbar^{\unifa\circ \pihat\ind{i,t+1}}$. Now, by combining \cref{eq:pass0} with the fact
                                                                                      that $\a_t$ is independent of $\s_t$, we get that 
	\begin{align}
		P_{\bayes}\ind{i,t} (a\mid s,s')	& = \frac{\P^{\pihat\ind{i,t+1}}[\s_h=s' \mid \s_t=s,\a_t=a] \cdot \P^{\unifa\circ_{t+1} \pihat\ind{i,t+1}}[\a_t=a]}{\sum_{a'\in \cA} \P^{\pihat\ind{i,t+1}}[\s_h=s' \mid \s_t=s,\a_t=a'] \cdot \P^{\unifa \circ_{t+1} \pihat\ind{i,t+1}}[\a_t=a']}, \nn \\ 
		& = \frac{\P^{\pihat\ind{i,t+1}}[\s_h=s' \mid \s_t=s,\a_t=a] }{\sum_{a'\in \cA} \P^{\pihat\ind{i,t+1}}[\s_h=s' \mid \s_t=s,\a_t=a']}, \label{eq:bayes20}
	\end{align}
	where \cref{eq:bayes20} follows because $\a_t\sim\unifa$. 
	Now, since the partial policy $\pihat\ind{i,t+1}$ never takes the terminal action, we have $\P^{\pihat\ind{i,t+1}}[\s_h=s' \mid \s_t=s,\a_t=a]=\Pbar^{\pihat\ind{i,t+1}}[\s_h=s' \mid \s_t=s,\a_t=a]$ (the \lhs has $\P$ while the \rhs has $\Pbar$) for all $a\in\cA$ and $i\in[S]$. This, together with \cref{eq:bayes20} implies \begin{align} P\ind{i,t}_{\bayes}(a\mid s,s') =  \frac{\Pbar^{\pihat\ind{i,t+1}}[\s_h=s' \mid \s_t=s,\a_t=a]}{\sum_{a'\in \cA} \Pbar^{\pihat\ind{i,t+1}}[\s_h=s' \mid \s_t=s,\a_t=a']}. \nn 
	\end{align}
\end{proof}
\begin{proof}[{Proof of \cref{lem:regtab}}]%
	Fix $t\in[h-1]$ and $i\in[S]$. By \cref{lem:bayes0}, $P\ind{i,t}_{\bayes}\in \{f\colon [S^2]\rightarrow \Delta_A\}$ is the Bayes-optimal solution of the conditional density estimation problem in \cref{line:inversenew20} of \cref{alg:IKDP-tab}. And so, by a standard guarantee for log-loss conditional density estimation (see, e.g.,~\citet[Proposition E.2]{chen2022partially}),\footref{fn:hellinger} there exists an absolute constant $C'>0$ (independent of $t$, $h$, and other problem parameters) such that with probability at least $1-\delta$, 
	\begin{align}
		{\E}^{\unif(\Psi\ind{t})\circ_t \unifa \circ_{t+1} \pihat\ind{i,t+1}} \left[ \sum_{a\in\cA}\left(  \fhat\ind{i,t}(a\mid \s_t,\s_h))	-   P\ind{i,t}_{\bayes}(a\mid \s_t ,\s_h) \right)^2 \right]\leq   \tilde\veps^2_\stat(n,\delta), \label{eq:theoen0}
	\end{align}
	where $\tilde \veps^2_\stat(n,\delta)\coloneqq C' \ln   \cN_{\cF}(1/n)  + C' \ln (1/\delta)$ and $ \cN_{\cF}(\veps)$ denotes the $\veps$-covering number of the set $\cF \coloneqq \{f\colon [S]^2 \rightarrow \Delta_{A}\}$ in $\ell_{\infty}$-distance. It is easy to verify that $\cN_{\cF}(1/n)\leq n^{A S^2}$, and so by setting $C^2 \coloneqq C'$ we have
	\begin{align}
		\tilde \veps_{\stat}^2(n,\delta) \leq C^2 \cdot \veps_{\stat}^2(n,\delta). \label{eq:bound0}
	\end{align} 
	Now, since $\afrak$ is never taken by the partial policies $(\pihat ^{(j)}_{\tau:h-1})_{j\in[S]}$, $\tau \in[h-1]$, in \cref{alg:IKDP} or by the policies in $\Psi\ind{2}, \dots, \Psi\ind{h-1}$ (by assumption), the guarantee in \cref{eq:theoen0} also holds in $\Mbar$. Combining this with \cref{eq:bound0} completes the proof.
\end{proof}

\subsection{Proof of \creftitle{lem:proppost} (Local Optimality Guarantee)}
\label{app:geniklemma0proof}

\begin{proof} [\pfref{lem:proppost}]
Let $\veps_\stat' \coloneqq C\cdot \veps_\stat(n, \delta/(SH^2))$, where $\veps_\stat(n,\delta)$ and $C>0$ are as in \cref{lem:regtab}. We will show that for any $t\in[h-1]$ in \cref{alg:IKDP-tab}, there exists an event $\cE_t$ of probability at least $1-\delta/H^2$ under which the learned partial policies $\{\pihat\ind{i,t}\}_{i\in[S]}$ and $\{\pihat\ind{i,t+1}\}_{i\in[S]}$ are such that for any $i\in \cS_{h}$, we have 
	\begin{align}
	\sum_{\pi \in \Psi\ind{t}}	\dbar^{\pi}(s_t)	\left(	\max_{a\in \cA}	Q_t^{\pihat\ind{i,t+1}}(s_t,a;i) - Q^{\pihat\ind{\i,t+1}}_{t}(s_t, \pihat\ind{i,t}(s_t);i) \right) \leq 2S A\veps_\stat', \quad \forall s_t \in \cS_t.\label{eq:flight20} %
\end{align}
	where $Q^{\pihat\ind{i,t+1}}_{t}(\cdot; i)$ is the $Q$-function at layer $t$ with respect to the partial policy $\pihat\ind{i,t+1}$ for the BMDP $\wbar \cM$ with rewards $r\ind{i}_\tau(s)=\mathbf{1}\{s=i\}\cdot \mathbf{1}\{\tau=h\}$, for $\tau\in[h]$ (see \pref{eq:Qfunction}). This implies the desired result of the lemma, since $\fk\ind{t}(i\mid s,a)=Q_t^{\pihat\ind{i,t+1}}(s,a;i)$; see \pref{eq:FK} and \pref{eq:fk}. We write \eqref{eq:flight20} in terms of $Q$-functions (instead of $\fk\ind{t}$) to highlight similarities with the analysis of \musik{} in the more general BMDP setting.
	
	Fix $t\in[h-1]$ and $\i\in \cS_{h}$. Further, let $\cS^{+}_t$ be the subset of states defined by \begin{align}&\cS^{+}_t \coloneqq \left\{s \in \cS_t   \colon \sum_{\pi \in \Psi\ind{t}} \dbar^{\pi}(s)\sum_{a\in\cA}  P\ind{\i,t}(\i\mid s,a) > 0 \right\},
		\nn\\
		\text{where}\quad &	 P\ind{\i,t}(s'\mid s,a)\coloneqq \Pbar^{\pihat\ind{\i,t+1}}[\s_h=s'\mid \s_t=s,\a_t=a]. \label{eq:things}
	\end{align}
By \cref{lem:regtab} and Jensen's inequality, there is an event $\cE^{(\i)}_t$ of probability at least $1-\delta/(SH^2)$ under which the solution $\fhat\ind{\i,t}$ of the conditional density estimation problem in \cref{line:inversenew20} of \cref{alg:IKDP-tab} satisfies,
	\begin{align}
		\Ebar_{\s_t\sim \unif(\Psi\ind{t}),\a' \sim
			\unifa}  \left[ \sum_{s_h\in \cS_h}  P\ind{\i,t}(s_h\mid \s_t,\a') \max_{a\in\cA}\left| \err\ind{\i,t}(a,\s_t,s_h)\right| \right]\leq   \veps_\stat', \label{eq:neverbefore0}
		\end{align}
where $\err\ind{\i,t}(a,s_t,s_h)  \coloneqq \fhat\ind{\i,t}(a\mid  s_t, s_h)	-  P\ind{\i,t}_{\bayes}(a \mid s_t,s_h)$ and 
\begin{align}  P\ind{\i,t}_{\bayes}(a \mid s,s') \coloneqq \frac{ P\ind{\i,t}(s'\mid s,a)}{ \sum_{a'\in\cA}P\ind{\i,t}(s'\mid s,a')}. \label{eq:definition0}
	\end{align}
	In what follows, we condition on $\cE_t \coloneqq
	\bigcap_{j\in[S]}\cE\ind{j}_t$. Now, fix
	$\j\in \cS^{+}_t.$ From \cref{eq:neverbefore0}, we
	have that
	\begin{align}
		S A\veps_\stat'& \geq 	\sum_{\pi \in \Psi\ind{t}} \sum_{a'\in\cA}\sum_{s_h\in\cS_h}\dbar^{\pi}(\j) P\ind{\i,t}(s_h\mid \j,a') \max_{a\in\cA}\left| \err\ind{\i,t}(a,\j,s_h) \right|,\nn  \\
		& \geq 	\sum_{\pi \in \Psi\ind{t}}\sum_{a'\in \cA}\dbar^{\pi}(\j) P\ind{\i,t}(\i\mid \j,a') \max_{a\in\cA}\left| \err\ind{\i,t}(a,\j,\i) \right|,\nn \\
		& = \sum_{\pi \in \Psi\ind{t}}\sum_{a'\in \cA}\dbar^{\pi}(\j) P\ind{\i,t}(\i\mid \j,a') \max_{a\in\cA}\left| \fhat\ind{\i,t}(a\mid \j, \i)	-  P\ind{\i,t}_{\bayes}(a \mid \j,\i)\right|.\nn 
	\end{align}
	By rearranging and using the fact that $\sum_{\pi \in \Psi\ind{t}}\sum_{a'\in \cA}\dbar^{\pi}(\j)	 P\ind{\i,t}(\i\mid \j,a')>0$ (since $\j\in \cS^{+}_t$), we get
	\begin{align}
		\max_{a\in\cA}\left| \fhat\ind{\i,t}(a\mid \j, \i)	-  P\ind{\i,t}_{\bayes}(a \mid \j,\i)\right|\leq  \frac{S A\veps_\stat'}{\sum_{\pi \in \Psi\ind{t}}\dbar^{\pi}(\j)	\cdot \sum_{a'\in \cA} P\ind{\i,t}(\i\mid \j,a')}. \label{eq:dedant0}
	\end{align}
	Now, let $\ahat\ind{\i,t}(s) \in \argmax_{a\in \cA} \fhat\ind{\i,t}(a\mid s, \i)$ and note that $\ahat\ind{\i,t}(s) = \pihat\ind{\i,t}(s)$, where $\pihat\ind{\i,t}$ is as in \cref{alg:IKDP-tab}. With this, \cref{eq:dedant0} and the fact that $|\|y\|_{\infty}-\|z\|_{\infty}|\leq \|y-z\|_{\infty}$, for all $y,z\in \reals^{A}$, we have that
	\begin{align}
		\max_{a\in \cA}  P\ind{\i,t}_{\bayes}(a \mid \j, \i)& \leq \fhat\ind{\i,t}(\ahat\ind{\i,t}(\j)\mid \j, \i) + \frac{S A\veps_\stat'}{\sum_{\pi \in \Psi\ind{t}}\dbar^{\pi}(\j)	\cdot \sum_{a'\in \cA} P\ind{\i,t}(\i\mid \j,a')}, \nn \\
		& \leq    P\ind{\i,t}(\ahat\ind{\i,t}(\j) \mid  \j, \i) +  \frac{2S A\veps_\stat'}{\sum_{\pi \in \Psi\ind{t}}\dbar^{\pi}(\j)	\cdot \sum_{a'\in \cA} P\ind{\i,t}(\i\mid \j,a')}.  \quad (\text{by \cref{eq:dedant0} again}) \label{eq:control}
	\end{align}
	With this in hand, we have that 
	\begin{align}
		\max_{a\in \cA}	Q_t^{\pihat\ind{\i,t+1}}(\j,a;\i)&=\max_{a\in \cA}  \Pbar^{\pihat\ind{\i,t+1}}[\s_h=\i \mid \s_t=\j,\a_t=a],\nn \\ 
		&=\max_{a\in \cA} P\ind{\i,t}(\i\mid \j,a),  \quad \text{(by definition---see \cref{eq:things})}\nn \\ &  = \max_{a\in \cA} P\ind{\i,t}_{\bayes}(a \mid \j, \i)\sum_{a'\in \cA}P\ind{\i,t}(\i\mid \j,a'),  \quad (\text{by def.~of $P\ind{i,t}_{\bayes}$ in \cref{eq:definition0}})\nn  \\ &  \leq  P\ind{\i,t}(\ahat\ind{\i,t}(\j) \mid \j, \i)  \sum_{a'\in \cA}	P\ind{\i,t}(\i\mid \j,a')+ \frac{2S A\veps_\stat'}{\sum_{\pi \in \Psi\ind{t}}\dbar^{\pi}(\j)},  \quad (\text{by \cref{eq:control}})\nn \\ 
		& =    \Pbar^{\pihat\ind{\i,t+1}}[\s_h=\i \mid \s_t=\j, \a_t = \ahat\ind{\i,t}(\j)] + \frac{2S A\veps_\stat'}{\sum_{\pi \in \Psi\ind{t}}\dbar^{\pi}(\j)}, \quad (\text{by \cref{eq:definition0}}) \nn \\
		& =Q^{\pihat\ind{\i,t+1}}_{t}(\j, \pihat\ind{\i,t}(\j);\i) + \frac{2S A\veps_\stat'}{\sum_{\pi \in \Psi\ind{t}}\dbar^{\pi}(\j)},\nn
	\end{align}
	where the last equality follows by definition of $\pihat\ind{\i,t}$ in \cref{alg:IKDP-tab}.

	The argument above implies that
	\begin{align}
		\sum_{\pi \in \Psi\ind{t}}	\dbar^{\pi}(s_t)	\left(	\max_{a\in \cA}	Q_t^{\pihat\ind{\i,t+1}}(s_t,a;\i) -Q^{\pihat\ind{\i,t+1}}_{t}(s_t, \pihat\ind{\i,t}(s_t);\i)  \right) \leq 2S A\veps_\stat', \quad \forall s_t \in \cS^{+}_t. \label{eq:flight0}
	\end{align}
	On the other hand, for any $s_t \notin \cS^{+}_t$, we have \begin{align}\sum_{\pi \in \Psi\ind{t}}	\dbar^{\pi}(s_t) \max_{a\in\cA} Q_t^{\pihat\ind{\i,t+1}}(s_t,a;\i)\leq  \sum_{\pi \in \Psi\ind{t}}\dbar^{\pi}(s_t)\cdot \sum_{a'\in \cA}	 P\ind{\i,t}(\i\mid s_t,a') =0, \nn 
	\end{align}
	by definition of $\cS^{+}_{t}$. This, combined with the fact that $Q_t^{\pihat\ind{\i,t+1}}(\cdot,\cdot;\i)\geq 0$ implies that \cref{eq:flight0} also holds for $s_t\in \cS_t\setminus \cS^{+}_t$. Thus, we have that 
	\begin{align}
		\sum_{\pi \in \Psi\ind{t}}	\dbar^{\pi}(s_t)	\left(	\max_{a\in \cA}	Q_t^{\pihat\ind{\i,t+1}}(s_t,a;\i) - Q^{\pihat\ind{\i,t+1}}_{t}(s_t, \pihat\ind{\i,t}(s_t);\i) \right) \leq 2S A\veps_\stat', \quad \forall s_t \in \cS_t.\nn
	\end{align}
\end{proof}

	\section{Proofs for \bmdps}
        \label{sec:BMDP}
	\subsection{MLE Guarantee for \bmdps}

We now state and prove a guarantee for the minimizer
$(\fhat\ind{t}, \phihat\ind{t})$ of the conditional density
estimation problem in \cref{line:inversenew2} of
\cref{alg:IKDP} under realizability. We first
derive the expression of the Bayes-optimal solution of
this problem (we express this solution as a function of probability measures in the extended BMDP, which will be convenient in the proof of \cref{thm:policycover}).
\begin{lemma}
	\label{lem:bayes}
	Let $h\in [H]$ and $t\in[h-1]$ be given, and
	define 
	\begin{align}
		P\ind{t}_{\bayes}((a,j)\mid s, s') \coloneqq \frac{\Pbar^{\pihat\ind{j,t+1}}[\s_h=s' \mid \s_t=s,\a_t=a]}{ \sum_{a'\in \cA, i\in [S]} \Pbar^{ \pihat\ind{i,t+1}}[\s_h=s' \mid \s_t=s,\a_t=a']}, \label{eq:good}
	\end{align}
with $\pihat\ind{j,t+1}$ as in \cref{alg:IKDP}. Consider the solution to the unconstrained
	maximum problem 
	\begin{align}
		\wtilde P_{\bayes}\ind{t} \in \argmax_{\wtilde P\colon \cX_t \times \cX_h \rightarrow \Delta(\cA\times [S])} \E_{\bi_t\sim \unif([S])}\E^{\unif(\Psi\ind{t})\circ_t \unifa \circ_{t+1}  \pihat\ind{\bi_t,t+1}} \left[ \ln  \wtilde P\ind{t}((\a_t,\bi_t)\mid \x_t, \x_h)\right]. \label{eq:normal}
	\end{align}
	Then, for any $a\in\cA$, $j\in[S]$, $x \in \cX_t$, and $x' \in
	\cX_h$, letting $s = \phi_\star(x)$ and $s' = \phi_\star(x')$, $\wtilde P_{\bayes}\ind{t}$ satisfies
	\begin{align}
		\wtilde P\ind{t}_{\bayes}((a,j)\mid x, x') =
		P\ind{t}_{\bayes}((a,j)\mid s, s').\nn 
	\end{align}
	In addition, $(P\ind{t}_{\bayes}, \phi_\star)$ is the Bayes-optimal
	solution to the maximum likelihood problem in \cref{line:inversenew2} of \cref{alg:IKDP}; that is,
	\begin{align}
		( P_{\bayes}\ind{t}, \phi_\star) \in \argmax_{f\colon [S]^2 \rightarrow \Delta(\cA\times [S]), \phi \in \Phi} \E_{\bi_t\sim \unif([S])}\E^{\unif(\Psi\ind{t})\circ_t \unifa \circ_{t+1}  \pihat\ind{\bi_t,t+1}} \left[ \ln   f((\a_t,\bi_t)\mid \phi(\x_t), \phi(\x_h))\right]. \nn 
	\end{align}
\end{lemma}
\begin{proof}[\pfref{lem:bayes}]
	Fix $a\in\cA$, $j\in[S]$, and $(x,x')\in \cX_{t}\times
	\cX_{h}$. Further, let $\bi_t \sim \unif([S])$ and $\bbQ$ denote the law over $(\a_t, \bi_t,\x_t,\x_h)$ induced by first sampling $\mb{j}_t\sim\unif(\brk{S})$ then executing $\unif(\Psi\ind{t})\circ_t\unifa \circ_{t+1} \pihat\ind{\bi_t,t+1}$. With this, the solution $\wtilde P_{\bayes}\ind{t}$ of the problem in \cref{eq:normal} satisfies	
	\begin{align}
	\wtilde P_{\bayes}\ind{t} ((a,j)\mid x,x') & =  	 \bbQ[\a_t=a, \bi_t=j \mid \x_t=x,\x_h=x'], \nn \\
		& =  \frac{	\bbQ[\x_h=x' \mid \x_t=x,\a_t=a, \bi_t=j] \cdot \bbQ[\a_t=a, \bi_t=j\mid \x_t=x]}{\sum_{a'\in \cA}\sum_{i\in[S]} \bbQ[\x_h=x' \mid \x_t=x,\a_t=a',\bi_t=i] \cdot \bbQ[\a_t=a', \bi_t=i\mid \x_t=x]}, \label{eq:pass}
	\end{align} 
	where the last equality follows by Bayes Theorem; in particular  
	\begin{align*}
		\mu [A\mid B, C]=\frac{\mu[B\mid A, C]\cdot \mu[A\mid C]}{\mu[B\mid C]},
	\end{align*} applied with $A=\{\a_t =a, \bi_t=j\}$,
	$B= \{\x_h =x' \}$, $C= \{\x_t = x\}$, and $\mu\equiv
	\bbQ$. Now, by combining
	\cref{eq:pass} with the fact that $\bbQ[\x_h=x' \mid \x_t=x,\a_t=a,
	\bi_t=j] =  	\P^{\pihat\ind{j,t+1}}[\x_h=x' \mid
	\x_t=x,\a_t=a]$ (note that the \rhs is well-defined since $\pihat\ind{j,t+1}\in \Pinm^{t+1:h-1}$), and using that $(\a_t,\bi_t)$ is independent of $\x_t$, we get that 
	\begin{align}
		\wtilde P_{\bayes}\ind{t} ((a,j)\mid x,x')	& = \frac{\P^{\pihat\ind{j,t+1}}[\x_h=x' \mid \x_t=x,\a_t=a] \cdot \bbQ[\a_t=a, \bi_t=j]}{\sum_{a'\in \cA}\sum_{i\in[S]} \P^{\pihat\ind{i,t+1}}[\x_h=x' \mid \x_t=x,\a_t=a'] \cdot \bbQ[\a_t=a',\bi_t=i]}, \nn \\ 
		& = \frac{\P^{\pihat\ind{j,t+1}}[\x_h=x' \mid \x_t=x,\a_t=a] }{\sum_{a'\in \cA}\sum_{i\in[S]} \P^{\pihat\ind{i,t+1}}[\x_h=x' \mid \x_t=x,\a_t=a']}, \label{eq:bayes2}
	\end{align}
	where \cref{eq:bayes2} follows by the fact that $\bbQ[\a_t=a',\bi_t=i']= \bbQ[\a_t=a'',\bi_t=i'']$, for $i',i''\in[S], a',a''\in\cA$.

	Now, since the partial policies $\pihat\ind{i,t+1}$, $i\in[S]$, never take the terminal action, we have $\P^{\pihat\ind{i,t+1}}[\x_h=x' \mid \x_t=x,\a_t=a]=\Pbar^{\pihat\ind{i,t+1}}[\x_h=x' \mid \x_t=x,\a_t=a]$ (the \lhs has $\P$ while the \rhs has $\Pbar$), for all $i\in[S]$. This, together with \cref{eq:bayes2} implies \begin{align}\wtilde P\ind{t}_{\bayes}((a,j)\mid x,x') =  \frac{\Pbar^{\pihat\ind{j,t+1}}[\x_h=x' \mid \x_t=x,\a_t=a]}{\sum_{a'\in \cA,i\in[S]} \Pbar^{\pihat\ind{i,t+1}}[\x_h=x' \mid \x_t=x,\a_t=a']}. \label{eq:ratio}
	\end{align}
	Note that since the outputs of the (potentially non-Markovian) partial policies
	$\{\pihat\ind{i,t+1}, i\in[S]\} \subseteq \Pinm^{t+1:h-1}$ depend only on $\x_{t+1:h-1}$, and not on $\x_{1:t}$, we have 
	\begin{align}&\Pbar^{\pihat\ind{j,t+1}}[\x_h=x' \mid \x_t=x,\a_t=a]\nn \\ &  = \sum_{x''\in \cX_{t+1}} \Pbar^{\pihat\ind{j,t+1}}[\x_h=x' \mid \x_{t+1}=x'', \x_t=x,\a_t=a]\cdot  \Pbar^{\pihat\ind{j,t+1}}[\x_{t+1}=x'' \mid  \x_t=x,\a_t=a],\nn \\ &= \sum_{x''\in \cX_{t+1}} \Pbar^{\pihat\ind{j,t+1}}[\x_h=x' \mid \x_{t+1}=x'', \x_t=x,\a_t=a]\cdot \Pbar[\x_{t+1}=x'' \mid  \s_t=\phi_\star(x),\a_t=a],\nn \\ & =\sum_{x''\in \cX_{t+1}} \Pbar^{\pihat\ind{j,t+1}}[\x_h=x' \mid \x_{t+1}=x'']\cdot \Pbar[\x_{t+1}=x'' \mid  \s_t=\phi_\star(x),\a_t=a], \ \ (\text{since} \ \ \pihat\ind{j,t+1} \in\Pinm^{t+1:h-1}) \nn \\ & =\Pbar^{\pihat\ind{j,t+1}}[\x_h=x' \mid \s_t=\phi_\star(x),\a_t=a],\nn \\   &=q(x'\mid \phi_\star(x'))  \cdot \Pbar^{\pihat\ind{j,t+1}}[\s_h=\phi_\star(x') \mid \s_t=\phi_\star(x),\a_t=a].\nn 
	\end{align}
	This, together with \cref{eq:ratio} implies that $\wtilde
	P\ind{t}_{\bayes}(\cdot \mid x,x')\equiv P\ind{t}_{\bayes}(\cdot \mid
	\phi_\star(x),\phi_{\star}(x'))$ (after canceling the terms involving
	$q(x'\mid \phi_\star(x'))$), where $P\ind{t}_{\bayes}$ is as in
	\cref{eq:good}. Now that we have established \cref{eq:good}, we show
	the second claim of the lemma. The population version of the problem in \cref{line:inversenew2} of \cref{alg:IKDP} becomes equivalent to the following optimization problem: 
	\begin{align}
		\max_{f\colon [S]^2 \rightarrow \Delta(\cA\times [S]), \phi \in \Phi} \E_{\bi_t\sim \unif([S])}\E^{\unif(\Psi\ind{t})\circ_t \unifa \circ_{t+1}  \pihat\ind{\bi_t,t+1}} \left[ \ln  f((\a_t,\bi_t)\mid \phi(\x_t), \phi(\x_h))\right]. \label{eq:abnormal}
	\end{align}
	Note that the value of this problem is always at least that of \cref{eq:normal}. On the other hand, by \cref{eq:good}, the value of the objective in \cref{eq:abnormal} with the pair $(f,\phi)=(P\ind{t}_\bayes, \phi_\star)$ matches the optimal value of the problem in \cref{eq:normal}, and so $( P\ind{t}_{\bayes}, \phi_\star)$ is indeed a solution of \cref{eq:abnormal}. 
\end{proof}
\begin{lemma}[MLE guarantee]
	\label{lem:reg}
	Let $n\geq 1$ and $\delta\in(0,1)$, and define $\veps_\stat(n,\delta)\coloneqq n^{-1/2}  \sqrt{S^3 A\ln n + \ln (|\Phi|/\delta)}$. Further, let $1\leq t<h\leq H$ and suppose that $\Phi$ satisfies \cref{assum:real} and that the policies in $ \Psi\ind{t}$ never take the terminal action $\afrak$. Then, there exists an absolute constant $C>0$ (independent of $t,h$, and other problem parameters) such that the MLE $(\fhat\ind{t},\phihat\ind{t})$ of the conditional density estimation problem in \cref{line:inversenew2} of \cref{alg:IKDP} satisfies with probability at least $1-\delta$,
	\begin{align}
	&\Ebar_{\bi_t\sim \unif([S])}	\Ebar^{\unif(\Psi\ind{t})\circ_t \unifa \circ_{t+1} \pihat\ind{\bi_t,t+1}}\left[\sum_{a\in\cA,j \in[S]}\left( \fhat\ind{t}((a,j)\mid \phihat\ind{t}(\x_t), \phihat\ind{t}(\x_h))	-  P\ind{t}_{\bayes}((a,j) \mid \s_t,\s_h) \right)^2 \right]\nn \\ & \quad \leq   C^2 \cdot \veps^2_\stat(n,\delta).\nn 
	\end{align}
	where $P\ind{t}_{\bayes}$ is as in \cref{lem:bayes}.
\end{lemma}
\begin{proof}[\pfref{lem:reg}]
	Fix $t\in[h-1]$. By \cref{lem:bayes}, $(P\ind{t}_{\bayes},\phi_\star)$ is the Bayes-optimal solution of the conditional density estimation problem in \cref{line:inversenew2} of \cref{alg:IKDP}. And so, by \cref{assum:real} and a standard MLE guarantee for log-loss conditional density estimation (see e.g.~\citet[Proposition E.2]{chen2022partially}),\footnote{\label{fn:hellinger}Technically, \citep[Proposition E.2]{chen2022partially} bounds the Hellinger distance, which immediately implies a bound on MSE.} there exists an absolute constant $C'>0$ (independent of $t$, $h$, and other problem parameters) such that with probability at least $1-\delta$, 
	\begin{align}
&\E_{\bi_t\sim \unif([S])}	\E^{\unif(\Psi\ind{t})\circ_t \unifa \circ_{t+1} \pihat\ind{\bi_t,t+1}}\left[ \sum_{a\in\cA,j \in[S]}\left( \fhat\ind{t}((a,j)\mid \phihat\ind{t}(\x_t), \phihat\ind{t}(\x_h))	-  P\ind{t}_{\bayes}((a,j) \mid \s_t,\s_h) \right)^2 \right]\nn \\ & \quad  \leq   \tilde\veps^2_\stat(n,\delta), \label{eq:theoen}
	\end{align}
	where $\tilde \veps^2_\stat(n,\delta)\coloneqq C' \ln   \cN_{\cF}(1/n)  + C' \ln (|\Phi|/\delta)$ and $ \cN_{\cF}(\veps)$ denotes the $\veps$-covering number of the set $\cF \coloneqq \{f\colon [S]^2 \rightarrow \Delta([S]\times \cA)\}$ in $\ell_{\infty}$-distance. It is easy to verify that $\cN_{\cF}(1/n)\leq n^{A S^3}$, and so by setting $C^2\coloneqq C'$, we have 
	\begin{align}
		\tilde \veps_{\stat}^2(n,\delta) \leq C^2\cdot\veps_{\stat}^2(n,\delta). \label{eq:bound}
	\end{align} 
	Now, since $\afrak$ is never taken by the partial policies $(\pihat\ind{i,\tau})_{i\in[S],\tau\in[h-1]}$, in \cref{alg:IKDP} or by the policies in $\Psi\ind{2}, \dots, \Psi\ind{h-1}$ (by assumption), the guarantee in \cref{eq:theoen} also holds in $\Mbar$. Combining this with \cref{eq:bound} completes the proof.
\end{proof}

\subsection{Proof of \creftitle{thm:policycover} (Main Guarantee for \musik)}
\label{app:geniklemmaproof}
\begin{proof}[{Proof of \cref{thm:policycover}}]
	Let $\epsilon \coloneqq \veps/(2S)$. Let $\veps_\stat(\cdot, \cdot)$ and $C$ be as in \cref{thm:geniklemma} (note that $C$ is an absolute constant independent of all problem parameters). Let $\cE_h$ be the success event of \cref{thm:geniklemma} for $h\in[H]$ and $\epsilon$, and define $\cE\coloneqq \bigcap_{h\in[H]}\cE_h$. Note that by the union bound we have $\P[\cE]\geq 1 -\delta$. For $n$ large enough such that $8 AS^4 H C \veps_{\stat}(n, \frac{\delta}{H^2})\leq \epsilon$ (which is implied by the condition on $n$ in the theorem's statement for $c=2^5 C$), \cref{thm:geniklemma} implies that under $\cE$, the output $\Psi\ind{1},\dots,\Psi\ind{H}$ of $\musik$ are $(1/2,\epsilon)$-policy covers relative to $\Pibar_{\epsilon}$ in $\wbar \cM$ for layers 1 to $H$, respectively. Thus, by \cref{lem:transfer}, the desired result holds under $\cE$.

	The parameter $n$ in \cref{thm:policycover} represents the input to
	$\musik$ used to generate an approximate $(1/4,\veps)$-policy
	cover relative to $\Pim$ in $\cM$ at all layers. $\musik$
	passes $n$ to all of $\ikdp$ invocations (see \cref{line:ikdp1}
	of \cref{alg:GenIk}). Since, for any layer $h\in[H]$, the
	corresponding $\ikdp$ instance in $\musik$ requires $n$
	trajectories for each layer $t\in[h-1]$ (see
	\cref{line:mainiter} of \cref{alg:IKDP}), the total number of
	trajectories needed by $\musik$ in the setting of
	\cref{thm:policycover} is  
	\begin{align}
		\text{\# of trajectories} =  \wtilde{O}(1)\cdot  \frac{A^2 S^{10} H^4 \left( A S^3 + \ln (|\Phi| H^2/\delta)\right)}{\veps^2}.\nn %
	\end{align}
\end{proof}

\subsection{Proof of \creftitle{thm:geniklemma} (Main Guarantee for \ikdp)}
Before proving \cref{thm:geniklemma}, we first define the $Q$- and $V$-functions corresponding to certain `fictitious' rewards we introduce for the analysis. These functions will be instrumental in our proofs. We then present a generalized performance difference lemma that holds for the non-Markovian partial policies of $\musik$ (since the policies are non-Markovian the standard performance difference lemma does not give us something useful for the proof of \cref{thm:geniklemma}). In \cref{sec:errorbound}, we bound the errors appearing on the RHS of our generalized performance difference lemma. Finally, we present the proof of \cref{thm:geniklemma} in \cref{sec:proofbmdp}

\subsubsection{The $Q$- and $V$-functions}
For $t,h\in[H]$, $a\in \cA$, and $s'\in \cS_h$, define $r_t(\cdot ; s') \colon \cX_t\rightarrow \{0,1\}$ as 
\begin{align}
r_t(x;s')=\mathbf{1}\{\phi_\star(x)=s'\}. \label{eq:r}
\end{align}
This can be interpreted as a reward function that takes value $1$ whenever the latent state is $s'$. %
For $t\in[H]$ and any two partial policies $\pi\ind{t}\in \Pinm^{t:h-1}$ and $ \pi\ind{t+1}\in \Pinm^{t+1:h-1}$, we define the corresponding $t$th layer $Q$- and $V$-functions in $\Mbar$ as 
\begin{align}
	Q^{\pi\ind{t+1}}_t(x,a;s') &\coloneqq r_t(x;s') +\Ebar^{\pi\ind{t+1}}\left[ \left. \sum_{\tau=t+1}^h r_\tau(\x_{\tau}; s')\ \right| \ \x_t=x,\a_t = a\right], \label{eq:Q} \\  \text{and}\ \  V^{\pi\ind{t}}_t(x;s') &  \coloneqq  r_t(x;s') +\Ebar^{\pi\ind{t}}\left[ \left. \sum_{\tau=t+1}^h r_\tau(\x_{\tau}; s')\ \right| \ \x_t=x,\a_t = \pi\ind{t}(x)\right]. \label{eq:V}
	\end{align}
These match the standard definitions of the $Q$- and $V$- functions for Markovian policies, albeit with action-independent rewards. Note that we only define $Q^{\pi\ind{t+1}}_\tau$ and $V^{\pi\ind{t}}_\tau$ for $\tau=t$, as the policies involved are non-Markovian, and are undefined on layers $\tau<t$. 

\paragraph{Useful properties of the $Q$- and $V$-functions} Given the definition of the rewards in \eqref{eq:r}, the $Q$-function in \eqref{eq:Q} can be expressed in terms of certain conditional probabilities for visiting latent states, which will be useful throughout the proof.
\begin{lemma}
	\label{fact:Q}
	For any $s'\in \cS_h$, $x\in \cX_t$, $a\in \cA$, and $\pi\ind{t+1}\in \Pinm^{t+1,h-1}$, we have  \begin{align}Q_t^{\pi\ind{t+1}}(x,a;s')= \Pbar^{\pi\ind{t+1}}[\s_h=s'\mid \s_t=\phi_\star(x),\a_t=a].\label{eq:Qprop}
	\end{align}
	\end{lemma}
For $\ikdp$'s partial policies $\{\pihat\ind{i,\tau} \colon i\in [S]\} \subseteq  \Pinm^{\tau:h-1}$, for $\tau \in[h-1]$, the corresponding $V$-functions satisfy an identity similar to \eqref{eq:Qprop}.
\begin{lemma}
	\label{fact:V}
		For any $i\in[S]$, $s'\in \cS_h$, $x\in \cX_t$, and $a\in \cA$, we have  \begin{align}V_t^{\pihat\ind{i,t}}(x;s') = \Pbar^{\pihat\ind{\iotahat(x),t+1}}[\s_h=s' \mid \s_t=\phi_\star(x), \a_t =\ahat(x)],\nn
		\end{align} where $(\ahat(x),\iotahat(x))\in \argmax_{(a',i')} \hat f\ind{t}((a',i') \mid \hat \phi\ind{t}(x),i)$ and $(\hat f\ind{t}, \hat \phi\ind{t})$ are defined as in \cref{line:inversenew2} of \cref{alg:IKDP}. 
              \end{lemma}
             Note that unlike the $Q$-function in \cref{fact:Q}, it is not the case that the $V$-function in \cref{fact:V} depends on $x\in\cX_t$ only through $\phistar(x)$.
\begin{proof}[Proof of \cref{fact:Q}]
By definition of the reward functions, we have that $r_t(\cdot,s')\equiv 0$, for all $t<h$, and
\begin{align}
	\sum_{\tau =t+1}^h r_{\tau}(\x_\tau;s') = r_h(\x_h;s')= \mathbb{I}\{\phi_\star(\x_h)=s'\}.\nn 
	\end{align}
Therefore, 
\begin{align}
	Q^{\pi\ind{t+1}}_t(x,a;s') & = r_t(x;s') +\Ebar^{\pi\ind{t+1}}\left[ \left. \sum_{\tau=t+1}^h r_\tau(\x_{\tau}; s')\ \right| \ \x_t=x,\a_t = a\right],\nn \\
	& = 0 + \Ebar^{\pi\ind{t+1}}\left[ \left.  r_h(\x_{h}; s')\ \right| \ \x_t=x,\a_t = a\right],\nn \\
	& = \Pbar^{\pi\ind{t+1}}[\s_h=s'\mid \s_t=\phi_\star(x),\a_t=a],\nn 
	\end{align}
where the last equality follows from the fact that, while $\pi\ind{t+1}\in \Pinm^{t+1:h-1}$ is non-Markovian, it only depends on the observations at layers $t+1$ to $h-1$.
	\end{proof}
\begin{proof}[Proof of \cref{fact:V}]
	By definition of the reward functions, we have that $r_t(\cdot,s')\equiv 0$, for all $t<h$, and
	\begin{align}
		\sum_{\tau =t+1}^h r_{\tau}(\x_\tau;s') = r_h(\x_h;s')= \mathbb{I}\{\phi_\star(\x_h)=s'\}.\label{eq:allzeros}
	\end{align}
	Therefore, 
	\begin{align}
		V^{\pihat\ind{i,t}}_t(x;s') & = r_t(x;s') +\Ebar^{\pihat\ind{i,t}}\left[ \left. \sum_{\tau=t+1}^h r_\tau(\x_{\tau}; s')\ \right| \ \x_t=x,\a_t = \pihat\ind{i,t}(x)\right],\nn \\
		& = 0 + \Ebar^{\pihat\ind{i,t}}\left[ \left.  r_h(\x_{h}; s')\ \right| \ \x_t=x,\a_t = \pihat\ind{i,t}(x)\right],\quad \text{(by \eqref{eq:allzeros})} \nn \\
		& =  \Ebar^{\pihat\ind{i,t} \circ_{t+1} \pihat\ind{\iotahat(x),t+1}}\left[ \left.  r_h(\x_{h}; s')\ \right| \ \x_t=x,\a_t = \ahat(x)\right],\label{eq:middle} \\
		& =  \Ebar^{\pihat\ind{\iotahat(x),t+1}}\left[ \left.  r_h(\x_{h}; s')\ \right| \ \x_t=x,\a_t = \ahat(x)\right],\quad \text{(since $\pihat\ind{j,t+1}\in \Pinm^{t+1,h-1}, \forall j$)}\nn  \\
		& = \Pbar^{\pihat\ind{\iotahat(x),t+1}}[\s_h=s'\mid \s_t=\phi_\star(x),\a_t=\ahat(x)],\nn\quad \text{(by \eqref{eq:r} and $\pihat\ind{j,t+1}\in \Pinm^{t+1,h-1}, \forall j$)} 
	\end{align}
where \eqref{eq:middle} follows by the definition of $\pihat\ind{i,t}$ in \cref{line:inversenew2} of \cref{alg:IKDP}.
	
\end{proof}

\subsubsection{Generalized Performance Difference Lemma for \musik's Non-Markovian Policies}
\renewcommand{\sh}{\sfrak}
We now present a generalized performance difference lemma that holds for the non-Markovian partial policies used in $\musik$/$\ikdp$. In what follows, we use the convention that for any $\pi \in \Pinm^{l:r}$ and $\tau \in\range{l}{r}$, $\pi(x_{l:\tau})=\afrak$, for any $x_{l:\tau}\in \wbar\cX_l \times \dots \times \wbar\cX_{\tau}$ such that $x_\tau=\tfrak_\tau$ (or equivalently $\phi_\star(x_\tau)=\tfrak_\tau$).\footnote{We recall that we have assumed the state $\tfrak_h$ emits itself as an observation.}

\begin{lemma}
	\label{lem:performancediffbmdp}
	Fix $h\in[H]$, and let us adopt the convention that  $\pihat\ind{i,h}\equiv \unifa, \forall i\in [S]$. The partial policies $\pihat\ind{i,t}$ produced by \ikdp for $i\in [S]$, $t\in[h-1]$ satisfy for any $\sfrak\in\cS_h$, 
	\begin{align}
		\min_{i\in[S] }\Ebar\brk*{V_1^{\pistar\ind{\sfrak}}(\bx_1;\sfrak)
			-V_1^{\pihat\ind{i,1}}(\bx_1;\sfrak)
		}
		\leq{} \sum_{t=1}^{h-1}
		\min_{i\in[S]}\max_{j\in  [S]}\Ebar^{\pistar\ind{\sfrak}}\brk*{\mathbb{I}\{\s_t\in \cS_{t,\epsilon} \}\left(Q_t^{\pihat\ind{j,t+1}}(\bx_t, \pistar\ind{\sfrak}(\x_t);\sfrak)
			-V_t^{\pihat\ind{i,t}}(\bx_t;\sfrak)
			\right)	},\nn 
	\end{align}
	where $\pistar\ind{s} \in \argmax_{\pi\in  \Pibar_{\epsilon}}
	\bar{d}^{\pi}(s)$ and $Q^{\pihat\ind{j,t+1}}_{t}(\cdot;
	\sfrak)$ (resp.~$V^{\pihat\ind{i,t}}_{t}(\cdot;
	\sfrak)$) is defined as in \eqref{eq:Q} (resp.~\eqref{eq:V}) with $\pi\ind{t+1}=\pihat\ind{j,t+1}$ (resp.~$\pi\ind{t}=\pihat\ind{i,t}$). 
\end{lemma}

\begin{proof}[{Proof of \cref{lem:performancediffbmdp}}]
  Let $\sfrak\in \cS_{h}$ be fixed. We proceed by backwards induction to show that for all
	$\tau\in[h-1]$, the learned partial
	policies $\pihat\ind{1,\tau}, \dots,\hat
	\pi\ind{S,\tau}\in \Pinm^{\tau:h-1}$ have the property that
	\begin{align}
		\min_{i\in[S]}	\Ebar^{\pi_\star^{(\sfrak)}} \left[ V^{\pi_\star^{(\sfrak)}}_{\tau}(\x_{\tau};\sfrak)- V^{\pihat\ind{i,\tau}}_{\tau}(\x_{\tau};\sfrak)\right] \leq \Sigma_\tau, \label{eq:new}
	\end{align}
	where $\Sigma_\tau \coloneqq \sum_{t=\tau}^{h-1}
	\min_{i\in[S]}\max_{j\in  [S]}\Ebar^{\pistar\ind{\sfrak}}\brk*{Q_t^{\pihat\ind{j,t+1}}(\bx_t, \pistar\ind{\sfrak}(\x_t);\sfrak)
		-V_t^{\pihat\ind{i,t}}(\bx_t;\sfrak)}.$

	\paragraph{Base case}
	The base case (i.e.~$\tau=h-1$) reduces to showing that for any $\sfrak\in \cS_h$, \begin{align}\Ebar^{\pi_\star^{(\sfrak)}} \left[V^{\pi_\star^{(\sfrak)}}_{h-1}(\x_{h-1};\sfrak) \right]\leq \max_{j\in [S]}\Ebar^{\pistar\ind{\sfrak}}\brk*{Q_{h-1}^{\pihat\ind{j,h}}(\bx_{h-1}, \pistar\ind{\sfrak}(\x_{h-1});\sfrak)}.\nn 
	\end{align}
	This holds with equality regardless of how
	$\crl{\pihat^{(i,h)}}_{i\in [S]}$ are chosen, since the reward function for layer $h$ is independent of the actions taken at that layer.

	\paragraph{Inductive step}
	Now, let $t\in[h-2]$ and suppose that \cref{eq:new} holds for $\tau =t+1$, and we show that it holds for $\tau=t$. Fix $\sh \in \cS_{h,\epsilon}$ and let $\pi_\star \equiv \pi_\star^{(\sh)}$ to simplify notation. Further, fix $\i\in [S]$ and let $\j$ be the minimizer of the \lhs of \cref{eq:new} for $\tau=t+1$. 
	By \cref{lem:bla}, we know that
	$\pi_\star(x_t)=\fraka$ for all $x_t \in \supp \
	q(\cdot\mid s_t)$ and $s_t \in \cS_t
	\setminus \cS_{t,\epsilon}$. Therefore, since the $V$-function at layer $t+1$ is zero on the terminal state $\tfrak_{t+1}$ (see definition of the $V$-function in \eqref{eq:V}), we have 
		\begin{align}
			\Ebar^{\pi_\star}\left[V_{t+1}^{\pihat\ind{\j,t+1}} (\x_{t+1};\sh)\right]&=\sum_{s_t\in \cS_{t,\epsilon}} \bar{d}^{\pi_\star}(s_t)\Ebar^{\pi_\star}\left[ \left.V_{t+1}^{\pihat\ind{\j,t+1}} (x_{t+1};\sh) \right| \s_t = s_t\right],\nn \\
			& = -\sum_{s_t\in \cS_{t,\epsilon}} \bar{d}^{\pi_\star}(s_t) \Ebar_{\x_t \sim q(\cdot \mid s_t)} \left[r_t(\x_t;\sh) \right] \nn \\ &\qquad +\sum_{s_t\in \cS_{t,\epsilon}} \bar{d}^{\pi_\star}(s_t)\Ebar_{\x_t \sim q(\cdot \mid s_t)} \left[Q^{\pihat\ind{\j,t+1}}_t (\x_{t}, \pistar(\x_t);\sh)  \right],\nn \\
			& = \sum_{s_t\in \cS_{t,\epsilon}} \bar{d}^{\pi_\star}(s_t)\Ebar_{\x_t \sim q(\cdot \mid s_t)} \left[ Q^{\pihat\ind{\j,t+1}}_t (\x_{t}, \pistar(\x_t);\sh) \right], \label{eq:gone}
		\end{align}%
		where the last equality follows by the fact that $t<h$ and that the reward $r_t(\cdot;\sh)\equiv 0$ in this case. Combining \cref{eq:gone} with \cref{eq:new}%
		, we have 
		\begin{align}
			&  \sum_{s_t \in \cS_{t,\epsilon}} \bar{d}^{\pi_\star}(s_t)\Ebar_{\x_t \sim q(\cdot \mid s_t)}	\left[Q_t^{\pihat\ind{\j,t+1}}(\x_t,\pi_\star(\x_t);\sh) - V_t^{\pihat\ind{\i,t}}(\x_t;\sh) \right]   \nn \\ & \geq \Ebar^{\pi_\star}\left[V_{t+1}^{\pihat\ind{\j,t+1}} (\x_{t+1};\sh)\right]- \sum_{s_t\in \cS_{t,\epsilon}} \bar{d}^{\pi_\star}(s_t)\Ebar_{\x_t \sim q(\cdot \mid s_t)} \left[ V^{\pihat\ind{\i,t}}_t (\x_{t};\sh) \right],\nn \quad \text{(by \cref{eq:gone})} \\
			& \geq \Ebar^{\pi_\star}[V_{t+1}^{\pi_\star}(\x_{t+1};\sh)] - \sum_{s_t\in \cS_{t,\epsilon}} \bar{d}^{\pi_\star}(s_t)\Ebar_{\x_t \sim q(\cdot \mid s_t)}\left[ V^{\pihat\ind{\i,t}}_t (\x_{t};\sh)\right]- \Sigma_{t+1}, \quad \text{(by induction)} \nn \\
			& = \sum_{s_t\in \cS_{t,\epsilon}} \bar{d}^{\pi_\star}(s_t)\Ebar^{\pi_\star}\left[\left.V_{t+1}^{\pi_\star}(\x_{t+1};\sh) \right|\s_t = s_t \right] - \sum_{s_t\in \cS_{t,\epsilon}} \bar{d}^{\pi_\star}(s_t)\Ebar_{\x_t \sim q(\cdot \mid s_t)} \left[V^{\pihat\ind{\i,t}}_t (\x_{t};\sh) \right]- \Sigma_{t+1},\nn \\
			& = \sum_{s_t\in \cS_{t,\epsilon}} \bar{d}^{\pi_\star}(s_t)\Ebar_{\x_t \sim q(\cdot \mid s_t)}  \left[V_{t}^{\pi_\star}(\x_{t};\sh)   -V^{\pihat\ind{\i,t}}_t (\x_{t};\sh) \right]- \Sigma_{t+1}, \label{eq:newdeal}
		\end{align}
		where in the last step we used that the rewards are zero except at layer $h$ and that $\pi_\star(x_t)=\fraka$ for $x_t \in \supp q(\cdot\mid s_t)$ with $s_t\in \cS_t\setminus \cS_{t,\epsilon}$ (by  \cref{lem:bla}). For such an $x_t$, we also have that $V_{t}^{\pi_\star}(x_{t};\sh)=0$, and so since $V^{\pihat\ind{\i,t}}_t (\cdot;\sh)$ is non-negative, \cref{eq:newdeal} implies that
		\begin{align}
			&		\Ebar^{\pi_\star} \left[ V^{\pi_\star}_t(\x_t;\sh)- V^{\pihat\ind{\i,t}}_t(\x_t;\sh)\right] 	\nn \\
			&	\leq  \sum_{s_t \in \cS_{t,\epsilon}} \bar{d}^{\pi_\star}(s_t)\Ebar_{\x_t \sim q(\cdot \mid s_t)}	 \left[ V^{\pi_\star}_t(\x_t;\sh)- V^{\pihat\ind{\i,t}}_t(\x_t;\sh)\right],\nn \\
			& \leq \Sigma_{t+1} + \sum_{s_t \in \cS_{t,\epsilon}} \bar{d}^{\pi_\star}(s_t)\Ebar_{\x_t \sim q(\cdot \mid s_t)}	\left[Q_t^{\pihat\ind{\j,t+1}}(\x_t,\pi_\star(\x_t);\sh) - V_t^{\pihat\ind{\i,t}}(\x_t;\sh) \right],\nn \\
			& \leq \Sigma_{t+1} + \max_{j\in [S]}\sum_{s_t \in \cS_{t,\epsilon}} \bar{d}^{\pi_\star}(s_t)\Ebar_{\x_t \sim q(\cdot \mid s_t)}	\left[Q_t^{\pihat\ind{j,t+1}}(\x_t,\pi_\star(\x_t);\sh) - V_t^{\pihat\ind{\i,t}}(\x_t;\sh) \right].
			 \label{eq:min}
		\end{align}
		Recall that $\i$ was chosen arbitrarily in $[S]$, and so taking the min over $\i\in[S]$ on both sides of \eqref{eq:min} implies the desired result. 
	\end{proof}
        \subsubsection{Local Error Guarantee}
\renewcommand{\k}{j}
\newcommand{\ilocal}{\j}
\newcommand{\jlocal}{\i}
\newcommand{\shlocal}{\sfrak'}
\newcommand{\stlocal}{\sfrak}
\label{sec:errorbound}
The following lemma, which is a restatement of \cref{lem:propnew}, gives us a way of bounding the error terms appearing on the \rhs of the inequality in  \cref{lem:performancediffbmdp}.
\begin{lemma}[Restatement of \cref{lem:propnew}]
	\label{lem:intermediate}
	Let $\epsilon,\delta \in(0,1)$, $h\in[H]$, and suppose $\Phi$ satisfies \cref{assum:real}. If the policies in $\Psi\ind{2}\cup \dots\cup \Psi\ind{h-1}$ never take the terminal action $\afrak$, then for any $t\in[h-1]$, there is an event $\cE_t$ of probability at least $1-\frac{\delta}{H^2}$ under which the partial policies $\crl*{\pihat\ind{\k,\tau}}_{\k\in[S],\tau \in[h-1] }$, constructed during the call to $\ikdp(\Psi\ind{1},\dots,\Psi\ind{h-1},\Phi,n)$ are such that for any $s_h \in \cS_{h}$ there exists $\jlocal\in[S]$ that satisfies 
		\begin{align}
0	\leq 	\sum_{\pi \in \Psi\ind{t}}	\bar{d}^{\pi}(s_t) \Ebar_{\x_t \sim q(\cdot \mid s_t)}	\left[	\max_{a\in \cA, \k \in[S]}	Q_t^{\pihat\ind{\k,t+1}}(\x_t,a;s_h) - V_t^{\pihat\ind{\jlocal,t}}(\x_t;s_h) \right] \leq 2S^3 A C\veps_\stat(n ,\tfrac{\delta}{H^2}), \ \  \forall s_t \in \cS_t,\label{eq:flight2}
	\end{align}
where $\veps_\stat(\cdot,\cdot)$ and $C>0$ are as in \cref{lem:reg}; here $C>0$ is an absolute constant independent of problem parameters.
	\end{lemma}

\begin{proof}[{Proof of \cref{lem:intermediate}}]
	To simplify notation throughout the proof, let 
	\begin{align}
	P\ind{t}(s_h\mid s_t,a) \coloneqq \frac{1}{S}\sum_{\k\in[S]}\Pbar^{ \pihat\ind{\k,t+1}}[\s_h=s_h\mid \s_t=s_t,\a_t=a].\label{eq:Pbar}
		\end{align}
By \cref{lem:reg} and Jensen's inequality, we have
that with probability at least $1-\delta/H^2$,
the solution $(\fhat\ind{t},\phihat\ind{t})$ of
the conditional density estimation problem in \cref{line:inversenew2} of \cref{alg:IKDP}
satisfies,         
\begin{align}
	\Ebar_{\s_t\sim \unif(\Psi\ind{t}),\x_t\sim q(\cdot \mid \s_t),\a' \sim \unifa} \left[ \sum_{s_h\in \cS_h,x_h}P\ind{t}(s_h\mid \s_t,\a')  q(x_h \mid s_h)  \max_{a\in \cA,\k \in[S]}\left| \err\ind{t}(a,\k,\x_t,x_h)  \right| \right]\leq   \veps_\stat' , \label{eq:neverbefore}
\end{align}
where $\veps_{\stat}'\coloneqq
C\cdot \veps_\stat(n,\frac{\delta}{H^2})$, $C$ is an absolute constant independent of $t,h$, and other problem parameters,
\[\err\ind{t}(a,\k,x_t,x_h)  \coloneqq \fhat\ind{t}((a,\k)\mid \phihat\ind{t}(x_t), \phihat\ind{t}(x_h))	-  {P}\ind{t}_{\bayes}((a,\k) \mid \phi_\star(x_t),\phi_\star(x_h)),\]
and finally
\begin{align}
	{P}\ind{t}_{\bayes}((a,\k) \mid s_t,s_h) \coloneqq \frac{\Pbar^{\pihat\ind{\k,t+1}}[\s_h=s_h \mid \s_t=s_t,\a_t=a]}{\sum_{a'\in \cA,i\in [S]}\Pbar^{ \pihat\ind{i,t+1}}[\s_h=s_h \mid \s_t=s_t,\a_t=a']}. \label{eq:bayes}
\end{align}
We denote this event by $\cE_t$. Note that to rewrite the result of \pref{lem:reg} as \eqref{eq:neverbefore}, we use that the policies $ \pihat\ind{\k,t+1}$, $\k \in [S]$, while non-Markovian, only depend on $\x_{t+1},\ldots,\x_{h-1}$. Moving forward, we condition on $\cE_t$. 

Fix $\shlocal\in \cS_{h}$ and let $\jlocal \in \argmax_{i\in[S]} \sum_{x_h}
\mathbb{I}\{\phihat\ind{t}(x_h)=i\}   q(x_h\mid
\shlocal)$, and note that \begin{align}\sum_{x_h}
\mathbb{I}\{\phihat\ind{t}(x_h)=\jlocal\}   q(x_h\mid
\shlocal)\geq \frac{1}{S}. \label{eq:propty}\end{align} 
Further, let $\cS^{+}_t$ be
the subset of states defined by 
\begin{align} &\cS^{+}_t  \coloneqq \left\{\tilde s \in \cS_t   \colon \sum_{\pi \in \Psi\ind{t}} \bar{d}^{\pi}(\tilde s)\sum_{a\in \cA} P\ind{t}(\shlocal \mid \tilde s,a) > 0 \right\}.\nn 
\end{align}%
Now, fix $\stlocal\in \cS^+_t$. From \cref{eq:neverbefore}, we
have that 
\begin{align}
	&\sum_{\pi \in \Psi\ind{t}, x_t,x_h,a'\in \cA}\bar{d}^{\pi}(\stlocal)	q(x_t \mid \stlocal)  P\ind{t}(\shlocal\mid \stlocal,a')  q(x_h\mid \shlocal) \max_{a\in \cA, \k \in[S]}\left| \err\ind{t}(a,\k ,x_t,x_h) \right|     \notag \\
	&\leq\sum_{\pi \in \Psi\ind{t}, x_t,x_h,a'\in \cA, s_t\in \cS_t,s_h\in\cS_h}\bar{d}^{\pi}(s_t)	q(x_t \mid s_t)
	P\ind{t}(s_h\mid s_t,a')  q(x_h\mid s_h) \max_{a\in
		\cA, \k \in [S]}\left| \err\ind{t}(a,\k ,x_t,x_h)
	\right|,\notag\\
	&\leq   {S A\veps_\stat'}. \label{eq:never}
\end{align}
Applying \cref{eq:propty} within \cref{eq:never} implies that
\begin{align}
	& 	 {SA\veps_\stat'} \nn \\ 
	& \geq   \sum_{\substack{\pi \in \Psi\ind{t}, x_t,a'\in \cA,\\ x_h: \phihat\ind{t}(x_h)=\jlocal}}\bar{d}^{\pi}(\stlocal)	q(x_t \mid \stlocal) P\ind{t}(\shlocal\mid \stlocal,a')  q(x_h\mid \shlocal)
	\max_{a\in \cA, \k \in [S]}\left| \fhat\ind{t}((a,\k)\mid\phihat\ind{t}(x_t), \jlocal)	-  P\ind{t}_{\bayes}((a,\k) \mid \stlocal,\shlocal)\right|, \nn \\
	& \geq   \frac{1}{S}\sum_{\pi \in \Psi\ind{t}, x_t,a'\in \cA}\bar{d}^{\pi}(\stlocal)	q(x_t \mid \stlocal) P\ind{t}(\shlocal\mid \stlocal,a')
	\max_{a\in \cA, \k \in [S]}\left| \fhat\ind{t}((a,\k)\mid \phihat\ind{t}(x_t), \jlocal)	- P\ind{t}_{\bayes}((a,\k) \mid \stlocal,\shlocal)\right|.\nn
\end{align} 
By rearranging and using that $\sum_{\pi \in \Psi\ind{t}}\bar{d}^{\pi}(\stlocal)	 \sum_{a'\in \cA} P\ind{t}(\shlocal\mid \stlocal,a')>0$ (since $\stlocal\in \cS^+_t$), we get
\begin{align}
	\Ebar_{\x_t\sim q(\cdot \mid \stlocal)} \left[
	\max_{a\in \cA,\k \in[S]}\left| \fhat\ind{t}((a,\k)\mid \phihat\ind{t}(\x_t), \jlocal)	-  P\ind{t}_{\bayes}((a,\k) \mid \stlocal,\shlocal)\right| \right]\leq  \frac{S^2 A\veps_\stat'}{\sum_{\pi \in \Psi\ind{t}}\bar{d}^{\pi}(\stlocal)	\sum_{a'\in \cA} P\ind{t}(\shlocal\mid \stlocal,a')}. \label{eq:dedant}
\end{align}
Now, let $\ahat\ind{\jlocal,t}(x_t), \iotahat\ind{\jlocal,t}(x_t) \in \argmax_{a\in \cA,
	\k \in [S]} \fhat\ind{t}((a,\k)\mid \phihat\ind{t}(x_t), \jlocal)$
and note that $\ahat\ind{\jlocal,t}(x_t) = \pihat\ind{\jlocal, t}(x_t)$,
where $\pihat\ind{\jlocal, t}(x_t)$ is defined as in
\cref{alg:IKDP}. With this, \cref{eq:dedant}, and the
fact that $|\|y\|_{\infty}-\|z\|_{\infty}|\leq
\|y-z\|_{\infty}$, for all $y,z\in \reals^{A\times S}$
we have 
\begin{align}
& 	\max_{a\in \cA,\k\in[S]}  P\ind{t}_{\bayes}((a,\k) \mid
	\stlocal, \shlocal) \nn \\ & \leq \Ebar_{\x_t \sim q(\cdot \mid \stlocal)}\left[\max_{a\in\cA,\k\in\brk{S}}\fhat\ind{t}((a,\k)\mid \phihat\ind{t}(\x_t), \jlocal) \right] +\frac{S^2 A\veps_\stat'}{\sum_{\pi \in \Psi\ind{t},a'\in \cA}\bar{d}^{\pi}(\stlocal)	 P\ind{t}(\shlocal\mid \stlocal,a')}, \nn \\
	& = \Ebar_{\x_t \sim q(\cdot \mid \stlocal)}\left[\fhat\ind{t}((\ahat\ind{\jlocal,t}(\x_t),\iotahat\ind{\jlocal,t}(\x_t))\mid \phihat\ind{t}(\x_t), \jlocal) \right] +\frac{S^2 A\veps_\stat'}{\sum_{\pi \in \Psi\ind{t},a'\in \cA}\bar{d}^{\pi}(\stlocal)	 P\ind{t}(\shlocal\mid \stlocal,a')}, \nn \\
	& \leq   \Ebar_{\x_t \sim q(\cdot \mid \stlocal)} \left[ P\ind{t}_{\bayes}((\ahat\ind{\jlocal,t}(\x_t),\iotahat\ind{\jlocal,t}(\x_t)) \mid  \stlocal, \shlocal) \right] +  \frac{2S^2 A\veps_\stat'}{\sum_{\pi \in \Psi\ind{t}}\bar{d}^{\pi}(\stlocal)\sum_{a'\in \cA}	 P\ind{t}(\shlocal\mid \stlocal,a')}. \label{eq:bayes_bound_max}
\end{align}
Now, observe that from the definition~of $P\ind{t}_{\bayes}$ in \cref{eq:bayes}, we have that for all $s_t \in \cS_t$, $a\in\cA$, $\k\in[S]$, and $x_t\in\supp q(\cdot\mid{}s_t)$,
\begin{align}
	Q_t^{\pihat\ind{\k,t+1}}(x_t,a;\shlocal)= \Pbar^{\pihat\ind{\k,t+1}}[\s_h=\shlocal \mid \s_t=s_t,\a_t=a]  = P\ind{t}_{\bayes}((a,\k) \mid s_t, \shlocal)\sum_{a'\in \cA}S\cdot{}P\ind{t}(\shlocal\mid s_t,a'),\label{eq:pbayes_pbar}
\end{align}
where we have used that $\sum_{i\in[S],a'\in\cA}\Pbar^{ \pihat\ind{i,t+1}}[\s_h=\shlocal\mid \s_t=s_t,\a_t=a']=S\cdot{}\sum_{a'\in\cA}P\ind{t}(\shlocal\mid s_t,a')$ by definition of $P\ind{t}$ in \eqref{eq:Pbar}. %
Combining this with \pref{eq:bayes_bound_max}, we have 
\begin{align}
&  \Ebar_{\x_t \sim q(\cdot \mid \stlocal)} \left[\max_{a\in \cA,\k\in[S]}	Q_t^{\pihat\ind{\k,t+1}}(\x_t,a;\shlocal) \right]\nn \\
	&  \leq  \Ebar_{\x_t \sim q(\cdot \mid \stlocal)}\brk*{  P\ind{t}_{\bayes}((\ahat\ind{\jlocal,t}(\x_t),\iotahat\ind{\jlocal,t}(\x_t)) \mid \stlocal, \shlocal) }\cdot \sum_{a'\in \cA}S\cdot{}	P\ind{t}(\shlocal\mid \stlocal,a')+ \frac{2 S^3 A\veps_\stat'}{\sum_{\pi \in \Psi\ind{t}}\bar{d}^{\pi}(\stlocal)}, \ \ \text{(by \eqref{eq:pbayes_pbar} \& \eqref{eq:bayes_bound_max})} \nn \\ 
	& =   \Ebar_{\x_t \sim q(\cdot \mid \stlocal)} \brk*{ \Pbar^{\pihat\ind{\iotahat(\x_t),t+1}}[\s_h=\shlocal \mid \s_t=\stlocal, \a_t =\ahat\ind{\jlocal,t}(\x_t)]} + \frac{2S^3 A\veps_\stat'}{\sum_{\pi \in \Psi\ind{t}}\bar{d}^{\pi}(\stlocal)}, \ \ \text{(where $\iotahat(x)\coloneqq  \iotahat\ind{\jlocal,t}(x)$)} \nn  \\
	& = \Ebar_{\x_t\sim q(\cdot\mid \stlocal)}\left[V_t^{\pihat\ind{\jlocal,t}}(\x_t;\shlocal)\right]+ \frac{2S^3 A\veps_\stat'}{\sum_{\pi \in \Psi\ind{t}}\bar{d}^{\pi}(\stlocal)},\nn 
\end{align}
where the first equality uses \pref{eq:pbayes_pbar}
once more and the second equality follows from \cref{fact:V} and the
definition of $\pihat\ind{\jlocal,t}$ in
\cref{alg:IKDP}.
Summarizing, we have shown that
\begin{align}
\sum_{\pi \in \Psi\ind{t}}	\bar{d}^{\pi}(s_t)  \Ebar_{\x_t \sim q(\cdot \mid s_t)}	\left[	\max_{a\in \cA,\k\in[S]}	Q_t^{\pihat\ind{\k,t+1}}(\x_t,a;\shlocal) - V_t^{\pihat\ind{\jlocal,t}}(\x_t;\shlocal) \right] \leq 2S^3 A\veps_\stat', \quad \forall s_t\in \cS^{+}_t. \label{eq:flight}
\end{align}
We now show that the LHS of \eqref{eq:flight} is larger than $0$. We have that for all $s_t\in \cS_t$, 
\begin{align}
	\max_{a\in \cA,j\in[S]}  P\ind{t}_{\bayes}((a,j) \mid
	s_t, \shlocal)  \geq \Ebar_{\x_t \sim q(\cdot \mid s_t)} \left[ P\ind{t}_{\bayes}((\ahat\ind{\jlocal,t}(\x_t),\iotahat\ind{\jlocal,t}(\x_t)) \mid  s_t, \shlocal) \right].\nn 
\end{align}
Combining this with \cref{eq:pbayes_pbar} implies that for all $s_t\in \cS_t$
\begin{align}
\Ebar_{\x_t\sim q(\cdot\mid s_t)}\left[V_t^{\pihat\ind{\jlocal,t}}(\x_t;\shlocal)\right]&\leq \Ebar_{\x_t\sim q(\cdot\mid s_t)}\left[\max_{a\in \cA,\k\in[S]}	Q_t^{\pihat\ind{\k,t+1}}(\x_t,a;\shlocal) \right].\label{eq:otherdirection}
\end{align}
Therefore, we have 
\begin{align}
0\leq \sum_{\pi \in \Psi\ind{t}}	\bar{d}^{\pi}(s_t)\Ebar_{\x_t \sim q(\cdot \mid s_t)}	\left[	\max_{a\in \cA,\k\in[S]}	Q_t^{\pihat\ind{\k,t+1}}(\x_t,a;\shlocal) - V_t^{\pihat\ind{\jlocal, t}}(\x_t;\shlocal) \right] \leq 2S^3 A\veps_\stat', \quad \forall s_t\in \cS^{+}_t. \label{eq:flight22}
\end{align}
On the other hand, for any $s_t \notin \cS^+_t$, we have $\sum_{\pi \in \Psi\ind{t},a'\in \cA}\bar{d}^{\pi}(s_t)	 P\ind{t}(\shlocal\mid s_t,a')=0$ (by definition of $\cS_t^+$), and so by \eqref{eq:pbayes_pbar} and \eqref{eq:otherdirection}, we have \begin{align*}\sum_{\pi \in \Psi\ind{t}}	\bar{d}^{\pi}(s_t) \Ebar_{\x_t\sim q(\cdot\mid s_t)}\left[V_t^{\pihat\ind{\jlocal,t}}(\x_t;\shlocal)\right]& \leq \sum_{\pi \in \Psi\ind{t}}	\bar{d}^{\pi}(s_t)  \Ebar_{\x_t\sim q(\cdot\mid s_t)}\left[ \max_{a\in \cA,\k\in[S]} Q_t^{\pihat\ind{\k,t+1}}(x_t,a;\shlocal)\right], \nn \\ & \leq  S\sum_{\pi \in \Psi\ind{t},a'\in\cA}	\bar{d}^{\pi}(s_t) P\ind{t}(\shlocal\mid s_t,a')  , \ \ \text{(by \eqref{eq:pbayes_pbar})}\nn \\
	&= 0.
\end{align*} 
This implies that \cref{eq:flight22} also holds for $s_t\in \cS_t\setminus \cS_t^+$, giving the desired result. 
\end{proof}

\subsubsection{Proof of \cref{thm:geniklemma}}
\label{sec:proofbmdp}

	\begin{proof}[{Proof of \cref{thm:geniklemma}}] 
	In light of \cref{lem:performancediffbmdp}, it suffices to show that, for any $t\in[h-1]$, there is an event $\cE_{t}$ which occurs with probability at least $1-
	\delta/H^2$, under which for any $s\in[S]$,
		\begin{align}
			\sigma_t \coloneqq	\min_{i\in[S]}\max_{j \in [S]}\En^{\pistar\ind{s}}\brk*{ \mathbb{I}\{\s_t\in \cS_{t,\epsilon} \} \left(Q_t^{\pihat\ind{j,t+1}}(\bx_t, \pistar\ind{s}(\x_t);s)
				-V_t^{\pihat\ind{i,t}}(\bx_t;s)\right)
			}\leq \frac{\epsilon}{2H}, \label{eq:target}
			\end{align}
where $\pistar\ind{s} \in \argmax_{\pi\in \Pibar_{\epsilon}}
		\bar{d}^{\pi}(s)$ and $V^{\pi}_{\tau}(\cdot;
		s)$ is the $V$-function at layer $\tau\in[h-1]$
		with respect to the partial policy $\pi$ for the BMDP
		$\wbar \cM$ with rewards
		$r_t(x;s)=\mathbf{1}\{\phi_\star(x)=s\} $, $t\in[h]$---see Definition in \eqref{eq:V}. 
By summing \cref{eq:target} over $t=1,\dots, h-1$, and using \cref{lem:performancediffbmdp} together with a union bound, we will be able to prove the desired result.

Fix $t\in [h-1]$ and let $\cE_t$ be the event of \cref{lem:intermediate}. Recall that $\P[\cE_t]\geq 1-\delta/H^2$. In what follows, we condition on $\cE_t$ and prove \eqref{eq:target}. Fix $\sfrak\in \cS_{h,t}$ and let $\pi_\star \equiv \pi_\star\ind{\sfrak}$. Further, let $\i$ be as in \cref{lem:intermediate} with $s_h = \sh$. 	Since $\Psi\ind{t}$ is an $(1/2,{\epsilon})$-policy
cover relative to $\Pibar_{\epsilon}$ at layer $t$ and $\pi_\star\in \Pibar_{\epsilon}$,
we have that \begin{align}\dbar^{\pi_\star}(s_t)\leq \max_{\tilde \pi\in \Pibar_{\epsilon}} \dbar^{\tilde \pi}(s_t)\leq 2 \sum_{\pi\in \Psi\ind{t}} \dbar^{\pi}(s_t), \quad \forall s_t\in\cS_{t,\eps} . \label{eq:cover0}
\end{align}
The last inequality and the definition of $\cS_{t,\epsilon}$ implies that for all $s_t\in\cS_{t,\eps}$, $\sum_{\pi\in \Psi\ind{t}} \dbar^{\pi}(s_t)>0$. This, together with \cref{lem:intermediate} (in particular, the \lhs inequality in \eqref{eq:flight2}) implies that 
\begin{align}
	\Ebar_{\x_t \sim q(\cdot \mid s_t)}	\left[\max_{a\in \cA,j\in[S]} Q_t^{\pihat\ind{j,t+1}}(\x_t,a;\sh) - V_t^{\pihat\ind{\i,t}}(\x_t;\sh)\right]\geq 0. \label{eq:nonneg0}
\end{align}
Thus, for any $s_t \in \cS_{t,\epsilon}$, we have
\begin{align}
	& \bar{d}^{\pi_\star}(s_t)\Ebar_{\x_t \sim q(\cdot \mid s_t)}	\left[\max_{a\in \wbar \cA,j\in[S]} Q_t^{\pihat\ind{j,t+1}}(\x_t,a;\sh) - V_t^{\pihat\ind{\i,t}}(\x_t;\sh)\right],  \nn\\
	& = 	\bar{d}^{\pi_\star}(s_t) \Ebar_{\x_t \sim q(\cdot \mid s_t)}	\left[\max_{a\in \cA,j\in[S]} Q_t^{\pihat\ind{j,t+1}}(\x_t,a;\sh) - V_t^{\pihat\ind{\i,t}}(\x_t;\sh)\right], \quad \text{(justified below)} \label{eq:something0}\\ & \leq  2	\sum_{\pi\in \Psi\ind{t}}	\bar{d}^{\pi}(s_t) \Ebar_{\x_t \sim q(\cdot \mid s_t)}	\left[\max_{a\in \cA,j\in[S]} Q_t^{\pihat\ind{j,t+1}}(\x_t,a;\sh) - V_t^{\pihat\ind{\i,t}}(\x_t;\sh)\right], \quad \text{(by \eqref{eq:nonneg0} and \eqref{eq:cover0})} \nn \\
	& \leq   4S^3 A C \veps_{\stat}(n,{\delta}/{ H^2}), \label{eq:wall}
\end{align}
for some absolute constant $C>0$; the last inequality follows by \cref{lem:intermediate} (in particular, the \rhs in inequality in \cref{eq:flight2}). Now, \cref{eq:something0} follows from the fact that
\begin{align}
\max_{a\in \wbar\cA} Q_t^{\pihat\ind{j,t+1}}(x_t,a;\sh)=\max_{a\in
	\cA} Q_t^{\pihat\ind{j,t+1}}(x_t,a;\sh), \quad \forall j\in[S],\nn 
\end{align} 
since
$Q_t^{\pihat\ind{j,t+1}}(x_t,\fraka;\sh)=0$. On
the other hand, by definition of $\sigma_t$ in \eqref{eq:target}, we have
\begin{align}
	\sigma_t & \leq \sum_{s_t\in \cS_{t,\epsilon}} \bar{d}^{\pi_\star}(s_t) \max_{j\in[S]}\Ebar_{\x_t \sim q(\cdot \mid s_t)}	\left[ Q_t^{\pihat\ind{j,t+1}}(\x_t,\pistar(\x_t);\sh) - V_t^{\pihat\ind{\i,t}}(\x_t;\sh)\right], \nn \\
	& \leq \sum_{s_t\in \cS_{t,\epsilon}} \bar{d}^{\pi_\star}(s_t)\Ebar_{\x_t \sim q(\cdot \mid s_t)}	\left[\max_{a\in \wbar \cA,j\in[S]} Q_t^{\pihat\ind{j,t+1}}(\x_t,a;\sh) - V_t^{\pihat\ind{\i,t}}(\x_t;\sh)\right], \nn \\ 
	& \leq  4S^4 A C \veps_{\stat}(n,\frac{\delta}{ H^2}). \quad \text{(by \eqref{eq:wall})}
	\end{align}
Now, by choosing $n$ large enough such that $8 A S^4 H  C \veps_{\stat}(n,\frac{\delta}{ H^2})\leq \epsilon$ (as in the theorem's statement), we get 
\begin{align}
\sigma_t \leq 	\frac{\epsilon}{2H}.\label{eq:final}
	\end{align}
Thus, under the event $\cE'\coloneqq \cE_1 \cup \dots \cup \cE_{h-1}$ (note that $\P[\cE']\geq 1- \delta/H$ by a union bound), we have by \cref{lem:performancediffbmdp} and \cref{eq:final} that
	\begin{align}
		\min_{i\in[S]}	\Ebar^{\pi_\star} \left[ V^{\pi_\star}_1(\x_1;\sh) - V^{\pihat\ind{i,1}}_1(\x_1;\sh)\right] \leq \sum_{t=1}^{h-1}\sigma_t \leq \frac{\epsilon}{2}. \label{eq:finalfinal}
\end{align}
Note that $\Ebar^{\pi_\star} \left[ V^{\pi_\star}_1(\x_1;\sh) \right]=\max_{\pi\in \Pibar_{\epsilon}}\dbar^{\pi}(\sfrak)$ and $\Ebar^{\pi_\star}\brk{V^{\pihat\ind{i,1}}_1(\x_1;\sh)} = \dbar^{\pihat\ind{i,1}}(\sfrak)$, by definition of $\pi_\star$ and the $V$-function. Thus, \eqref{eq:finalfinal} implies that 
\begin{align}
	\max_{i\in[S]} \dbar^{\pihat\ind{i,1}}(\sfrak) &\geq \max_{\pi\in \Pibar_{\epsilon}}\dbar^{\pi}(\sfrak) -\frac{\epsilon}{2} \geq \frac{1}{2}\max_{\pi\in \Pibar_{\epsilon}}\dbar^{\pi}(\sfrak),\nn 
\end{align}
where the last inequality follows from the fact that $ \max_{\pi\in \Pibar_{\epsilon}}\dbar^{\pi}(\sfrak) \geq \epsilon$, since $\sfrak\in \cS_{h,\epsilon}$. This means that $\Psi\ind{h}=\{\pihat\ind{i,1}\colon i\in[S]\}$ is a $(1/2,\epsilon)$-policy cover relative to $\Pibar_{\epsilon}$ for layer $h$ in $\Mbar$, which completes the proof. 
		\end{proof}

        \section{Proofs for Reward-Based RL}	
	\label{app:PSDPthmproof}
	\renewcommand{\r}{\bm{r}}
	\begin{lemma}
	\label{lem:reg0}
	Let $n\geq 1$ and $\delta\in(0,1)$, and define $\veps_\stat(n,\delta)\coloneqq n^{-1/2}  \sqrt{S A\ln n + \ln (|\Phi|/\delta)}$. Further, suppose that \cref{assum:real} and \cref{assum:reward} hold. Then, there exists an absolute constant $C>0$ such that for all $h\in H$ the solution $(\fhat\ind{h},\phihat\ind{h})$ of the least-squares problem in \cref{eq:mistake} of \cref{alg:PSDP} satisfies with probability at least $1-\delta$,
	\begin{align}
		\E^{\unif(\Psi\ind{h})}\left[ \max_{a\in\cA}\left( \fhat\ind{h}(\phihat\ind{h}(\x_h),a)	- Q_h^{\pihat\ind{h+1}}(\x_h, a) \right)^2 \right]\leq C^2\cdot   \veps^2_\stat(n,\delta).\nn 
	\end{align}
\end{lemma}

\begin{proof}[\pfref{lem:reg0}]
	Fix $h\in[h-1]$ and let $\tilde f\ind{h}_{\bayes}$ be as in \cref{lem:bayes00}. By \cref{lem:bayes00}, $(\tilde f\ind{h}_{\bayes},\phi_\star)$ is the Bayes-optimal solution of the least-square problem in \cref{eq:mistake} of \cref{alg:PSDP}. And so, by \cref{assum:real} and a standard guarantee for least-square regression (see e.g.~\citep{van2000empirical}), there exists an absolute constant $C'>0$ (independent of $h$ and any other problem parameter) such that with probability at least $1-\delta$, 
	\begin{align}
		{\E}^{\unif(\Psi\ind{h})} \left[ \max_{a\in\cA}\left( \fhat\ind{h}(\phihat\ind{h}(\x_h),a)	-   \tilde f\ind{h}_{\bayes}(\phi_\star(\x_h),a) \right)^2 \right]\leq  \tilde\veps^2_\stat(n,\delta), \nn 
	\end{align}
	where $\tilde \veps^2_\stat(n,\delta)\coloneqq C'  \ln   \cN_{\cF}(1/n)  + C' \ln (|\Phi|/\delta)$ and $ \cN_{\cF}(1/n)$ denotes the $\frac{1}{n}$-covering number of the set $\cF \coloneqq \{f\colon [S]\times \cA \rightarrow \reals_{\geq 0}\}$ in $\ell_{\infty}$ distance. It is easy to verify that $\cN_{\cF}(1/n)\leq n^{A S}$, and so by setting $C^2\coloneqq C'$, we have
	\begin{align}
		\tilde \veps_{\stat}^2(n,\delta) \leq  C^2\cdot \veps_{\stat}^2(n,\delta). \nn 
	\end{align} 
	Now, by the expression of $\tilde f_{\bayes}\ind{h}$ in \cref{eq:good00} and \cref{assum:reward}, we have that $\tilde f\ind{h}_{\bayes}(\phi_\star(x),a)= Q^{\pihat\ind{h+1}}_h(x,a)$, which completes the proof.
\end{proof}

\begin{lemma}
	\label{lem:bayes00}
	Let $h\in [H]$ and consider of the unconstrained problem 
	\begin{align}
		f_{\bayes}\ind{h} \in \argmin_{ f\colon \cX_h \times \cA \rightarrow \reals_{\geq 0}} \E^{\unif(\Psi\ind{h})\circ_h \unifa \circ_{h+1}  \pihat\ind{h+1}} \left[\left( f(\x_h, \a_h) - \sum_{\tau=h}^{H} \r_\tau \right)^2\right], \label{eq:normal00}
	\end{align}
	where $(\r_h)$ are the reward random variables and $\pihat\ind{h+1}\in \Pim^{h+1:H}$ is as in \cref{alg:PSDP}. Then, under \cref{assum:reward} for any $a\in\cA$, $x \in \cX_h$, and $s = \phi_\star(x)$, $f_{\bayes}\ind{h}$ satisfies
	\begin{align}
		f\ind{h}_{\bayes}(x,a) =  \tilde f\ind{h}_{\bayes}(s,a) \coloneqq \rbar_h(s,a ) +\E^{\pihat\ind{h+1}} \left[ \left. \sum_{\tau=h+1}^{H} r_\tau(\x_\tau, \pihat\ind{\tau}(\x_\tau)) \right|  \s_h = s, \a_h =a\right]. \label{eq:good00}
	\end{align}
	Further, $(\tilde f\ind{h}_{\bayes}, \phi_\star)$ is the Bayes-optimal solution of the problem in \cref{eq:mistake} of \cref{alg:PSDP}; that is,
	\begin{align}
		(\tilde f_{\bayes}\ind{h}, \phi_\star) \in \argmin_{\tilde f\colon [S]\times \cA \rightarrow \reals_{\geq 0}, \phi \in \Phi} \E^{\unif(\Psi\ind{h})\circ_h \unifa \circ_{h+1}  \pihat\ind{h+1}} \left[\left( \tilde f(\phi(\x_h), \a_h) - \sum_{\tau=h}^{H} \r_\tau \right)^2\right].\nn 
	\end{align}
\end{lemma}
\begin{proof}[\pfref{lem:bayes00}]
	Fix $a\in\cA$ and $x\in \cX_{h}$, and let $s= \phi_\star(x)$. The least-squares solution $f_{\bayes}\ind{h}$ of the problem in \cref{eq:normal00} is given by
	\begin{align}
		f_{\bayes}\ind{h} (x,a)& =  \E^{\pihat\ind{h+1}} \left[ \left. \sum_{\tau=h}^{H} \br_\tau \right| \x_h =x ,\a_h =a \right], \nn \\
		& = \E[ \br_h\mid \x_h = x,\a_h = a]+ \E^{\pihat\ind{h+1}} \left[ \left. \sum_{\tau=h+1}^{H} \br_\tau \right|  \x_h = x, \a_h =a\right], \nn \\
		& =   \rbar_h(s,a) +\E^{\pihat\ind{h+1}} \left[ \left. \sum_{\tau=h+1}^{H} \br_\tau \right|  \x_h = x, \a_h =a\right], \quad \text{(by \cref{assum:reward})} \nn \\
		& =   \rbar_h(s,a) +\E^{\pihat\ind{h+1}} \left[ \left. \sum_{\tau=h+1}^{H} \br_\tau \right|  \s_h = s, \a_h =a\right],\label{eq:pass00}\nn \\
		& = \tilde f\ind{h}_{\bayes}(s,a),
	\end{align} 
	where \cref{eq:pass00} follows by the Block MDP assumption. 
	Now that we have established \cref{eq:good00}, we show the second claim of the lemma. The unconstrained population version of the problem in \cref{eq:mistake} of \cref{alg:PSDP} becomes equivalent to the following problem:
	\begin{align}
		\min_{\tilde f\colon [S]\times \cA \rightarrow \reals_{\geq 0}, \phi \in \Phi} \E^{\unif(\Psi\ind{h})\circ_h \unifa \circ_{h+1}  \pihat\ind{h+1}} \left[\left( \tilde f(\phi(\x_h), \a_h) - \sum_{\tau=h}^{H} \r_\tau \right)^2\right]. \label{eq:abnormal0}
	\end{align}
	Note that the value of this problem is always at least that of \cref{eq:normal00}. On the other hand, by \cref{eq:good00}, the value of the objective in \cref{eq:abnormal0} with the pair $(\tilde f,\phi)=(\tilde f\ind{h}_\bayes, \phi_\star)$ matches the optimal value of the problem \cref{eq:normal00}, and so $( \tilde f\ind{h}_{\bayes}, \phi_\star)$ is indeed a solution of \cref{eq:abnormal0}. 
\end{proof}

We now restate and prove a slightly more detailed version of \cref{thm:PSDPthm}.
\begin{theorem}[Restatement of \cref{thm:PSDPthm}]
	\label{thm:PSDPthm2}
	Let $\alpha$, $\veps$, $\delta \in(0,1)$ be given. Further, let $\veps_\stat(\cdot, \cdot)$ and $C>0$ be as in \cref{lem:reg0} ($C$ is an absolute constant independent of problem parameters) and suppose that \cref{assum:real,assum:reward} hold, and that for all $h\in\brk{H}$:
	\begin{enumerate}
		\item $\Psi\ind{h}$ is a $(\alpha, \eps)$-approximate cover for layer $h$, where $\eps\coloneqq \veps/(2SH^2)$.
		\item $|\Psi\ind{h}|\leq S$.
	\end{enumerate}
	Then, as long as $n$ is chosen such that $4 S^2 H C \veps_{\stat}(n,\delta/H)/\alpha$, we have that with probability at least $1-\delta$, the policy $\pihat\ind{1}$ outputed by \pref{alg:PSDP} satisfies
	\begin{align}
		\E^{\pihat\ind{1}}\left[\sum_{h=1}^H \br_h\right]\geq    \max_{\pi \in \Pim}  \E^{\pi}\left[\sum_{h=1}^H \br_h\right] - \veps. \nn 
	\end{align}
	In particular, the total number of sampled trajectories required by the algorithm is
	\begin{align}
		\bigoht(1) \cdot\frac{ H^3 S^4 (S A  + \ln (|\Phi|/\delta))}{ \alpha^2 \veps^2}. \nn 
	\end{align}   
\end{theorem}

\begin{proof}[{Proof of \cref{thm:PSDPthm2}}]
	We proceed by induction to show that for any $h\in[H]$, there is an event $\cE_{h}$ of probability at least $1- \delta/H$ under which the learned partial policy $\pihat\ind{h}$ is such that
	\begin{align}
		\E^{\pi_\star} \left[Q^{\pihat\ind{h+1}}_h(\x_h,\pi_\star(\x_h))- Q^{\pihat\ind{h+1}}_{h}(\x_h, \pihat\ind{h}(\x_h))\right] \leq  \frac{\veps}{H}, \label{eq:new00}
	\end{align}
	where $\pi_\star \in \argmax_{\pi \in \Pim}  \E^{\pi}[\sum_{h=1}^H \br_h]$ is the optimal policy and 
	\begin{align}
		Q^{\pi}_h(x,a)\coloneqq  \E^{\pi}\left[\left.\sum_{\tau=h}^H \br_\tau \,\right|\, \x_h = x, \a_h = a\right],\nn 
		\end{align}
	is the $Q$-function corresponding to the rewards $(\br_h)$ and the policy $\pi$. Once we establish \cref{eq:new00} for all $h\in[H]$, we will apply the performance difference lemma to obtain the desired result. 
	
	Fix $h\in[H]$. By \cref{lem:bayes00}, there is an event $\cE_h$ of probability at least $1-\delta/H$ under which the solution $(\fhat\ind{h},\phihat\ind{h})$ of the least-squares regression problem on \cref{eq:mistake} of \cref{alg:PSDP} satisfies,
	\begin{align}
		\E_{\s_h\sim \unif(\Psi\ind{h})} \E_{\x_h\sim q(\cdot \mid \s_h)}\left[	\max_{a\in \cA}\left| \fhat\ind{h}(\phihat\ind{h}(\x_h),a)	-  Q^{\pihat\ind{h+1}}_h(\x_h, a)  \right| \right]\leq   C\cdot\veps_\stat(n,\tfrac{\delta}{H}), \label{eq:neverbefore00}
	\end{align}
for some absolute constant $C$ independent of $h$ and other problem parameters. Let $\wtilde \cS_{h,\eps} \subseteq \cS_h$ be the subset of states $s$ such that $\max_{\pi \in \Pim} d^{\pi}(s)<\eps$. Moving forward, we let $\veps'_\stat\coloneqq  C \cdot\veps_\stat(n,\frac{\delta}{H})$ and fix $\sh\in \wtilde\cS_{h,\eps}$. From \cref{eq:neverbefore00} and that $|\Psi\ind{h}|\leq S$, we have
	\begin{align}
		\sum_{\pi\in \Psi\ind{h}} d^{\pi}(\st) \E_{\x_h\sim q(\cdot \mid \st)}\left[ \max_{a\in \cA}\left| \fhat\ind{h}(\phihat\ind{h}(\x_h),a)	-   Q^{\pihat\ind{h+1}}_h(\x_h, a) \right| \right]\leq   {S\veps'_\stat}. \nn 
	\end{align}
	Now, let $\pihat\ind{h}(x_h)\in \argmax_{a\in \cA} \fhat\ind{h}(\phihat\ind{h}(x_h),a)$. With this and the fact that $|\|y\|_{\infty}-\|z\|_{\infty}|\leq \|y-z\|_{\infty}$, for all $y,z\in \reals^{A}$, we have
	\begin{align}
		\sum_{\pi\in \Psi\ind{h}} d^{\pi}(\st) \E_{\x_h\sim q(\cdot \mid \st)}\left[\max_{a\in \cA}  Q^{\pihat\ind{h+1}}_h(\x_h, a)\right]& \leq		\sum_{\pi\in \Psi\ind{h}} d^{\pi}(\st) \E_{\x_h\sim q(\cdot \mid \st)} \left[  \fhat\ind{h}( \phihat\ind{h}(\x_h), \pihat\ind{h}(\x_h))\right]  +S\veps'_\stat, \nn \\
		& \leq   		\sum_{\pi\in \Psi\ind{h}} d^{\pi}(\st) \E_{\x_h\sim q(\cdot \mid \st)}\left[ Q^{\pihat\ind{h+1}}_h(\x_h, \pihat\ind{h}(\x_h)) \right] +  2S \veps'_\stat. \nn
	\end{align}
	Thus, since $\Psi\ind{h}$ is a $(\alpha,\eps)$-approximate policy cover and $\sh\in \wtilde \cS_{h,\eps}$, we have that 
	\begin{align}
		2 S \veps'_\stat & \geq   \sum_{\pi\in \Psi\ind{h}} d^{\pi}(\st) \E_{\x_h\sim q(\cdot \mid \st)} \left[\max_{a\in \cA}  Q^{\pihat\ind{h+1}}_h(\x_h, a) -  Q^{\pihat\ind{h+1}}_h(\x_h, \pihat\ind{h}(\x_h))\right],\nn \\ & \geq  \alpha d^{\pi_\star}(\st) \E_{\x_h\sim q(\cdot \mid \st)} \left[\max_{a\in \cA}  Q^{\pihat\ind{h+1}}_h(\x_h, a) -  Q^{\pihat\ind{h+1}}_h(\x_h, \pihat\ind{h}(\x_h))\right].
	\nn 
	\end{align}
	We have just shown that
	\begin{align}
		d^{\pi_\star}(s_h) \E_{\x_h\sim q(\cdot \mid s_h)} \left[\max_{a\in \cA}  Q^{\pihat\ind{h+1}}_h(\x_h, a) -  Q^{\pihat\ind{h+1}}_h(\x_h, \pihat\ind{h}(\x_h))\right]
		\leq 2S \veps'_\stat/\alpha, \quad \forall s_h\in \wtilde\cS_{h,\eps}. \label{eq:flight00}
	\end{align}
	On the other hand, for any $s_h \notin \wtilde \cS_{h,\eps}$, we have $d^{\pi_\star}(s_h)< \eps$. Using this and the fact that $Q^{\pihat\ind{h+1}}(x,a)\in [0,H]$, we have 
	\begin{align}
		d^{\pi_\star}(s_h) \E_{\x_h\sim q(\cdot \mid s_h)} \left[\max_{a\in \cA}  Q^{\pihat\ind{h+1}}_h(\x_h, a)\right] \leq H \eps, \quad \forall s_h \notin  \wtilde \cS_{h,\eps}.\nn 
	\end{align}
	Combining this with \cref{eq:flight00} and that the $Q$-function is non-negative (by \cref{assum:reward}), we have 
	\begin{align}
		&	\E^{\pi_\star} \left[ Q^{\pihat\ind{h+1}}_h(\x_h, \pi_\star(\x_h))-  Q^{\pihat\ind{h+1}}_h(\x_h, \pihat\ind{h}(\x_h))\right] \nn \\   & \leq   \E^{\pi_\star} \left[\max_{a\in \cA}  Q^{\pihat\ind{h+1}}_h(\x_h, a)-  Q^{\pihat\ind{h+1}}_h(\x_h, \pihat\ind{h}(\x_h))\right], \nn \\
		& =  \E^{\pi_\star} \left[\max_{a\in \cA}  Q^{\pihat\ind{h+1}}_h(\x_h, a)-  Q^{\pihat\ind{h+1}}_h(\x_h, \pihat\ind{h}(\x_h))\right],\nn \\
		&\leq  		d^{\pi_\star}(s_h) \E_{\x_h\sim q(\cdot \mid s_h)} \left[\max_{a\in \cA}  Q^{\pihat\ind{h+1}}_h(\x_h, a)-  Q^{\pihat\ind{h+1}}_h(\x_h, \pihat\ind{h}(\x_h))\right],\nn \\
		& \leq  	2  S^2 \veps'_\stat/\alpha + HS \eps,\nn \\
	&	\leq  2  S^2 \veps'_\stat/\alpha +\veps/(2H),  \label{eq:prelabor} 
	\end{align}
	where the last inequality follows by the fact that $\eps = \veps/(2H^2S)$. Now, by choosing $n$ large enough such that $4 S^2 H C \veps_{\stat}(n,\delta/H)/\alpha \leq \veps$ (as in the theorem's statement), we have that $\veps'_\stat \leq \veps \alpha /(4 H S^2)$ (by definition of $\veps_\stat'$) and so \cref{eq:prelabor} implies
	\begin{align}
	\E^{\pi_\star} \left[ Q^{\pihat\ind{h+1}}_h(\x_h, \pi_\star(\x_h))-  Q^{\pihat\ind{h+1}}_h(\x_h, \pihat\ind{h}(\x_h))\right]\leq \frac{\veps}{H} .  \label{eq:labor}
		\end{align}
	Recall that this inequality holds under the event $\cE_h$.
	On the other hand, by the performance difference lemma \citep{kakade2003sample} and the definition of $\pi_\star$, the $V$-function $V_1^{\pi}(x) \coloneqq  \E^{\pi}[\sum_{h=1}^H \br_h\mid \x_1 = x]$ satisfies
	\begin{align}
		\E	[V^{\pihat\ind{1}}_1(\x_1)] -    \max_{\pi \in \Pim} \E[V^{\pi}_1(\x_1)]& =\E[V^{\pihat\ind{1}}_1(\x_1)] -     \E[V^{\pistar}_1(\x_1)],\nn\\ &=   \E^{\pi_\star} \left[ Q^{\pihat\ind{h+1}}_h(\x_h, \pi_\star(\x_h))-  Q^{\pihat\ind{h+1}}_h(\x_h, \pihat\ind{h}(\x_h))\right]. \nn  %
	\end{align}
	Thus by \cref{eq:labor}, we have that under the event $\cE \coloneqq \bigcup_{h=1}^H\cE_h$,
	\begin{align}
		\E	[V^{\pihat\ind{1}}_1(\x_1)] -    \max_{\pi \in \Pim} \E[V^{\pi}_1(\x_1)] \leq \veps.\nn
	\end{align}
	The desired suboptimality result follow by the fact that $\P[\cE] \geq 1 -\delta$.
	
	\paragraph{Sample complexity of $\psdp$} In order to satisfy the condition $4 S^2 H C \veps_{\stat}(n,\delta/H)/\alpha \leq \veps$ in the theorem statement (where $C$ is some absolute constant), $n$ needs to be larger than $N=\bigoht(1) \cdot (H^2 S^4 (S A + \ln (|\Phi|/\delta))/(\alpha\veps)^2)$, where $\wtilde{O}$ hides log-factors in $1/\veps, A, S$, $H$, and $\ln |\Phi|$. Since $n$ represents the number of sampled trajectories per layer in $\psdp$, the total number of sampled trajectories in the latter is simply $N_{\psdp}= H N$.
\end{proof}
         
        \arxiv{
                \section{Details for Experiments}
        \label{app:experiments}

In this section, we give the details for the $\musik$ and $\psdp$
implementations in our experiments, as well as hyperparameter choices.

\paragraph{Implementation of \musik} We use $\musik$ to learn a policy
cover, which is then used within $\psdp$ to find a near-optimal policy in the
CombLock environment. In the CombLock environment, the optimal policy cover can be
learned by composing optimal policy covers at each layer (though this is
not true in general, many problems share this property). We follow an
approach taken with $\homer$ in \cite{misra2020kinematic}, and take
advantage of this composability property to implement a more sample-efficient
version of $\musik$, where during the call to the $\ikdp$ subroutine
at layer $h$, we only learn $\fhat\ind{h-1},\phihat\ind{h-1}$
(i.e.~the $\ikdp$ for-loop stops at $t=h-1$); this is exactly what was
done in \cite{misra2020kinematic}. This version of $\musik$, which we
name $\musik.\texttt{comp}$, is displayed in \cref{alg:IKDPcomp}. (for
this variant, we do not write $\ikdp$ as a separate subroutine).

We use $\Phi \coloneqq
\{\phi_W\colon x \mapsto\argmax_{i\in[N]} W x  \mid W\in \reals^{N\times d} \}$ for the
decoder class, where we recall that $N$ is the number of latent states
per layer in the CombLock environment; this is exactly the same
decoder class as the one used in \citet{misra2020kinematic} for
experiments with $\homer$. Given the observation process in the
CombLock environment, there exists a matrix $W_\star \in
\reals^{N\times d}$ (corresponding to the Hadamard matrix used to generate the observation; see \cref{sec:experiments}) such that the true decoder $\phi_\star$ is given
by $\phi_{W_\star}$.\footnote{Technically, this is not a decodable setting, since noise is added to the observation processes (see \cref{sec:experiments}). Here, we let $\phi_\star$ be the true decoder in the noiseless case. }
To learn $W_\star$, we do not use the parameterization in $\Phi$
directly, and instead work with the differentiable decoder class
$\Phi'\coloneqq \{ x \mapsto\mathrm{softmax} (W x)  \mid W\in
\reals^{N\times d}\}$ during training; this is reflected in the
objective in the next display. Further, we make a slight, empirically-motivated modification to the conditional density
estimation problem in Line~\ref{line:inversenew2} of $\ikdp$, and
instead solve
\begin{align}
	\fhat\ind{h-1} ,
	\hat{\psi}\ind{h-1}\gets \argmax_{f: \cX \times [N]  \rightarrow
		\Delta(\cA) , \psi \in \Phi'}\  \sum_{(a_h,x_{h-1},x_h)\in \cD\ind{h-1}} \log \left(\sum_{i=1}^N f(a_{h-1} \mid x_{h-1}, i) \cdot [\psi(x_h)]_i  \right).\label{eq:newobjective}
\end{align}
Compared to the original objective of $\ikdp$ in
\cref{line:inversenew2} of \cref{alg:IKDP}, we no longer need to
predict the index of the future roll-out policies (since the for-loop
of $\ikdp$ now stops at $t=h-1$, there are no future
roll-outs). Another difference is that we do not use a decoder at
layer $h-1$; we use $f(a\mid x,j)$ instead of $f(a\mid \phi(x),j)$
(this helps with the training). For each $j\in[N]$, we instantiate
$f(\cdot \mid \cdot, j)$ with a two-layer neural network with $\tanh$
activation, input dimension $d$, hidden dimension $N_{\texttt{hidden}}$, and output dimension $A$, where the
output is pushed through a softmax so that $f(\cdot \mid x, j)$ is a
distribution over actions for any $x\in \cX$. We use $\texttt{Adam}$ to solve the optimization problem in \eqref{eq:newobjective}. We specify the choices of hyperparameters in the sequel.

With $\hat{\psi}\ind{h-1}$ as in \eqref{eq:newobjective}, the learned decoder is given by $\phihat\ind{h-1}(x)\coloneqq \argmax_{i\in[N]}[\hat{\psi}\ind{h-1}x]_i$. Further, for $\fhat\ind{h-1}$ as in \eqref{eq:newobjective}, the $h$th layer policy cover $\Psi\ind{h} =\{\pihat\ind{j,h}\}_{ j \in[N]}$ constructed by $\musik.\texttt{comp}$ is essentially given by: 
\begin{align}
	\pihat\ind{j,h} =  \pihat \circ_{h-1} \ahat\ind{j,h-1} , \quad \text{where} \quad \pihat\in \argmax_{\pi \in \Psi\ind{h-1}}  \P^{\pi \circ_{h-1} \unifa}\left[\phihat\ind{h-1}(\x_h)=j\right], \label{eq:argmax}
\end{align}
and $\ahat\ind{j,h-1}(x) \coloneqq
\argmax_{a\in\cA}\fhat\ind{h-1}(a\mid x,j).$ That is, the policy
$\pihat\ind{j,h}$ is the composition of the best partial policy $\pihat$ among the partial policies in $\Psi\ind{h-1}$ (the policy
cover at the previous layer) and the best action at layer $h-1$ to
maximize to probability of reaching the `abstract state'
$j\in[N]$. Technically, computing $\pihat$ requires
estimating $\P^{\pi \circ_{h-1} \unifa}\left[\phihat\ind{h-1}(\x_h)=j\right]$,
for all $\pi \in \Psi\ind{h-1}$. For this, we reuse the dataset
$\cD\ind{h-1}$ from \eqref{eq:newobjective} and solve another conditional density estimation problem---see \cref{eq:anothercond} in \cref{alg:IKDPcomp}.\footnote{Technically, the solution of the conditional estimation problem in \eqref{eq:anothercond} does not yield an estimator of $\P^{\pi \circ_{h-1} \unifa}\left[\hat\phi\ind{h-1}(\x_h)=j\right]$ per se. But it gives us a proxy for a function whose argmax $\pihat$ in \eqref{eq:argmax}.}

\paragraph{$\psdp$ implementation} The only modification we make to the $\psdp$
algorithm is that we use $f(a\mid x)$ instead of $f(a \mid \phi(x))$
in the objective \eqref{eq:mistake} (i.e.~we do not use a decoder). We instantiate $f(\cdot \mid \cdot)$ with a two-layer neural network with input dimension $d$, hidden dimension of 400, and output dimension of 1. We use the $\tanh$ activation function at all layers.

\paragraph{Hyper-parameters} For each $j\in[N]$, we instantiate $f(\cdot \mid \cdot, j)$ in \eqref{eq:newobjective} with a two-layer neural network with $\tanh$ activation, input dimension $d$, hidden dimension of size $N_{\texttt{hidden}}$, and output dimension $A$, where the output is run through the softmax activation function (with temperature 1) so that $f(\cdot \mid x, j)$ is a distribution over actions for any $x\in \cX$. We also instantiate $g(\cdot \mid \cdot)$ in \eqref{eq:anothercond} with a two-layer neural network with $\tanh$ activation. input dimension $N$, hidden dimension of size $N_{\texttt{hidden}}=400$, and output dimension $N$, where the output is pushed through a softmax (with temperature 1) so that $g(\cdot \mid j)$ is a distribution over $[N]$ for any $j\in [N]$. For the choice of hidden size $N_{\texttt{hidden}}$, we searched over the grid $\{100,200, 400\}$. The results reported in \cref{fig:horizon-plot} are for $N_{\texttt{hidden}}=200$.

We optimize the parameters of $(f,\theta)$ [resp.~$g$] in
\eqref{eq:newobjective} [resp.~\eqref{eq:anothercond}] using
$\texttt{Adam}$ with the default parameters in $\texttt{PyTorch}$. We
use a batch size of $\min(n,N_{\texttt{batch}})$, where $n$ is as in
\cref{alg:IKDPcomp}, and perform $N_{\texttt{update}}$ gradient
updates. For the batch size $N_{\texttt{batch}}$ and number of updates
$N_{\texttt{updates}}$, we searched over the girds
$\{1024,2048,4096,8196\}$ and $\{64, 128, 256\}$,
respectively. The results reported in \cref{fig:horizon-plot} are for
$N_{\texttt{batch}}=4096$ and $N_{\texttt{updates}}=128$. 

We selected the hyperparameters $N_{\texttt{hidden}}, N_{\texttt{batch}}$, and $N_{\texttt{updates}}$ based on performance in the setting where $H=100$, $N_{\texttt{episodes}}=1.1e^6$, and the seed is set to 0 (we did not use this seed for evaluation). Out of the choices of hyperparameters tested, the choice $(N_{\texttt{hidden}},N_{\texttt{batch}}, N_{\texttt{updates}})=(200,4096,128)$ enabled \musik{} to identify the optimal policy. 

\paragraph{Number of episodes tested in evaluation} For our evaluation results, we tested the following values for the number of episodes:
\begin{itemize}
	\item For $H=25$, we test $N_{\texttt{episodes}}\in\{60000, 65000,70000,75000 \}$.
	\item For $H=50$, we test $N_{\texttt{episodes}}\in\{250000,300000, 400000,450000\}$.
	\item For $H=100$, we test $N_{\texttt{episodes}}\in\{1000000, 1100000,1200000,1300000 \}$.
	\end{itemize}

\begin{algorithm}[htp]
	\caption{$\musik.\texttt{comp}$: Variant of $\musik$ for composable policy covers (used in experiments).}
	\label{alg:IKDPcomp}
	\begin{algorithmic}[1]\onehalfspacing
		\Require
		~
		\begin{itemize}[leftmargin=*]
			\item Dimension of the observation space $d$.
			\item Number of latent states per layer $N$.
			\item Number of samples $n$.
		\end{itemize}
		\State Set $\Psi\ind{1} = \{\unifa,\dots, \unifa\}$ with $|\Psi\ind{1}|=N$.
		\For{$h=2,\ldots, H$} 
		\State $\cD\ind{h}\leftarrow \emptyset$.
		\State Let $\iota\ind{h-1} \colon \Psi\ind{h-1}\rightarrow [N]$ be any one-to-one mapping.
		\Statex[1]\algcommentbiglight{Collect data by rolling in with
			policy cover}
		\For{$n$ times}
		\State Sample $\pihat  \sim \unif(\Psi\ind{h-1})$.  
		\State Sample $(\x_{h-1}, \a_{h-1}, \x_{h})\sim \hat\pi \circ_{h-1} \unifa$. 
		\State $\cD\ind{h-1} \leftarrow \cD\ind{h-1}\cup \{(\iota\ind{h-1}(\hat\pi),\a_{h-1}, \x_{h-1}, \x_{h})\}$.
		\EndFor
		\Statex[1] \algcommentbiglight{Inverse kinematics}
		\State For $\Phi'\coloneqq \{ x \mapsto\mathrm{softmax} (W x)  \mid W\in \reals^{N\times d}\}$, solve
		\begin{equation}
			\fhat\ind{h-1},	\hat{\psi}\ind{h-1}\gets \argmax_{f\colon \cX \times[N] \rightarrow
				\Delta(\cA) , \psi\in \Phi'}
			\sum_{(-,a,x,x')\in \cD\ind{h-1}}  \ln   \left(\sum_{j\in[N]} f( a \mid
			x, j) \cdot [\psi(x)]_j\right).\label{eq:newobjective} 
		\end{equation}
		\Statex[1] \algcommentbiglight{Inverse Kinematics to learn associations between policies at subsequent layers}
		\State Solve
		\begin{equation}
			\hat g\ind{h-1}\gets \argmax_{g\colon [N] \rightarrow
				\Delta([N]) }
			\sum_{(i,-,-,x')\in \cD\ind{h-1}} \ln    g\left( i \, \left|\,
			\argmax_{i\in[N]}[\hat{\psi}\ind{h-1}(x')]_i \right. \right). \label{eq:anothercond} 
		\end{equation}
		\Statex[1] \algcommentbiglight{Update partial policy cover}
		\State For each $j\in\brk{S}$, define
		\begin{align}
			\ahat\ind{j,h-1}(x) &=
			\argmax_{a\in \cA} \fhat\ind{h-1}(a \mid
			x,j), \quad x\in \cX_t.\nn \\
			\iotahat\ind{j,h-1}(x)	& = 	\argmax_{i\in [N]}  \hat g \ind{h-1}(i\mid j).\nn 
		\end{align}
		\State  For each $j \in[S]$, define $\hat
		\pi^{(j,h)}\in \Pim^{1:h-1}$ via
		\begin{align}
			\hat  \pi^{(j,h)}(x_{\tau})\coloneqq \left\{
			\begin{array}{ll}
				\ahat\ind{j,h-1}(x_\tau),&\quad \tau=h-1,\\
				\pihat
				^{(\iotahat\ind{j,h-1},h-1)}(x_{\tau}),&\quad\tau\in[h-2],
			\end{array}
			\right. \quad x_\tau \in \cX_\tau.\nn 
		\end{align}
	\State Define $\Psi\ind{h} =\{ \pihat\ind{j,h} \colon j \in[N] \}$ \algcommentlight{Policy cover for layer $h$.}
	\EndFor
	\State \textbf{Return:} Policy covers
	$\Psi\ind{1}, \dots, \Psi\ind{H}$.
\end{algorithmic}
\end{algorithm}

     }

\end{document}